\documentclass[11pt]{article}
\usepackage[letterpaper, margin=1.0in]{geometry}

\usepackage{times}

\usepackage[numbers,sort]{natbib}

\usepackage[dvipsnames]{xcolor} %
\usepackage{tikz}
\usetikzlibrary{shapes}
\usetikzlibrary{positioning}
\usetikzlibrary{plotmarks}
\usepackage{microtype}
\usepackage{graphicx}
\usepackage{enumitem,caption,subcaption}
\usepackage{booktabs} %
\usepackage{multirow}
\usepackage{adjustbox}

\usepackage{stmaryrd}
\usepackage{graphicx}
\usepackage{caption}
\usepackage{subcaption}

\usepackage{textcomp} %
\usepackage{quoting}
\usepackage{titletoc} %
\usepackage{fontawesome} %

\usepackage[colorlinks = true,
            linkcolor = blue,
            urlcolor  = blue,
            citecolor = blue,
            anchorcolor = blue]{hyperref}

\usepackage{algorithm}
\usepackage[noend]{algpseudocode}
\algrenewcommand\algorithmicrequire{\textbf{Input:}}
\algrenewcommand\algorithmicensure{\textbf{Output:}}

\newcommand{\fedavg}{\textit{FedAvg}}
\newcommand{\newfl}{$\Delta$-\textit{FL}} %
\newcommand{\conf}{\mathrm{conf}}
\renewcommand{\epsilon}{\varepsilon}

\newlength{\continueindent}
\usepackage{etoolbox}
\makeatletter
\newcommand*{\ALG@customparshape}{\parshape 2 \leftmargin \linewidth \dimexpr\ALG@tlm+\continueindent\relax \dimexpr\linewidth+\leftmargin-\ALG@tlm-\continueindent\relax}
\apptocmd{\ALG@beginblock}{\ALG@customparshape}{}{\errmessage{failed to patch}}
\makeatother
\algblock{ParFor}{EndParFor}
\algnewcommand\algorithmicpardo{\textbf{in parallel do}}
\algrenewtext{ParFor}[1]{\algorithmicfor\ #1\ \algorithmicpardo}
\algrenewtext{EndParFor}{\textbf{end for}}
\algtext*{EndParFor}
\algdef{SE}[SUBALG]{Indent}{EndIndent}{}{\algorithmicend\ }%
\algtext*{Indent}
\algtext*{EndIndent}

\newcommand{\myparagraph}[1]{\paragraph{#1.}\hspace{-0.8em}}  %
\usepackage{bm,mathtools,amsthm,amssymb}
\usepackage{bbm}

\usepackage{color}
\definecolor{puorange}{rgb}{0.80,0.20,0}
\definecolor{bluegray}{rgb}{0.04,0,0.7}
\definecolor{greengray}{rgb}{0.05,0.50,0.15}
\definecolor{darkbrown}{rgb}{0.40,0.2,0.05}
\definecolor{darkcyan}{rgb}{0,0.4,1}
\definecolor{black}{rgb}{0,0,0}
\definecolor{grey}{rgb}{0.93,0.93,0.93}

\newcommand \reals {\mathbb{R}}

\newcommand \T {^{\top}}	%

\newcommand \prob {\mathbb{P}}

\newcommand \expect {\mathbb{E}}

\DeclarePairedDelimiterX{\inp}[2]{\langle}{\rangle}{#1, #2} %
\DeclarePairedDelimiterX{\norm}[1]{\Vert}{\Vert}{#1} %
\DeclarePairedDelimiterX{\normsq}[1]{\Vert}{\Vert^2}{#1} %

\newcommand \normd [1]{\Vert #1 \Vert^{*}}

\newcommand \eps \epsilon

\newcommand \argmin {\operatorname*{arg\,min}} %
\newcommand \argmax {\operatorname*{arg\,max}} %
\newcommand \conv {\operatorname*{conv}} %
\newcommand \dom {\operatorname*{dom}} %
\newcommand \dist {\operatorname*{dist}} %

\newcommand \grad {\nabla}

\newtheorem{theorem}{Theorem}
\newtheorem{lemma}[theorem]{Lemma}

\newtheorem{property}[theorem]{Property}
\newtheorem{proposition}[theorem]{Proposition}
\newtheorem{corollary}[theorem]{Corollary}

\theoremstyle{definition}
\newtheorem{definition}[theorem]{Definition}

 \usepackage{thmtools}
 \declaretheoremstyle[
notefont=\bfseries, notebraces={}{},
bodyfont=\normalfont\itshape,
headformat=\NAME \NOTE
]{nopar}

\title{Device Heterogeneity in Federated Learning: \\ 
A Superquantile Approach}

\author{
Yassine Laguel$^{1}$  \hspace{2em}
Krishna Pillutla$^{2}$ \hspace{2em}
J\'{e}r\^{o}me Malick$^{3}$  \hspace{2em}
Zaid Harchaoui$^{2}$ \\
\small{$^{1}$UGA, Lab. J. Kuntzmann, Grenoble, France}  \\
\small{
$^{2}$University of Washington, Seattle, USA} \\
\small{
$^{3}$CNRS, Lab. J. Kunztmann, Grenoble, France} 
}
\date{\vspace{-2em}}
\begin{document}
\maketitle

\begin{abstract}
We propose a federated learning framework to handle heterogeneous client devices which do not conform to the population data distribution. The approach hinges upon a parameterized superquantile-based objective, where the parameter ranges over levels of conformity. We present an optimization algorithm and establish its convergence to a stationary point. We show how to practically implement it using secure aggregation by interleaving iterations of the usual federated averaging method with device filtering. We conclude with numerical experiments on neural networks as well as linear models on tasks from computer vision and natural language processing. \end{abstract}

\section{Introduction} \label{sec:intro}
The proliferation of mobile phones, wearables
and edge devices has led to 
an unprecedented growth in the
generation of user interaction data. Systems which 
tap into the power of this rich data while respecting 
the privacy of users are geared to lead the next generation of intelligent applications and devices.
The leading distributed learning framework in this setting is federated learning~\citep{mcmahan2017communication}.

In federated learning, a number of client devices with privacy-sensitive
data collaboratively learn a machine learning model under the orchestration of a central server, while
keeping their data decentralized. 
This is achieved by pushing the actual 
computation to the devices while 
the server coordinates with the devices for  aggregation of model updates.
{\em Secure aggregation} ensures 
that no individual device's updates are known 
to either the server or other devices~\citep{bonawitz2017practical}.
Federated learning has found myriad applications ranging from
smartphone apps~\citep{yang2018applied} %
to healthcare~\citep{huang2019patient}.
A key feature of federated learning is 
statistical heterogeneity, i.e., 
client data distributions
are {\em not} identically distributed. 
Each user has unique
characteristics which are reflected in the data they generate. These characteristics are influenced by personal, cultural, regional and geographical 
factors. %
For instance, the varied use of language contributes to data heterogeneity
in a next word prediction task.

Vanilla federated learning and its de facto 
standard algorithm, \fedavg~\citep{mcmahan2017communication},
aim to fit a model 
to the population distribution of the devices available for training. 
While this approach works for users who conform to the population (e.g., trend followers),
it is liable to fail on individuals 
who do not conform to the population,
leading to poor user experience.
The goal of this work 
is to present a framework to 
improve the experience of 
these diversely non-conforming users
without sacrificing the good experience of conforming users.

Diversity of users leads to heterogeneity in the loss functions of the users,
which in turn manifests itself as heavy tails in 
the loss distribution over users.
Therefore, a natural approach to handling user heterogeneity consists in building an objective based on upper quantiles, in order to focus on the tail
distribution. 
In particular, we use the superquantile~\citep{rockafellar2000optimization} of the loss distribution (i.e.,\;the expectation over its upper tail). 

Training with a superquantile-based objective is not straightforward because of its inherent non-smoothness. 
It is worth emphasizing that any optimization issue can be exacerbated in the federated setting because of the constraints imposed by communication costs and privacy-preserving requirements. Here, we present an algorithm
to optimize a superquantile-based objective which overcomes these challenges. It enjoys a time and space complexity which is a constant multiple of the complexity of \fedavg{}.

\begin{figure*}[t]
\begin{center}
\begin{minipage}[c]{0.90\linewidth}

\begin{tikzpicture}[scale=0.15,
servernode/.style={ellipse, draw=black, fill=Goldenrod!10, very thick, minimum size=7mm},
devicenode/.style={rectangle, draw=black!60, very thick, minimum size=5mm, aspect=1.62},
secaggnode/.style={circle, draw=black, very thick, minimum size=1mm, fill=Lavender!10},
]

\tikzset{mark size=8pt} %
\tikzstyle{linesty1} = [cyan, very thick]
\tikzstyle{linesty2} = [BurntOrange, ultra thick, densely dotted]
\tikzstyle{linesty3} = [ForestGreen, ultra thick, dashdotted]
\tikzstyle{linesty4} = [OrangeRed, ultra thick, dashed]
\tikzstyle{linesty1s} = [cyan!75, very thick]
\tikzstyle{linesty2s} = [BurntOrange!75, ultra thick, densely dotted]
\tikzstyle{linesty3s} = [ForestGreen!75, ultra thick, dashdotted]
\tikzstyle{linesty4s} = [OrangeRed!75, ultra thick, dashed]
\tikzstyle{rectsty0} = [fill=gray!15, very thick]
\tikzstyle{rectsty1} = [fill=cyan!10, very thick]
\tikzstyle{rectsty2} = [fill=BurntOrange!10, very thick]
\tikzstyle{rectsty3} = [fill=ForestGreen!10, very thick]
\tikzstyle{rectsty4} = [fill=OrangeRed!10, very thick]

\node at (25,48) {\textbf{Trajectories of 
model parameters over time}};  %

\draw[thick, -] (5, 10) -- (45, 10); %
\draw[thick, -] (5, 10) -- (5, 33); %
\node at (24,7) {Iteration $\to$}; %
\node[rotate=90] at (3, 22) {Model parameters $\to$}; %

\draw[thin, -] (5, 1) rectangle (45, 5);
\draw[linesty1] (6,3) -- (10, 3);
\node  at (11.5,3) {$\theta_1$};

\draw[linesty2] (16,3) -- (20, 3);
\node  at (21.5,3) {$\theta_2$}; 

\draw[linesty3] (26,3) -- (30, 3);
\node  at (31.5,3) {$\theta_3$}; 

\draw[linesty4] (36,3) -- (40, 3);
\node  at (41.5,3) {$\theta_4$}; 

\draw [linesty1] plot [smooth] coordinates {(10.0, 27.0) (18.0, 31.5) (30.0, 25.5) (33.5, 26.2) (35.0, 25.7) (36.0, 26.2) (37.0, 25.7) (38.0, 26.2) (39.0, 25.7) (40.0, 26.0) };
\draw [linesty1] plot [only marks, mark=*] coordinates {(40,26)};
\node[outer sep=2pt] at (42.5,26) {$w_{\scriptscriptstyle\theta_1}$};

\draw [linesty2] plot [smooth] coordinates {(10.0, 27.0) (16.0, 29.0) (28.0, 20.0) (33, 22.5) (35.0, 21.8) (36.0, 21.9) (37.0, 22.1) (38.0, 21.8) (39.0, 22.1) (40.0, 22.0) };

\draw [linesty2] plot [only marks, mark=square*] coordinates {(40,22)};
\node[outer sep=2pt] at (42.5,22) {$w_{\scriptscriptstyle\theta_2}$};

\draw [linesty3] plot [smooth] coordinates {(10.0, 27.0) (15.0, 26.0) (26.0, 15.0) (33.0, 18.0) (34.4, 18.5) (36.0, 17.8) (37.0, 18.1) (38.0, 17.9) (39.0, 18.2) (40.0, 18.0) };

\draw [linesty3] plot [only marks, mark=diamond*] coordinates {(40,18)};
\node[outer sep=2pt] at (42.5,18) {$w_{\scriptscriptstyle\theta_3}$};

\draw [linesty4] plot [smooth] coordinates {(10.0, 27.0) (15.0, 27.0) (27.0, 12.0) (32.0, 13.5) (33.0, 14.0) (35.0, 13.7) (36.0, 14.2) (37.0, 13.8) (38.0, 14.1) (39.0, 13.9) (40.0, 14.05) };

\draw [linesty4] plot [only marks, mark=pentagon*] coordinates {(40,14)};
\node[outer sep=2pt] at (42.5,14) {$w_{\scriptscriptstyle\theta_4}$};

\fill (10,27) circle[radius=0.5];
\node[outer sep=2pt] at (8,27) {$w_0$};

\draw  [fill=RoyalPurple!50,draw=gray!50, opacity=0.25]  (25, 10) rectangle (27, 33) ;
\node[outer sep=2pt] at (26,9) {${t}$};
\node[draw=black!0] (highlight_rect) (26, 33) {}; %

\node (trainlabel) at (75,48) {\textbf{In iteration $t$ of training}};  %
\node[servernode] (trainserver) [below=1pt of trainlabel] {Server};

\node[secaggnode,scale=0.8] (sec_agg) [below=15pt of trainserver] {$\bm{+}$};

\node  at ([xshift=150pt,yshift=-10pt]sec_agg.north) {\scriptsize{Sec.Agg.}}; 

\node[devicenode, rectsty0] (traindev2) [below=20pt of sec_agg] {};
\node[devicenode, rectsty0] (traindev1) [left=of traindev2] {};
\node[devicenode, rectsty0] (traindev3) [right=of traindev2] {};

\node  [left=0.1pt of traindev1, yshift=3.5pt] {\scriptsize{selected}};
\node  [left=0.1pt of traindev1, yshift=-3.5pt] {\scriptsize{devices}};

\draw[->,linesty4s] ([xshift=30pt]traindev1.north) to [bend left=10] ([xshift=10pt,yshift=-30pt]sec_agg.west) ;

\draw[->,linesty3s] ([xshift=10pt]traindev1.north) to [bend left=20] ([xshift=5pt,yshift=-10pt]sec_agg.west) ;

\draw[->,linesty2s] ([xshift=-10pt]traindev1.north) to [bend left=30] ([xshift=0pt,yshift=5pt]sec_agg.west) ;

\draw[->,linesty1s] ([xshift=-30pt]traindev1.north) to [bend left=40] ([xshift=3pt,yshift=20pt]sec_agg.west) ;

\draw[->,linesty4s] ([xshift=30pt]traindev2.north) to [bend right=20] ([xshift=22.5pt,yshift=0pt]sec_agg.south) ;

\draw[->,linesty3s] ([xshift=10pt]traindev2.north) to [bend right=10] ([xshift=7.5pt,yshift=0pt]sec_agg.south) ;

\draw[->,linesty2s] ([xshift=-10pt]traindev2.north) to [bend left=10] ([xshift=-7.5pt,yshift=0pt]sec_agg.south) ;

\draw[->,linesty1s] ([xshift=-30pt]traindev2.north) to [bend left=20] ([xshift=-22.5pt,yshift=0pt]sec_agg.south) ;

\draw[->,linesty4s] ([xshift=30pt]traindev3.north) to [bend right=40] ([xshift=-3pt,yshift=20pt]sec_agg.east) ;

\draw[->,linesty3s] ([xshift=10pt]traindev3.north) to [bend right=30] ([xshift=0pt,yshift=5pt]sec_agg.east) ;

\draw[->,linesty2s] ([xshift=-10pt]traindev3.north) to [bend right=20] ([xshift=-5pt,yshift=-10pt]sec_agg.east) ;

\draw[->,linesty1s] ([xshift=-30pt]traindev3.north) to [bend right=10] ([xshift=-10pt,yshift=-30pt]sec_agg.east) ;

\draw[->,linesty4s] ([xshift=30pt]sec_agg.north) to [bend right=20] ([xshift=22.5pt,yshift=0pt]trainserver.south) ;

\draw[->,linesty3s] ([xshift=10pt]sec_agg.north) to [bend right=10] ([xshift=7.5pt,yshift=0pt]trainserver.south) ;

\draw[->,linesty2s] ([xshift=-10pt]sec_agg.north) to [bend left=10] ([xshift=-7.5pt,yshift=0pt]trainserver.south) ;

\draw[->,linesty1s] ([xshift=-30pt]sec_agg.north) to [bend left=20] ([xshift=-22.5pt,yshift=0pt]trainserver.south) ;

\node (testlabel) at (75,20) {\textbf{At test time}};  %
\node[servernode] (testserver) [below=1pt of testlabel] {Server};

\node[devicenode,rectsty2] (testdev3) [below=of testserver] {};
\node[devicenode,rectsty3] (testdev2) [left=of testdev3] {};
\node[devicenode,rectsty1] (testdev1) [left=of testdev2] {};
\node[devicenode,rectsty1] (testdev4) [right=of testdev3] {};
\node[devicenode,rectsty4] (testdev5) [right=of testdev4] {};

\draw[->,linesty1] (testserver.south) -- ([yshift=10pt,xshift=5pt]testdev1.north);
\draw[->,linesty3] (testserver.south) -- ([yshift=5pt]testdev2.north);
\draw[->,linesty2] (testserver.south) -- (testdev3.north);
\draw[->,linesty1] (testserver.south) -- ([yshift=5pt]testdev4.north);
\draw[->,linesty4] (testserver.south) -- ([xshift=-5pt,yshift=10pt]testdev5.north);

\node  [below=-4pt of testdev3, label={[align=center]below:\scriptsize{test devices select their level of conformity $\theta$}}] {};

\draw[-latex,thin, dashed] (26, 33.2) to [bend left=10] (trainlabel.west);

\end{tikzpicture} \end{minipage}
\end{center}
    \caption{\small{
    Schematic summary of the \newfl{} framework.~~\textbf{Left}: The server maintains multiple models $w_{\theta_j}$, one for each level of conformity\;$\theta_j$. 
    \textbf{Right-top}:
    In each iteration of the training process, 
    the selected devices participate in training 
    {\em each} model $w_{\theta_j}$.
    The updates proposed by individual devices are 
    the combined with the use of secure aggregation to
    update the model parameters at the server. 
    \textbf{Right-bottom}:
    At test time, each test device is provided with a tuning 
    knob which lets the user select their level of conformity
    $\theta$, and they are served the corresponding
    model $w_\theta$.
    Note that the conformity level cannot directly be
    measured due to data privacy restrictions.
    For simplicity, this schematic omits 
    the device filtering step, which is explained further in\;Figure\;\ref{fig:main:schematic_of_algo}.
    }}
    \label{fig:main:schematic_of_framework}
\end{figure*}
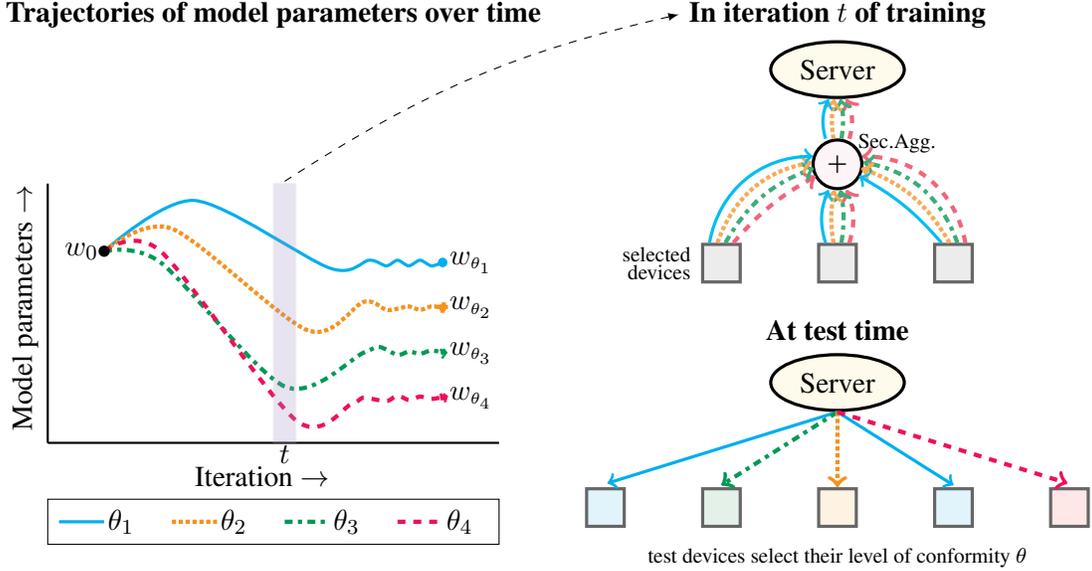

\myparagraph{Contributions}
We make the following contributions.
\begin{enumerate}[label=(\alph*),topsep=0.0ex,partopsep=0.0ex,itemsep=0.0ex,leftmargin=1.5em]

\item \textit{The \newfl{} Framework}\footnote{
pronounced as {\em Simplicial-FL}}:
We introduce the \newfl{} framework,
summarized in Figure~\ref{fig:main:schematic_of_framework},
to handle 
heterogeneity of client data distributions. 
The framework relies on a superquantile-based objective 
parameterized by the conformity level, 
which is a scalar summary of
how closely a device conforms to the 
population.

\item \textit{Optimization Algorithm}:
We present an algorithm to 
optimize the \newfl{} objective and establish 
its almost sure convergence to a stationary point.
We show how to implement a practical variant of the algorithm
with the use of secure aggregation such that it has the same
per-iteration communication 
cost as \fedavg.
See Figure~\ref{fig:main:schematic_of_algo} for a schematic summary of the algorithm.

\item \textit{Numerical Simulations}:
We demonstrate the breadth of our framework with numerical simulations using 
neural networks and linear models, on tasks including image classification, language modeling and sentiment analysis based on public datasets. The simulations demonstrate superior performance of \newfl{} 
on the upper quantiles of the error on test devices, while being competitive with vanilla federated learning on the mean error.
We have released a 
Python package with scripts to reproduce all simulations~\cite{simplicial_fl_repo}.

\end{enumerate}

\myparagraph{Outline}
Section~\ref{sec:setting} describes the setting
and precisely defines conformity.
Section~\ref{sec:method} describes the \newfl{} framework
and the training objective.
Section~\ref{sec:algos} describes a provably convergent
algorithm to optimize the \newfl{} objective, and presents a practical variant of the algorithm. 
Section~\ref{sec:expt} presents numerical simulations of the proposed method.
Section~\ref{sec:related} surveys related work.
The supplement contains
a rigorous presentation of the material with proofs.

\begin{figure*}[t!]
    \centering %
    \begin{subfigure}[b]{0.47\linewidth}
    		\centering
        \adjincludegraphics[width=0.99\linewidth, trim={0 0 0 0},clip=true]{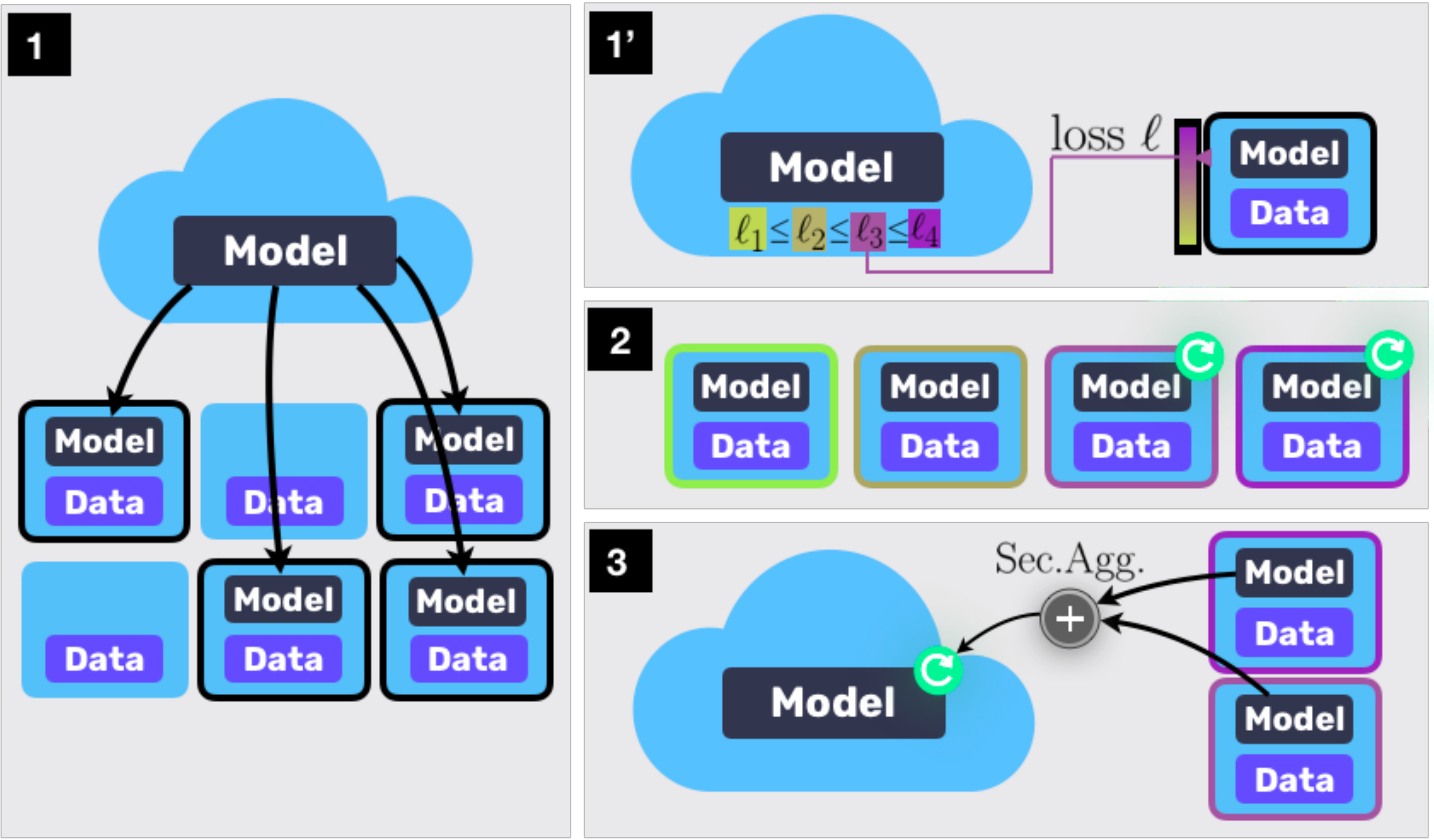}
     \caption{\small{\newfl.}}
    \end{subfigure}
    \hspace{0.5mm}
    \begin{subfigure}[b]{0.47\linewidth}
    		\centering
        \adjincludegraphics[width=0.99\linewidth, trim={0 0 0 0},clip=true]{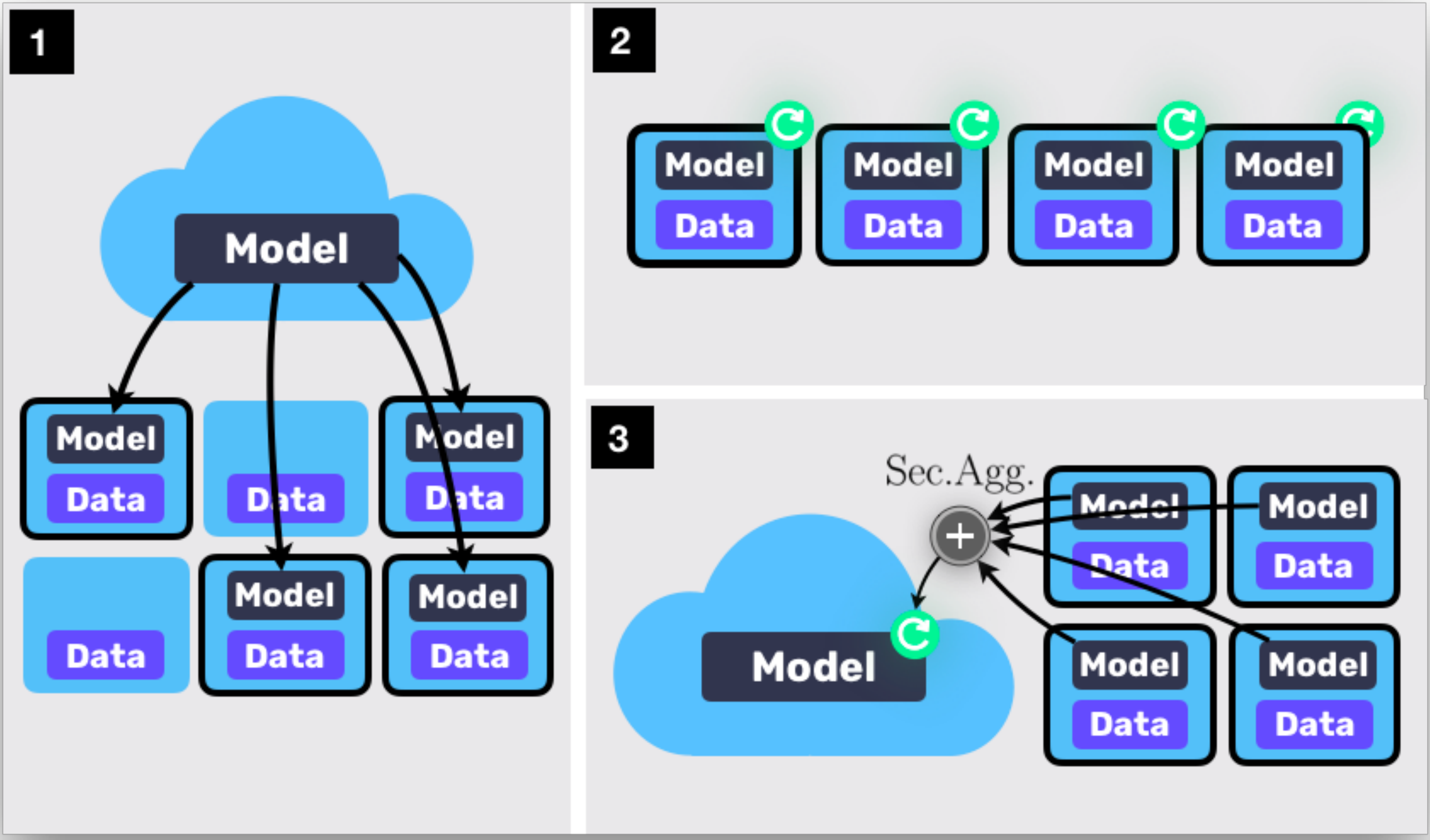}
     \caption{\small{\fedavg.}}
    \end{subfigure}
     \caption{\small{
     Key steps of each round of 
     the optimization algorithm (Algo.~\ref{algo:main:fed:proposed}) 
     of the \newfl{} framework for a fixed level
     $\theta$ of conformity, contrasted with \fedavg.
     Both algorithms consist of the following steps
     (note difference in step 1'). 
     \textbf{Step 1}:
     Server selects $m$ client devices and broadcasts 
     the global model to each selected device. 
     \textbf{Step 1' (\newfl{} only)}: Each selected device computes the 
     loss (a scalar) incurred by the global model on its local data and sends it to the server. Based on these losses,
     the server computes a threshold loss. It only keeps
     devices whose losses are larger than this threshold,
     and un-selects the other devices.  
     \textbf{Step 2}: Each selected device 
     computes an update to the server model based on its local
     data.  
     \textbf{Step 3}: Updates from selected devices are securely aggregated 
     to update the server model.
     }}
     \label{fig:main:schematic_of_algo}
\end{figure*}

\section{Problem Setting} \label{sec:setting}
Federated learning consists of a number of heterogeneous 
client devices which collaboratively 
train a machine learning model. The model 
is then deployed on all client devices, including those not seen 
during training.
We first review the training setup,
followed by test devices.

Concretely, suppose that we have $N$ client devices available for training.
We characterize each training device by a probability distribution
$q_k$ over some data space $\mathcal{Z}$ and a weight
$\alpha_k > 0$. We assume that
the data on device $k$ are distributed i.i.d. according to $q_k$
and $\sum_{k=1}^N \alpha_k = 1$ w.l.o.g.

We measure the loss incurred by a model $w \in \reals^d$ 
on a device with data distribution $q$ by 
\[
    F(w; q) := \expect_{\xi \sim q}
        [f(w; \xi)] \,,
\]
where $f : \reals^d \times \mathcal{Z} \to \reals$ is given.
We use $F_k(w) := F(w; q_k)$ 
to denote the loss on training device $k$.

We are interested in supervised machine learning, where $\xi \in \mathcal{Z}$ is an input-output pair $\xi = (x, y)$.
The function $f$ is of the form $\ell(y, \varphi(x ; w))$, where $\varphi(x; w)$ 
makes a prediction on input $x$
under model $w$ using, e.g., a neural network, and $\ell$ is a loss function such as the logistic loss.
The weight $\alpha_k$ is set proportional to the amount of data on device $k$.

\myparagraph{Test Devices and Conformity}
In this work, we consider
``test'' devices, unseen during training, whose distribution
can be written as a mixture of the training distributions. 
We define a mixture $p_\pi$
with weight $\pi \in \Delta^{N-1}$ %
as
\begin{align*}
    p_\pi := \sum_{k=1}^N \pi_k q_k\,,
\end{align*}
where $\Delta^{N-1}$ is the probability simplex in $\reals^N$.
Under this notation, the training distribution
is $p_\alpha$.
We now define {\textit{conformity}} of a 
mixture
to the training distribution.

\begin{definition} \label{def:setting:conformity}
    The {\em conformity} $\conf(p_\pi)$
    of a mixture distribution $p_\pi$ with weight $\pi$
    to the training distribution $p_\alpha$ 
    is defined as $\min_{k \in [N]} \alpha_k / \pi_k$.
    The conformity of a client device refers to the conformity of
    its data distribution.
\end{definition}

For every mixture $p_\pi$, we have that 
$0 < \conf(p_\pi) \le 1$.
A mixture distribution $p_\pi$
with $\conf(p_\pi)=\theta$ must satisfy
$\pi_k \le \alpha_k/\theta$ for each $k$.
Since $\sum_k \pi_k = 1$, we also get that 
$\pi_k \ge \max\{0, \alpha_k - (1-\theta)\}$.
We do not directly impose a lower bound on 
$\pi_k$ because it is not realistic to assume that 
the distribution on a test device must necessarily contain a component of 
every training distribution $q_k$.

\myparagraph{Interpretation}
Assuming that the training devices are a representative sample
of the population, every device's distribution can be 
well-approximated by a mixture $p_\pi$ for some $\pi \in \Delta^{N-1}$. 
Then, the conformity of a device is 
{\em a scalar summary of how close it is to the population}.
A test device with conformity $\theta \approx 1$ closely 
conforms to the population. Then, a model 
trained on the population 
$p_\alpha$ can be expected
to have a high predictive power, and the user experience 
on such a device is likely good.
In contrast, a test device with conformity 
$\theta \approx 0$ would be vastly different from the 
population $p_\alpha$. 
Here, the predictive power of a model trained on $p_\alpha$ could be arbitrarily poor. 
There is a trade-off between the fitting to the population distribution and supporting 
non-conforming test devices, i.e., 
those with 
distribution $p_\pi$ for small $\conf(p_\pi)$.
The conformity level $\theta$
presents a natural way to
encapsulate this tradeoff in a scalar
parameter. That is, given a 
conformity $\theta \in (0, 1)$,
we choose to only support test devices 
with distribution $p_\pi$ satisfying $\conf(p_\pi) \ge \theta$.

\myparagraph{Quantile and Superquantile}
Before proceeding, we recall that
the $(1-\theta)$-superquantile~\citep{rockafellar2000optimization} 
of a real-valued random variable $X$ is defined as 
\[
    S_X(\theta) := 
    \inf_{\eta \in \reals} \left\{
        \eta + \frac{1}{\theta} \, \expect(X - \eta)_+
    \right\} \,,
\]
where $(\rho)_+ := \max\{0, \rho\}$.
The right hand side is minimized by the 
corresponding quantile $\eta^\star = Q_X(\theta)$, 
\[
    Q_{X}(\theta) := 
    \inf\left\{ x \in \reals \, : \, \prob(X > x) \le \theta \right\} \,.
\]
When $X$ is continuous, the superquantile has the alternate 
representation
$S_X(\theta) = \expect\big[X \,\big|\, X > Q_{X}(\theta)\big]$,
as the average of $X$ 
above its $(1-\theta)$-quantile. 
The superquantile is, therefore, 
{\em a measure of the upper tail of $X$}.

\section{The \newfl{} Framework} \label{sec:method}

We now present the \newfl{} framework to 
(a) maintain good predictive power on 
high-conformity devices, and,
(b) improve the predictive power on 
low-conformity devices. 

The \newfl{} framework supplies each test 
device with a model appropriate to its 
conformity. In particular, given a discretization
$\{\theta_1, \ldots, \theta_r\}$ of $(0, 1]$, 
\newfl{} maintains $r$ models, one for 
each level $\theta_j$ of conformity. 
Owing to privacy restrictions, 
the local data is not allowed to leave a 
device, and hence, the conformity
of a test device cannot be measured.
Instead, we allow each test device to
tune their conformity.

In order to train a model for a given level $\theta$ of 
conformity, we aim to do well on {\em all}
mixtures $p_\pi$ with $\conf(p_\pi)\ge \theta$.
Therefore, we consider the optimization problem 
\begin{align} \label{eq:method:obj:minmax}
    \min_{w \in \reals^d} \left[
        F_\theta(w) := \max_{\pi \in \mathcal{P}_\theta}
        F(w; p_\pi) 
    \right]\,,
\end{align}
where $\mathcal{P}_\theta := \{\pi \,: \, \conf(p_\pi) \ge \theta\}$.
Equivalently, we have,
\[
   \mathcal{P}_\theta = \left\{
      \pi \in \Delta^{N-1} \, : \, 
        \pi_k \le \frac{ \alpha_k}{\theta}
      \,\,\forall \, k \in [N]
   \right\} \,.
\]

First, we formalize
the duality of $F_\theta$ as
a {superquantile}.
\begin{property} \label{property:main:duality}
   We have
   $F_\theta(w) = \min_{\eta \in \reals} \overline F_\theta(w, \eta)$,
   for any $\theta \in (0, 1)$,
   where $\overline F_\theta: \reals^d \times \reals \to \reals$ is given by
   \begin{align} \label{eq:main:overline_theta}
        \overline F_\theta(w, \eta) := 
            \eta + \frac{1}{\theta} 
            \sum_{k=1}^N \alpha_k 
            \big(F_k(w) - \eta \big)_+  \,.
   \end{align}
\end{property}

The optimal $\eta$ above is the
$(1-\theta)$-weighted quantile of $F_k(w)$ with weight $\alpha_k$ for $k=1,\ldots, N$.
The next property shows that the superquantile preserves convexity.

\begin{property} \label{property:main:cvx}
     If each $F_k$ is convex,
    then for any $\theta \in (0, 1)$,
    (a) $\overline F_\theta$ 
    is convex on $\reals^d \times \reals$, and,
    (b) $F_\theta$ is convex on $\reals^d$.

\end{property}

Motivated by these observations, we consider 
in lieu of \eqref{eq:method:obj:minmax}:
\begin{align} \label{eq:method:obj:minmin}
    \min_{w\in\reals^d, \eta \in \reals} \overline F_\theta(w, \eta) \,.
\end{align}

\begin{figure*}[t!]
    \centering %
    \begin{subfigure}[b]{0.3\linewidth}
    		\centering
        \adjincludegraphics[width=0.99\linewidth, trim={0 0 0 0},clip=true]{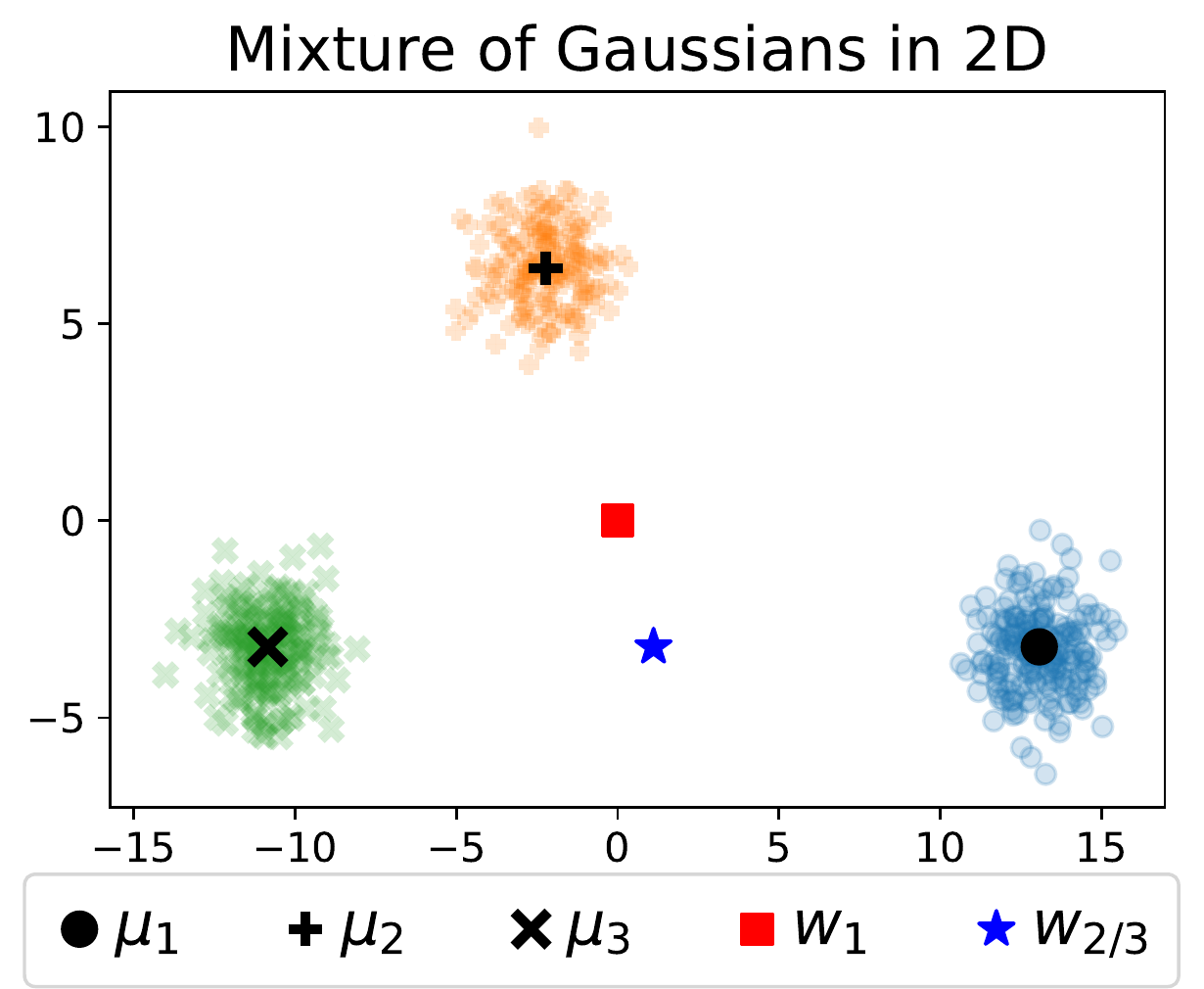}
    \end{subfigure}
    \hspace{10mm}
    \begin{subfigure}[b]{0.24\linewidth}
    		\centering
        \adjincludegraphics[width=0.99\linewidth, trim={0 0 0 0},clip=true]{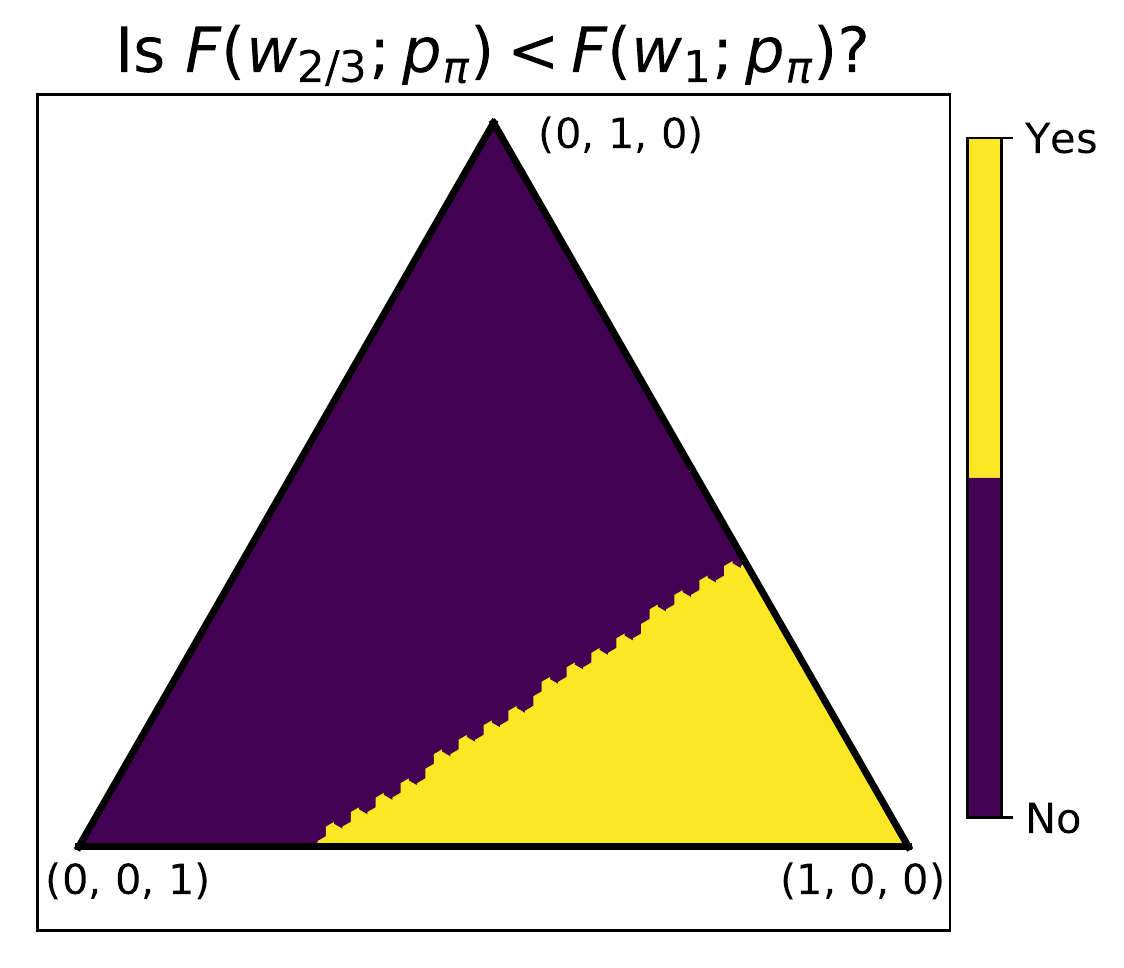}
        \vspace{2.1mm}
    \end{subfigure}
    \hspace{3mm}
    \begin{subfigure}[b]{0.32\linewidth}
    		\centering
        \adjincludegraphics[width=0.95\linewidth, trim={0 0 0 0},clip=true]{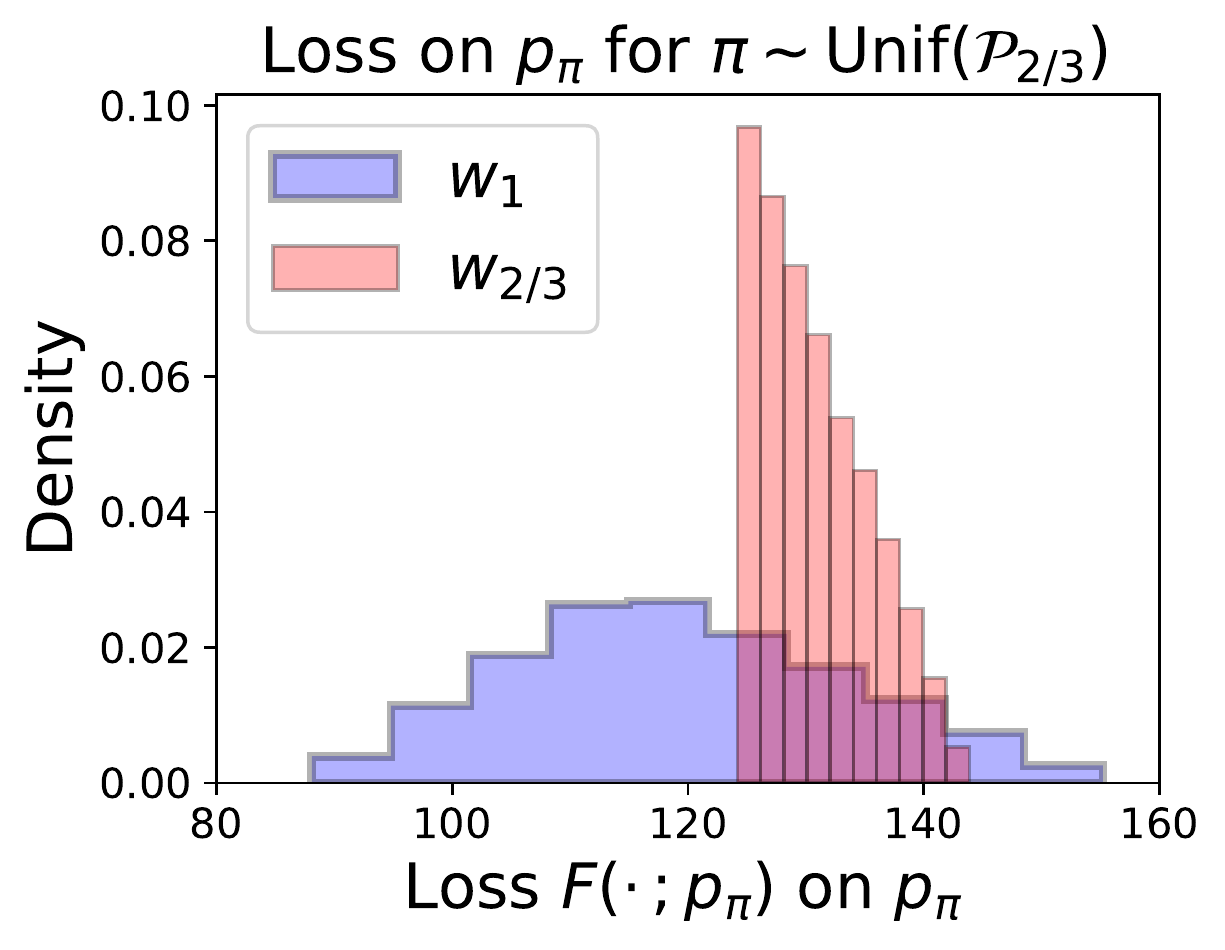}
        \vspace{1mm}
    \end{subfigure}

     \caption{\small{
     Illustration of \newfl{} with a uniform mixture of Gaussians.
     \textbf{Left}: Positions in $\reals^2$ 
        of the means $\mu_1, \mu_2, \mu_3$ of Gaussians 
        $q_1, q_2, q_3$ resp., 
        the vanilla federated learning model $w_1$, and the
        \newfl{} model $w_{2/3}$ at conformity $\theta=2/3$.
    \textbf{Middle}: Comparison of the loss $F(\cdot\,; p_\pi)$
        for each possible mixture $p_\pi$ with weights
        $\pi = (\pi_1, \pi_2, \pi_3)$.
    \textbf{Right}: Histogram of losses 
        $F(\cdot; p_\pi)$ 
        for $p_\pi$ drawn uniformly from the set of all 
        mixtures of $q_1, q_2, q_3$ with conformity at least
        $\theta = 2/3$.
        }}
     \label{fig:method:gauss-mix}
\end{figure*}

\myparagraph{Illustration}
We now illustrate the objective on a simple example
of a mixture of Gaussians. 
Consider a
mixture of $N=3$ Gaussian distributions
in $\reals^2$, with 
uniform weights ($\alpha_k=1/N$), identity covariance
and respective means $\mu_1, \mu_2, \mu_3$ which form a
scalene triangle -- see Figure~\ref{fig:method:gauss-mix}.
We assume that each distribution represents a training device.

Consider the task of mean estimation where
$f(w; \xi) = \normsq{\xi - w}$
so that $F(w; p)$ is minimized by 
the mean of $p$.
Suppose in our toy federated learning scenario that
a model $w$ trained on the 3 available 
training devices and is deployed on a test device
with distribution $p_\pi$.

Vanilla federated learning, 
which is a special case of the \newfl{} framework
with conformity $\theta = 1$,
aims to minimize $F(\cdot\,; p_\alpha)$
over the training distribution $p_\alpha$.
The minimizer $w_1$ of the 
loss $F(\cdot\,; p_\alpha)$ on the training distribution
is simply the mean
$w_1 := (\mu_1 + \mu_2 + \mu_3) / 3$.

Now consider a conformity level of $\theta=2/3$. In this 
case, a simple calculation shows that the \newfl{} 
objective is a piecewise quadratic, which is 
minimized at the midpoint of the longest side of 
the triangle formed by $\mu_1, \mu_2, \mu_3$.
In the example of Figure~\ref{fig:method:gauss-mix}, this is 
$w_{2/3}=(\mu_1+ \mu_3)/2$.

Next, consider the set $\mathcal{P}_{2/3}$ of all mixture weights $\pi$ 
such that $\conf(p_\pi)\ge 2/3$.
We see from Figure~\ref{fig:method:gauss-mix} (middle) that
there are mixtures for which $w_{2/3}$ is better than $w_1$ and vice-versa.
However, from the histogram of losses in Figure~\ref{fig:method:gauss-mix},
we see that the {\em worst loss} $F(\cdot\,; p_\pi)$ over all 
such mixtures is lower for the \newfl{} model $w_{2/3}$.
In practical terms, \newfl{} presents an improvement on devices
with the worst user experience. Moreover, by optimizing the superquantile, 
\newfl{} aims for good performance on {\em all}  
test devices with a given conformity, {\em irrespective of their 
distribution}. 
Note that while we use a uniform distribution in 
the illustration of Figure~\ref{fig:method:gauss-mix} (right), this
distribution is unknown in practice.

\section{Algorithms and Convergence}  \label{sec:algos}

We consider optimization algorithms to solve 
Problem~\eqref{eq:method:obj:minmin} for a fixed 
conformity level $\theta$.
We start with a meta-algorithm based
on the technique of alternating minimization
and 
then present a concrete implementation of it adapted to the engineering constraints of the federated setting.

\begin{algorithm}[tb]
	\caption{Alternating Minimization Meta-Algorithm}
	\label{algo:a:fed:proposed:am-raw}
\begin{algorithmic}[1]
		\Require Function $\overline F: \reals^d \times \reals \to \reals$, initial iterate $w_0 \in \reals^d$,
		    positive inexactness sequence $(\epsilon_t)_{t=0}^\infty$
	    \For{$t=0, 1,2, \cdots$}
	    	\State $\eta_t \in \argmin_{\eta \in \reals} 
	    	    \overline F(w_t, \eta)$
	    	 \State $w_{t+1} \approx \argmin_{w \in \reals^d} \overline F(w, \eta_t)$ such that
	    	 \eqref{eq:main:am:inexactness} holds
	    \EndFor
\end{algorithmic}
\end{algorithm}

\myparagraph{Meta-Algorithm}
We start by assuming that
all devices participate at all times.
An inexact
alternating minimization meta-algorithm
is given in
Algorithm~\ref{algo:a:fed:proposed:am-raw}.
It alternates updates of $w$ and $\eta$,
where the $\eta$-step can be performed in closed form.
For the $w$-step, we consider
for some $\epsilon_t > 0$ 
the inexactness
criterion\footnote{
    We use $\sigma(w_t)$ to denote the sigma field
    generated by $w_t$.
 }
\begin{align} \label{eq:main:am:inexactness}
    \expect\left[
        \overline F(w_{t+1}, \eta_t) 
        | \sigma(w_t) \right]
    - \min_w \overline F(w, \eta_t) \le \epsilon_t  \,.
\end{align}
The template in Algorithm~\ref{algo:a:fed:proposed:am-raw}
can be concretely instantiated with
a stochastic optimization algorithm such as SGD 
to satisfy the inexactness bound.

Note that $\overline F_\theta$ is not smooth\footnote{
We say $f: \reals^n \to \reals$ is $L$-smooth if it is continuously differentiable and $\grad f$ is $L$-Lipschitz w.r.t. $\norm{\cdot}_2$.
}
 owing to the non-smoothness of 
$(\rho)_+$. To show convergence, we consider a
smooth surrogate
$\overline F_{\theta, \nu}$ 
of $\overline F_\theta$ defined for $\nu > 0$ as
\begin{align}
    \label{eq:main:setup:smoothing}
    \overline F_{\theta, \nu}(w, \eta) &:=
    \eta + \frac{1}{\theta} \sum_{k=1}^N \alpha_k g_\nu\big(F_k(w) - \eta \big) \,, \text{where,}
    \\
    \label{eq:main:setup:smoothing:relu}
    g_\nu(\rho) &:= 
    \begin{cases}
        \nu / 2 \,, & \text{ if } \rho \le 0 \,, \\
        \rho^2/(2\nu) + \nu/2 \,, & \text{ if } 0 < \rho \le \nu \,, \\
        \rho\,, & \text{ if } \rho > \nu   \,.
    \end{cases}
\end{align}
i.e., $g_\nu(\rho)$ is a smoothing of $(\rho)_+$.
It is known that $\overline F_{\theta, \nu}$ uniformly
approximates $\overline F_{\theta}$ to $\nu/(2\theta)$
and enjoys the same convexity properties as $\overline F_\theta$.
Analogous to $F_\theta$, we define
\begin{align}
    F_{\theta, \nu}(w) := \min_{\eta \in \reals} 
        \overline F_{\theta, \nu}(w, \eta) \,,
\end{align}
where the minimization over $\eta$ can be performed in 
closed form again.
The next proposition shows the convergence of
Algorithm~\ref{algo:a:fed:proposed:am-raw} provided the inexactness in the $w$-step satisfies 
$\epsilon_t = o(t^{-1})$.
Note that the stationary point guarantee does 
not require convexity.

\begin{proposition} \label{prop:main:am:converge:asymp}
    Fix $\theta \in (0, 1)$ and $\nu = 2\delta\theta$ for some  $\delta > 0$.
    Suppose each $F_k$ is $B$-Lipshitz and $L$-smooth.
    Consider Algorithm~\ref{algo:a:fed:proposed:am-raw}
    with inputs $\overline F_{\theta, \nu}$
    and a positive sequence $(\epsilon_t)$
    such that $\sum_{t=0}^\infty \epsilon_t < \infty$.
    Then, the iterates $(w_t, \eta_t)$
    generated by Algorithm~\ref{algo:a:fed:proposed:am-raw}
    almost surely satisfy\footnote{
    The notation $\partial$ refers to the 
    Clarke subdifferential~\cite{clarke1990optimization}
    --- see
    Appendix~\ref{sec:a:technical_lemmas}.
    }\textsuperscript{,}\footnote{
    We use
    $\dist(a, S) := \inf_{z \in S} \norm{z-a}$.
    }:
    \begin{enumerate}[label=(\alph*),nolistsep,topsep=0em, leftmargin=1.6em]
        \item $\grad_{w, \eta} \overline F_{\theta, \nu}(w_t, \eta_t) \to 0$, ~and,
        \label{prop:main:am:part:ncvx}
        \vspace*{0.8ex}
        \item $\dist(0, \partial F_{\theta, \nu}(w_t)) \to 0$.
        \label{prop:main:am:part:subdiff}
    \end{enumerate}
    Furthermore, if each $F_k$ is convex, 
    then almost surely,
    \begin{enumerate}[resume, label=(\alph*), nolistsep, topsep=0em, leftmargin=1.6em]
        \item $F_{\theta, \nu}(w_t) \to \min F_{\theta, \nu}$, ~and,
        \label{prop:main:am:part:cvx1}
        \vspace*{0.7ex}
        \item $\limsup_{t \to \infty} 
         F_\theta(w_t) \le 
         \min F_\theta + \delta$.
         \label{prop:main:am:part:cvx2}
    \end{enumerate}
\end{proposition}
The proof
is given in Appendix~\ref{sec:a:algo}.
The proofs of parts~\ref{prop:main:am:part:ncvx},
\ref{prop:main:am:part:cvx1} and~\ref{prop:main:am:part:cvx2}
are elementary, while part~\ref{prop:main:am:part:subdiff}
is more technical.

\begin{algorithm}[tb]
	\caption{The \newfl{} algorithm}
	\label{algo:main:fed:proposed}
\begin{algorithmic}[1]
		\Require $N$ devices $\{(q_k, \alpha_k)\}_{k \in [N]}$, 
	   		number of local updates $n_\mathrm{local}$, 
	   		learning rate sequence $(\gamma_t)$,
	   		devices per round $m$,
	   		initial iterate $w_0$, conformity level $\theta \in (0, 1)$
	   	\Statex \textbf{Server executes:}
	    \For{$t=1,2, \cdots$}
	    	\State Sample devices $S_t \sim \mathrm{Unif}([N])^m$
	    	\State Broadcast $w_t$ to each device $k \in S_t$
	    	\State Each\;$k \in S_t$\;computes\;$F_k(w_t)$\;and\;sends\;to\;server
	    	\State $\eta_t \gets \textit{Quantile}\left(1-\theta,  \big(F_k(w_t), \alpha_k)\big)_{k\in S_t}\right)$
	    	\State Filter out $S_t' = \{k \in S_t \, : \, F_k(w_t) \ge \eta_t\}$ 
	    	\label{line:main:filtering}
	    	\ParFor{each device $k \in S_t'$}
	    		\State $w_{k, t} \gets $ \Call{LocalUpdate}{$k, w_t$}
	    	\EndParFor %
	   		\State $w_t \gets \textit{SecureAggregate}\left( \{(w_{k,t}, \alpha_{k} )\}_{k \in S_t'}\right)$
	    \EndFor
	    \Statex
	    \Function{LocalUpdate}{$k, w, \eta$}
	    \Comment Run on device $k$
	    
	        \For{$i = 1, \cdots, n_\mathrm{local}$}
	    		\State Update $w \gets w - \gamma_t \grad f(w;\xi_i)$ using $\xi_i \sim q_k$
	    	\EndFor
	    	\Return $w$
	    \EndFunction
\end{algorithmic}
\end{algorithm}

\myparagraph{Practical Implementation}
To obtain a practical algorithm, we
modify, without proof, Algorithm~\ref{algo:a:fed:proposed:am-raw} 
to respect system-level constraints of federated learning at scale.

Firstly, we estimate the $\eta$-step of Algorithm~\ref{algo:a:fed:proposed:am-raw}
from a sample of devices. This is because devices are unavailable when offline, and device
availability typically follows a diurnal pattern.
Difficulties %
caused by the bias of quantile estimators makes
the analysis of this scheme beyond 
this work.

Secondly, we execute the $w$-step as a single round of \fedavg.
Communication is often the bottleneck in the federated setting, while local computation is relatively cheap.
This heuristic allows us to make more progress at a
lower communication cost than strictly following 
Proposition~\ref{prop:main:am:converge:asymp}, i.e,
solving the $w$-step with decreasing suboptimality.

Lastly, we perform the $\eta$-step and 
the $w$-step using the same sample of devices.
With these modifications in place, the resulting algorithm is given in Algorithm~\ref{algo:main:fed:proposed}.
As illustrated in Figure~\ref{fig:main:schematic_of_algo},
Algorithm~\ref{algo:main:fed:proposed}
may be viewed as 
an augmentation of \fedavg{} with an additional step of filtering devices (Line~\ref{line:main:filtering}).
In particular, the aggregation of model parameters 
can be performed using secure aggregation.

As secure aggregation 
dominates the running time in the
federated setting due to its
expensive communication,
Algorithm~\ref{algo:main:fed:proposed} has the 
same per-iteration complexity as \fedavg{}.

\myparagraph{Privacy} 
Algorithm~\ref{algo:main:fed:proposed} 
reveals neither the data nor the model parameters of the client devices, the latter via the use of 
secure aggregation. However, the algorithm, as it is currently stated, requires 
each selected client devices to reveal its loss (a scalar) on the current model to the server. 
Appendix~\ref{sec:a:algo:quantile} presents a variant of  Algorithm~\ref{algo:main:fed:proposed} 
which ensures the same privacy-preservation of \fedavg{} at the cost of extra communication. This is achieved by implementing the quantile calculation using 
multiple secure aggregation calls.

\section{Numerical Simulations} \label{sec:expt}
\begin{table*}[t!]
\caption{\small{Dataset description and statistics.}}
\label{table:expt:dataset:descr}
\begin{center}
\begin{adjustbox}{max width=0.9\linewidth}
\begin{tabular}{lccccccc}
\toprule
Dataset & Task & \#Classes & 
\begin{tabular}{c} \#Train \end{tabular} & 
\begin{tabular}{c} \#Test \end{tabular} & \multicolumn{3}{c}{\#Points per train device} \\
 & & & Devices & Devices &  Median & Min & Max  \\
\midrule

EMNIST & \begin{tabular}{c} Image Classification \end{tabular} & 
    62 & 865  & 865 & 179 & 101 & 447 \\
    
Sent140 & \begin{tabular}{c} Sentiment Analysis \end{tabular} &
    2 & 438  & 439 & 69 & 51 & 549 \\
    
Shakespeare & \begin{tabular}{c} Character-level Language Modeling \end{tabular} & 
    53 & 544  & 545 & 1288 & 101 & 66963 \\

\bottomrule
\end{tabular}
\end{adjustbox}
\end{center}
\vskip -0.1in
\end{table*}

We now experimentally compare the performance of \newfl{} with \fedavg{}. 
The simulations were 
implemented in Python using automatic differentiation provided by PyTorch,
while the data was preprocessed using LEAF~\citep{caldas2018leaf}.
Full details of the simulations are given in Appendix~\ref{sec:a:expt}. A software package implementing 
the proposed algorithm and scripts to reproduce experimental results can be found in~\cite{simplicial_fl_repo}.

\myparagraph{Datasets, Tasks and Models}
We consider three tasks, whose 
datasets and numbers of train and test 
devices are described in Table~\ref{table:expt:dataset:descr}.
We weigh training device $k$ by 
the number of datapoints on the device.
All models were trained with the (multinomial, if applicable) logistic loss and evaluated with 
the misclassification error.

\begin{enumerate}[label=(\alph*),nolistsep,topsep=0em, leftmargin=1.6em]
\item{\textit{Character Recognition}:}
We use the EMNIST dataset~\citep{cohen2017emnist}, 
where the input $x$ is a $28\times 28$ grayscale image of a handwritten character
and the output $y$ is its identification (0-9, a-z, A-Z). 
Each device is a writer of the character $x$.
We use a linear model and 
a convolutional neural network (ConvNet).

\item{\textit{Sentiment Analysis}:}
We learn a binary classifier over the Sent140 dataset~\citep{go2009twitter}
where the input $x$ is a Twitter post
and the output $y=\pm1$ is its sentiment.
Each device corresponds to a distinct Twitter user.
The linear model is built on the average of the GloVe embeddings~\cite{pennington2014glove} 
of the words of the post, while the non-convex model
is a Long Short Term Memory model~\citep{hochreiter1997long} built on the GloVe embeddings. We refer to the latter as ``RNN''.

\item{\textit{Character-Level Language Modeling}:}
We learn a character-level language model over the 
Complete Works of \citet{shakespeare},
formulated as a multiclass classification problem, where 
the input $x$ is a window of 20 characters, the output $y$ is the next (i.e., 21st) character.
Each device is a role from a play (e.g., Brutus from The Tragedy of Julius Caesar).
The model is a Gated Recurrent Unit~\cite{cho2014learning},
which we refer to as ``RNN''.

\end{enumerate}

\myparagraph{Hyperparameters and Evaluation Metrics}
Hyperparameters of \fedavg{} were chosen similar 
to the defaults of~\citep{mcmahan2017communication}. 
We fixed an iteration budget for each dataset 
and tuned a learning rate schedule using grid search 
to find the best terminal loss averaged over training devices for \fedavg{}.
The {\em same} values were used on all \newfl{} runs without further tuning.
In addition, the linear models also use a
$\ell_2^2$ regularization, which was tuned 
separately for \fedavg{}
and each value of $\theta$ for \newfl{}.
The regularization parameter
was selected to minimize the 
$90$\textsuperscript{th} percentile of the misclassification error
on a held-out subset of training devices. 
Each \newfl{} experiment was repeated for
conformity levels 
$\theta = 0.8, 0.5, 0.1$.
Recall that we cannot actually measure the conformity of a test device due to privacy restrictions.

We track the loss $F(\cdot\,;q_k)$ incurred on each training device and the misclassification error on each test device. We summarize these 
distributions with their mean and the $90$\textsuperscript{th} percentile.
We use the latter to gauge 
performance on devices with low conformity. Other percentiles of these distributions,
and more simulation results are presented in  Appendix~\ref{sec:a:expt}.
We report each metric 
averaged over 5 different random seeds.

\begin{table*}[t!]
\caption{\small{
$90$\textsuperscript{th} percentile of the distribution 
of misclassification error (in $\%$) on the test devices.
}}
\label{table:expt:dataset:results_90}
\begin{center}
\begin{adjustbox}{max width=0.8\linewidth}
\begin{tabular}{lccccc}
\toprule
Dataset     & Model   & \fedavg{}       & 
\begin{tabular}{c} \newfl{}, $\theta =0.8$ \end{tabular}&
\begin{tabular}{c} \newfl{}, $\theta =0.5$ \end{tabular}&
\begin{tabular}{c} \newfl{}, $\theta =0.1$ \end{tabular}
\\ 
\midrule
 \multirow{2}{*}{EMNIST}      
            & Linear  & $ 49.66\pm0.67$ & $49.10\pm0.24$ & $\mathbf{48.44}\pm0.38$ & $50.34\pm0.95$ \\
            & ConvNet & $28.46\pm1.07$ & $26.23\pm1.15$ & $\mathbf{23.69}\pm0.94$ & $25.46\pm2.77$ \\
\midrule
\multirow{2}{*}{Sent140}     
            & Linear  & $  46.83\pm0.54$ & $\mathbf{46.44}\pm0.38$ & $46.64\pm0.41$ & $51.39\pm1.07$ \\
            & RNN    & $49.67\pm3.95$ & $\mathbf{46.46}\pm4.39$ & $50.48\pm8.24$ & $86.45\pm10.95$ \\
\midrule
Shakespeare & RNN    & 
$ \mathbf{46.45}\pm0.11$ & $\mathbf{46.33}\pm0.10$ & $\mathbf{46.32}\pm0.13$ & $47.17\pm0.14$ \\
\bottomrule
\end{tabular}
\end{adjustbox}
\end{center}
\end{table*}

\begin{table*}[t!]
\caption{\small{
Mean of the distribution of misclassification error (in $\%$) on the test devices.
}}
\label{table:expt:dataset:results_mean}
\begin{center}
\begin{adjustbox}{max width=0.8\linewidth}
\begin{tabular}{lccccc}
\toprule
Dataset     & Model   & \fedavg{}       & 
\begin{tabular}{c} \newfl{}, $\theta =0.8$ \end{tabular}&
\begin{tabular}{c} \newfl{}, $\theta =0.5$ \end{tabular}&
\begin{tabular}{c} \newfl{}, $\theta =0.1$ \end{tabular}
\\ 
\midrule
 \multirow{2}{*}{EMNIST}      
            & Linear  & $\mathbf{34.38}\pm0.38$ & $34.49\pm0.26$ & $35.02\pm0.20$ & $38.33\pm0.48$ \\
            & ConvNet & $16.64\pm0.50$ & $16.08\pm0.40$ & $\mathbf{15.49}\pm0.30$ & $16.37\pm1.03$ \\
\midrule
\multirow{2}{*}{Sent140}     
            & Linear  & $ 34.75\pm0.31$ & $\mathbf{34.41}\pm0.22$ & $35.29\pm0.25$ & $37.79\pm0.89$ \\
            & RNN    & $\mathbf{30.16}\pm0.44$ & $30.31\pm0.33$ & $33.59\pm2.44$ & $51.98\pm11.81$ \\
\midrule
Shakespeare & RNN    & 
$ \mathbf{42.90}\pm0.04$ & $\mathbf{42.93}\pm0.05$ & $43.13\pm0.05$ & $44.18\pm0.12$ \\
\bottomrule
\end{tabular}
\end{adjustbox}
\end{center}
\end{table*}

\myparagraph{Experimental Results}
Table~\ref{table:expt:dataset:results_90} lists the $90$\textsuperscript{th} 
percentile of the misclassification error 
on the test devices on the final model at the end
of our iteration budget.
We observe that for all datasets,
the $90$\textsuperscript{th} 
percentile of the test error 
is smaller for \newfl{} than for \fedavg{}
at some value of $\theta$, and often at multiple
values of $\theta$. This highlights the benefit 
of the \newfl{} framework in dealing with heterogeneous device distributions. Table~\ref{table:expt:dataset:results_mean}
records the mean of the distribution of test errors.
We see that \newfl{} is on par with \fedavg{} on the mean
of the misclassification errors on the test devices, and
sometimes even better. 

\myparagraph{Performance Across Devices}
We now visualize the misclassification error across 
all test devices in a histogram in Figure~\ref{fig:main:expt:hist}. We note that \newfl{}
exhibits thinner upper tails on the error, which shows an improved performance on devices which do not 
conform with the population.

Next, Figure~\ref{fig:main:expt:scatter} shows a scatter plot
of the loss (resp. error) and the number of datapoints on a training (resp. testing) device.
Observe, firstly, that \newfl{} reduces the variance of 
of the loss on the train devices.
Secondly, note that amongst test devices with a small number 
of data points (e.g., $<200$ for EMNIST or $<100$ for Sent140), \newfl{} reduces the variance of the misclassification error. Both effects are more pronounced on 
the neural network models.

Both these plots are indicative of an improved user experience of \newfl{} on devices which do not conform as well as devices with little data.

\myparagraph{Performance Across Iterations}
Figure~\ref{fig:main:expt:optim} compares the convergence of Algorithm~\ref{algo:main:fed:proposed} with 
 \fedavg, measured in terms of the number of secure 
aggregation calls. 
We see that
\newfl{} is competitive with \fedavg{} in convergence rate, 
despite using the same hyperparameters which were tuned for \fedavg.

     \begin{figure*}[t]
    		\centering
      \adjincludegraphics[width=0.9\linewidth,trim={0 20pt 0 28.5pt},clip=true]{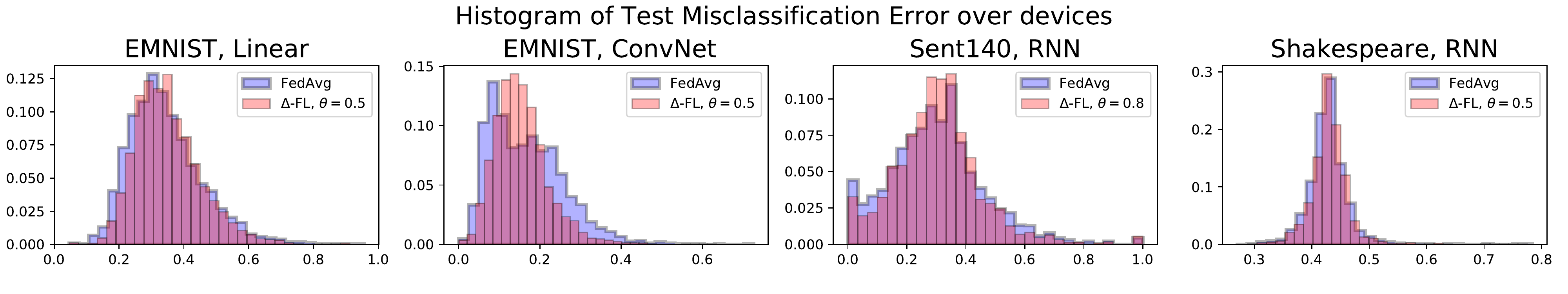}
     \caption{\small{
     Histogram of misclassification error on test devices.
     }}
     \label{fig:main:expt:hist}
     \end{figure*}
    
     \begin{figure*}[t]
		\centering
		
      \adjincludegraphics[width=0.9\linewidth]{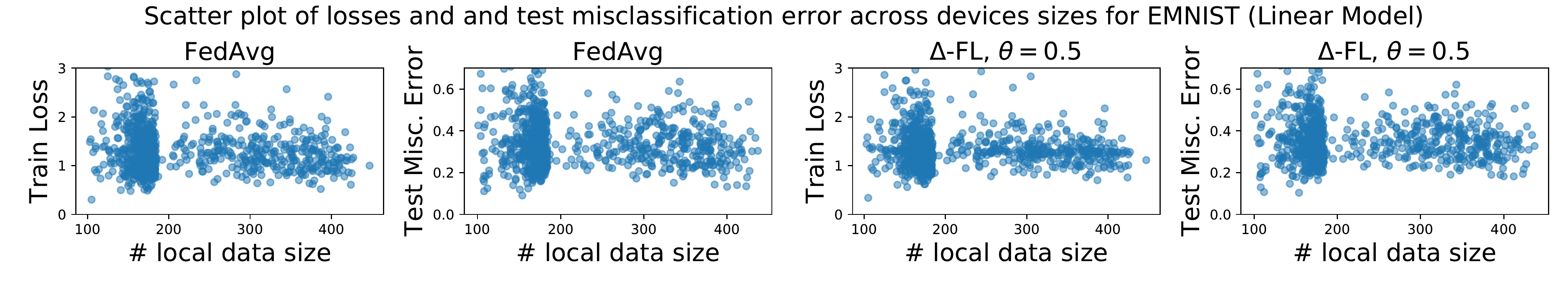}
      
      \adjincludegraphics[width=0.9\linewidth]{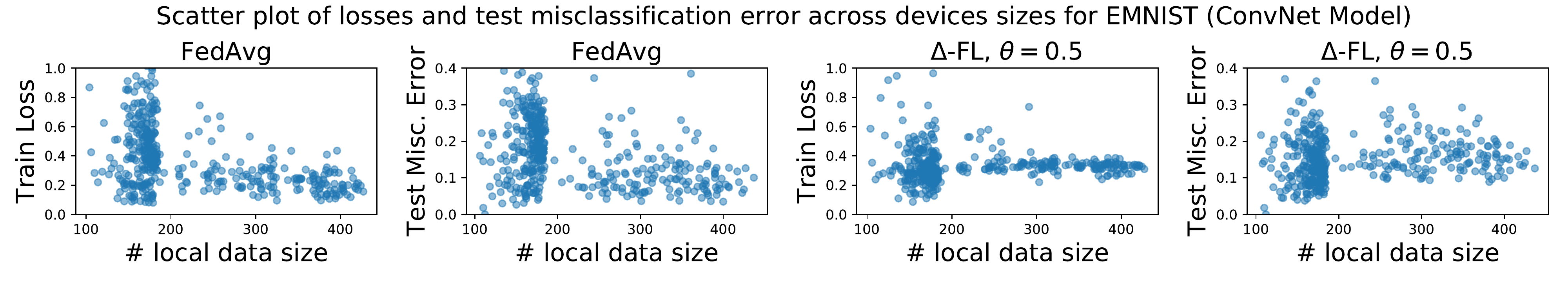}
      
      \adjincludegraphics[width=0.9\linewidth]{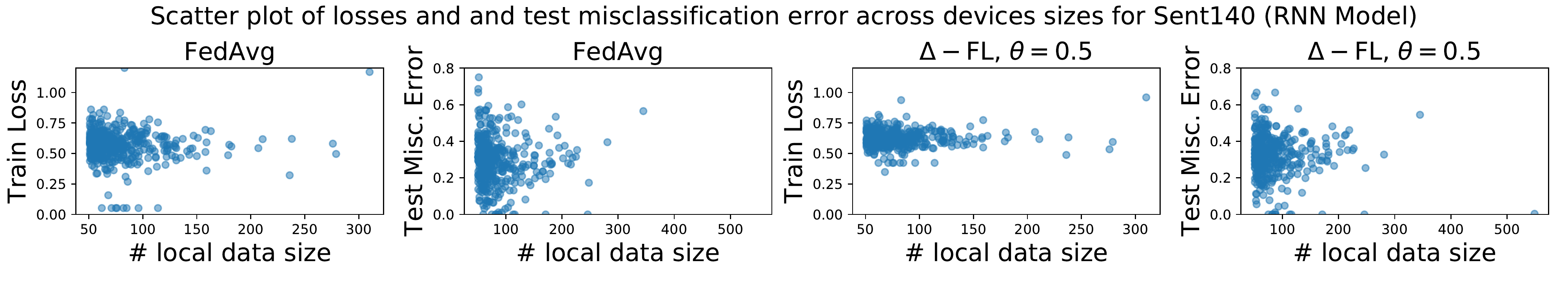}
     \caption{\small{
     Scatter plots of loss on training devices and misclassification error on test devices against the 
     number of local data points on a device.
     }}
     \label{fig:main:expt:scatter}
     \end{figure*}

    \begin{figure*}[t]
		\centering
        \adjincludegraphics[width=0.9\linewidth, trim={0 40pt 0 5pt}, clip=true]{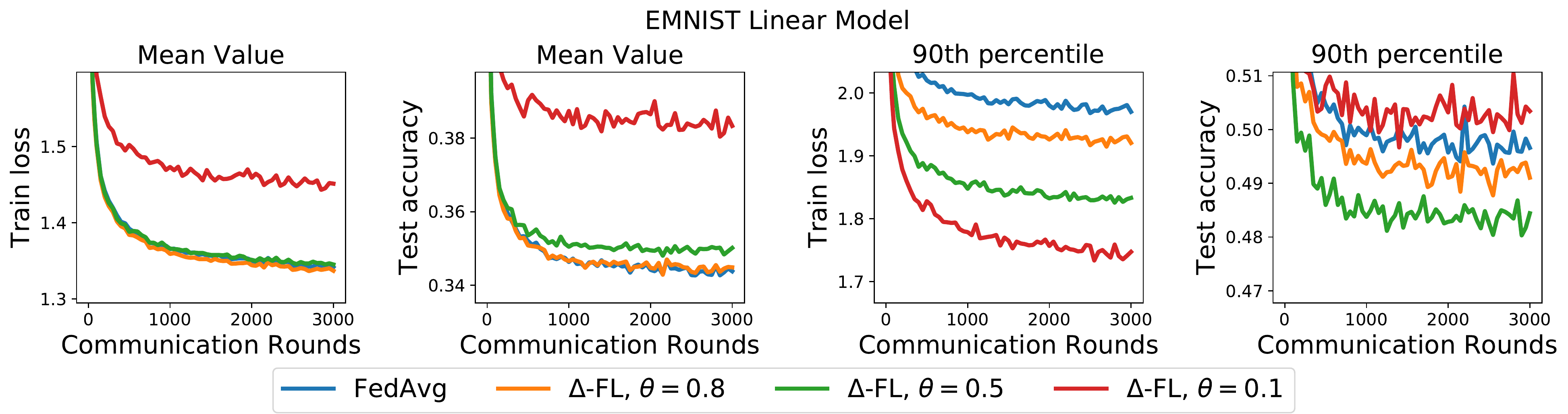}
        \adjincludegraphics[width=0.9\linewidth]{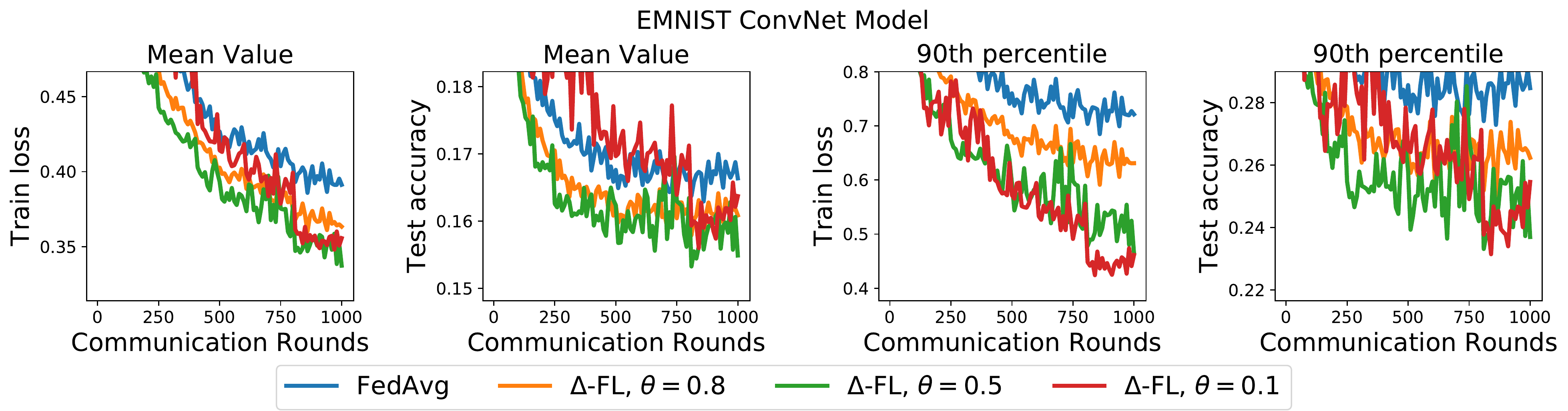}
     \caption{\small{Performance across iterations of loss on training devices and misclassification error of test devices.}}
     \label{fig:main:expt:optim}
     \end{figure*} 

\section{Related Work} \label{sec:related}
Federated learning was introduced by~\citet{mcmahan2017communication}
as a distributed learning approach to handle on-device machine learning,
with several recent extensions~\citep{konevcny2016federated,bonawitz2017practical,smith2017federated,sahu2018convergence,sattler2019robust,pillutla2019robust} -- see~\citep{kairouz2019flsurvey} for a survey.
We address non-conformity of heterogeneous client devices, which is broadly applicable in these settings. 
The classical field of distributed optimization~\citep{bertsekas1989parallel}
has seen a recent surge of interest with 
frameworks suited to
synchronous centralized~\citep{smith2018cocoa,ma2017distributed},
decentralized~\citep{he2018cola}
and asynchronous settings~\citep{leblond2018improved}.
Past works~\citep{dick2017data,eichner2019semi} have considered model plurality (i.e., an ensemble of global models) in the context of distributed learning.

Robust optimization~\citep{ben2009robust},
which espouses hedging against uncertainty by taking a worst-case approach, has become popular in machine
learning~\citep{lee2018minimax,duchi2019variance,kuhn2019wasserstein}. This approach is closely related to the risk measure approach studied in economics and finance~\citep{artzner1999coherent,rockafellar2000optimization,ben2007old} . The work we present here considers a novel use of 
the superquantile, a popular risk measure, in 
handling device heterogeneity in federated learning with a focus
on engineering plausibility.

Past works which considered the use of 
superquantiles in a centralized setting 
often used linear programming 
or convex programming approaches 
including interior point algorithms~\citep{rockafellar2014superquantile,rockafellar2018superquantile}.
In this work, we present a convergent 
inexact alternate minimization algorithm 
and show how to practically implement in 
the case of federated learning.

The paper \citep{mohri2019agnostic} gave generalization bounds
on test distributions which can be expressed as a mixture of training distributions.
The technical tools used to address fairness in 
federated learning in~\citep{mohri2019agnostic,li2020fair} bear a resemblance to ones used here, albeit to address different problems. This connection could however potentially be leveraged to obtain statistical results for our approach. We focus in this paper on the practical optimization aspects of our approach.
The orthogonal question of
personalization of federated learning models~\citep[e.g.][]{jiang2019improving}
is also an interesting avenue for future work.

\section{Conclusions} \label{sec:conclusion}
We present a federated learning framework 
to train models
better suited to heterogeneity of client data distributions in general, 
and non-conforming users in particular.
This is achieved by minimizing a 
parameterized superquantile-based
objective with the parameter 
ranging over conformity levels of the clients.
We study an optimization algorithm to minimize this
objective and present a
practical variant adapted to the engineering constraints of federated learning.
We present compelling numerical evidence in support of the
proposed framework on linear models and neural networks for various real-world tasks. 
\paragraph{Acknowledgements}
\small{
The authors would like to thank Zachary Garrett, Peter Kairouz, Jakub Kone{\v{c}}n{\'{y}}, Brendan McMahan, 
Krzysztof Ostrowski and Keith Rush
 for fruitful discussions.
This work was first presented at the Workshop on Federated Learning and Analytics in June 2019. 
This work was supported by NSF CCF-1740551, NSF DMS-1839371, 
the Washington Research Foundation for innovation in Data-intensive Discovery, 
the program ``Learning in Machines and Brains'', and faculty research awards. 
} 

\clearpage
\bibliography{bib/fl}
\bibliographystyle{abbrvnat}

\begin{titlepage}
   \vspace*{\stretch{1.0}}
   \begin{center}
      \Large\textbf{Device Heterogeneity in Federated Learning: A Superquantile Approach} 
      \\
      \vspace{5em}
      \textit{Appendices}
   \end{center}
   \vspace*{\stretch{2.0}}

   \section*{Table of Contents}
   \startcontents[sections]
   \printcontents[sections]{l}{1}{\setcounter{tocdepth}{2}}

\end{titlepage}
\appendix

\section{Problem Setup and Framework} \label{sec:a:setup}
Suppose we have $N$ client devices 
which aim to collaboratively train a machine learning model.
We assume that each device $k$ is equipped with a probability distribution $q_k$
over some measurable space $\mathcal{Z}$ (called ``data space'') 
such that 
the data on device $k$ are distributed i.i.d. according to $q_k$.
Since devices are heterogeneous, we impose no restriction on 
the similarity of $q_k$ and $q_{k'}$ for $k \neq k'$.
Further, we assume that each training device $k$ is assigned a weight $\alpha_k > 0$,
where $\sum_{k=1}^N \alpha_k = 1$ without loss of generality.

We measure the loss of a model $w \in \reals^d$ 
on a device with data distribution $q$ by 
\[
    F(w; q) = \expect_{\xi \sim q}
        [f(w; \xi)] \,,
\]
where $f : \reals^d \times \mathcal{Z} \to \reals$ is given. 
The expectation above is assumed to be well-defined and finite.
For a given distribution $q$, smaller values of $F(\cdot; q)$ denote a better fit of the model
to the data.  We use shorthand $F_k(w) := F(w; q_k)$.
We assume throughout that each $F_k$ is bounded from below.

\myparagraph{Example}
In the supervised machine learning setting, each $\xi \in \mathcal{Z}$ is an input-output pair $\xi = (x, y)$.
The function $f$ is of the form $\ell(y, \varphi(x ; w))$, where $\varphi(x; w)$ 
makes a prediction on input $x$
under model $w$ using, for example, a neural network, and $\ell$ is a loss function such as the
square loss $\ell(y, y') = (y - y')^2/2$.

\myparagraph{Test Devices}
At evaluation time, we get a novel device whose data 
distribution $p'$
is, owing to heterogeneity of devices, 
distinct from the training distribution $p_\alpha$, which we can write as
\[
   p_\alpha = \sum_{k=1}^N \alpha_k \, q_k \,.
\]
In this work, we investigate methods 
which perform well not only on average, 
but also on {\em all} devices whose data distribution
$p'$ is close to the training distribution $p_\alpha$.
We shall make this precise in the sequel.

It is known that the standard machine learning technique of minimizing 
$F(\cdot\, ; p_\alpha)$ can lead to bad performance on the test device, 
even when $p'$ is a small perturbation of the training distribution $p_\alpha$.

\myparagraph{Approach}
In this work, we focus on test devices 
whose distribution can be written as a convex combination of 
the distributions from training devices with weights 
close to the true training weights $\alpha_k$. 
Concretely, given a conformity level $0 < \theta < 1$, 
we define the set $\mathcal{P}_\theta$ of permissible weights as
\[
   \mathcal{P}_\theta = \left\{
      \pi \in \Delta^{N-1} \, : \, \pi_k \le \frac{ \alpha_k}{\theta}
      \, \forall \, k \in [N]
   \right\} \,.
\]
We now consider distributions of the form 
\[
    p_\pi = \sum_{k=1}^N \pi_k q_k\,,
\]
where $\pi \in \mathcal{P}_\theta$. Note that 
$\alpha = (\alpha_1, \cdots, \alpha_N) \in \mathcal{P}_\theta$
for every $\theta \in (0, 1)$.

The training approach pursued here consists in minimizing
\begin{align} \label{eq:setup:objective}
    F_\theta(w) := \max_{\pi \in \mathcal{P}_\theta} F(w; p_\pi) \,.
\end{align}
There is a trade-off between the size of $\mathcal{P}_\theta$
and the performance $\min_w F(w; p_\alpha)$ on the training distribution $p_\alpha$.
A small conformity $\theta$ implies that 
we take a more conservative approach 
where we would like to be able to make guarantees on 
test distributions $p_\pi$ 
which do not conform much with 
the training distribution $p_\alpha$. 
However, this may come at the cost of fitting the training distribution $p_\alpha$.

\myparagraph{Duality}
The objective $F_\theta$ defined above admits the following 
dual representation. 

\begin{property} \label{property:a:duality}
   For any $\theta \in (0, 1)$,
    we have that
   $F_\theta(w) = \min_{\eta \in \reals} \overline F_\theta(w, \eta)$,
   where $\overline F_\theta: \reals^d \times \reals \to \reals$ is given by
   \[
        \overline F_\theta(w, \eta) := 
            \eta + \frac{1}{\theta} 
            \sum_{k=1}^N \alpha_k 
            \big(F_k(w) - \eta \big)_+  \,,
   \]
\end{property}
\begin{proof} %
    We reproduce the elementary proof for completeness.
    Consider the linear program
    \[
        \max_{\pi \in \mathcal{P}_\theta}  \,
        \sum_{k=1}^N \pi_k x_k \,,
    \]
    where $x \in \reals^N$  and $0 < \theta < 1$ are fixed.
    Below, we use $\pi \ge 0$ to denote the the element-wise inequality $\pi_k \ge 0$ for each $k \in [N]$.
    Since the constraint set $\mathcal{P}_\theta$ is compact,
    the objective is bounded, and 
    strong duality holds. The maximum of the linear program above thus equals
    $\min_{\eta \in \reals, \mu \in \reals^N_+} D(\eta, \mu)$,
    where,
    \begin{align*}
        D(\eta, \mu) &= \sup_{\pi \ge 0} \left\{
            \sum_{k=1}^N \pi_k x_k + \eta\left(1- \sum_{k=1}^N \pi_k \right) + \sum_{k=1}^N \mu_k\left( \frac{\alpha_k}{\theta} - \pi_k \right)
        \right\} \\
        &= \sup_{\pi \ge 0} \left\{
            \sum_{k=1}^N \pi_k(x_k - \eta - \mu_k) \right\}
            + \eta + \frac{1}{\theta} \sum_{k=1}^N \mu_k \alpha_k 
    \end{align*}
    We must have $x_k - \eta - \mu_k \le 0$ for each $k \in [N]$ in which case the supremum is zero, 
    else the supremum over $\pi_k$ is $+\infty$.
    Therefore, the dual problem can be equivalently written as
    \[
        \min\left\{ \eta + \frac{1}{\theta} \sum_{k=1}^N \mu_k \alpha_k \, : \, \eta \in \reals, \mu \in \reals^N_+,
        \mu_k \ge x_k - \eta \text{ for } k \in [N]    
        \right\}\,.
    \]
    To complete the proof,
    note that we can eliminate $\mu$ using
    $\mu_k = \max\{x_k - \eta, 0\}$.
\end{proof}

Note that $\overline F_{\theta, \nu}$ is bounded from below
since each $F_k$ is bounded from below.
This alternate representation is useful because 
(a) $\overline F_\theta$ is jointly convex in its arguments whenever 
$f$ is convex in $w$, and
(b) $F_\theta$ can be recovered from $\overline F_\theta$ by finding a (weighted) quantile of $\{F_k(w)\}_{k=1}^N$.
\begin{property} \label{property:a:overline-F-cvx}
    The function $\overline F_\theta$ 
    is jointly convex over $\reals^d \times \reals$
    for each $\theta \in (0, 1)$ whenever 
    each $F_k$ is convex (this being true when $f(\cdot, \xi)$ is convex for all $\xi \in \mathcal{Z}$).
\end{property}
\begin{proof}
    The proof follows from the fact that 
    $(w, \eta) \mapsto \max\{0, F_k(w) - \eta\}$ is, 
    as the maximum of convex functions, jointly convex in 
    $(w, \eta)$ for each $k\in[N]$.
\end{proof}

\begin{property} \label{property:a:partial-min-eta}
    Denote by $\sigma_w$ a permutation of 
    $[N]$ satisfying 
    $F_{\sigma_w(1)}(w) \le \cdots \le F_{\sigma_w(N)}(w)$. 
    Define $\eta^\star(w) =  F_{j^*_w}(w)$, where
    \[
        j^\star_w = \sigma_w^{-1} \bigg( \min \big\{j \, : \, 
            \sum_{k=1}^j \alpha_{\sigma_w(j)} \ge 1-\theta \big\} \bigg) \,.
    \]
    Then, we have that 
    $F_\theta(w) = \overline F_\theta\big(w, \eta^\star(w) \big)$.
\end{property}
\begin{proof}
    Consider some $x\in \reals^N$ such 
    that $x_1 < \cdots < x_N$ (we will handle ties later). 
    We start by noting that the function
    \[
        h(\eta) := \eta + \frac{1}{\theta} \sum_{k=1}^N \alpha_k (x_k - \eta)_+
    \]
    is minimized at $\eta^\star = x_{j^\star}$ where 
    $j^\star = \min\{j \,:\, \sum_{k=1}^j \alpha_k \ge 1-\theta \}$. Indeed, we can write $h(\eta)$ as 
    \[
        h(\eta) = 
        \begin{cases}
        -\frac{1-\theta}{\theta} \eta + 
        \frac{1}{\theta} \sum_{k=1}^N \alpha_k x_k
        \,, & \text{ if } \eta < x_1 \\
        \frac{\eta}{\theta}
        \left(\sum_{k=1}^j \alpha_k - (1 - \theta) \right)
         \sum_{k=j+1}^N \frac{\alpha_k x_k}{\theta}  
        \,, & \text{ if } x_j  \le \eta < x_{j+1} \, ; \,
        j \in [N-1]  \\
        \eta
        \,, & \text{ if } \eta \ge x_N \,.
    \end{cases}
    \]
    Observe that $h$ is strictly decreasing on $(-\infty, x_{j^\star})$ and non-decreasing on $[x_{j^\star}, \infty)$. Therefore, $x_{j^\star}$ is a minimizer. 
    Finally, ties can be handled by recursively reducing 
    the instance $\alpha, x$ with $x_k = x_{k+1}$ to 
    $(\alpha_1, \cdots, \alpha_{k-1},
    \alpha_{k}+\alpha_{k+1}, \alpha_{k+2}, \cdots, \alpha_N), (x_1, \cdots, x_k, x_{k+2}, \cdots, x_N)$,
    an instance with no ties. Then, $h(\eta)$ and $\eta^\star$ as defined above are identical on both instances, and the result continues to hold.
\end{proof}

Note that $\eta^\star(w)$ above is simply the weighted quantile of the collection of $F_k(w)$ %
weighted by $\alpha_k$.
Throughout this work, we assume that $\overline F_{\theta, \nu}$ attains
its minimum at some $(w^\star, \eta^\star)$.

\subsection{Smoothing}
The function $\overline F_\theta$ is not smooth
 owing to the non-smoothness of 
$(\rho)_+$. To show convergence, we consider the smoothing
$\overline F_{\theta, \nu}: \reals^d \times \reals \to \reals$ of $\overline F_\theta$ defined for $\nu > 0$ as
\begin{align}
    \label{eq:a:setup:smoothing}
    \overline F_{\theta, \nu}(w, \eta) :=
    \eta + \frac{1}{\theta} \sum_{k=1}^N \alpha_k g_\nu\big(F_k(w) - \eta \big) \,, 
\end{align}
where $g_\nu : \reals \to \reals$ is a smoothing 
of $(\rho)_+ = \max_{\gamma \in [0, 1]} \{\rho\gamma\}$ 
defined as 
\begin{align}
    \label{eq:a:setup:smoothing:relu}
    g_\nu(\rho) := \max_{\gamma \in [0,1]} \left\{ \rho \gamma - \frac{\nu \gamma^2}{2}\right\} + \frac{\nu}{2} 
    = 
    \begin{cases}
        \nu / 2 \,, & \text{ if } \rho \le 0 \,, \\
        \rho^2/(2\nu) + \nu/2 \,, & \text{ if } 0 < \rho \le \nu \,, \\
        \rho\,, & \text{ if } \rho > \nu   \,.
    \end{cases}
\end{align}
It is known (see, e.g., Section~\ref{sec:a:techn:smoothing})
that $g_\nu$ is $1$-Lipschitz,
$(1/\nu)$-smooth and that is 
uniformly approximates $g$ to $\nu/2$, i.e., 
\[
    0 \le g_\nu(\rho) - (\rho)_+ \le \nu / 2 
    \quad \forall \rho \in \reals \,.
\]

Analogously to $F_\theta$, we define 
$F_{\theta, \nu} : \reals^d \to \reals$ as 
\begin{align} \label{eq:a:smoothing:w-only}
    F_{\theta, \nu}(w) := \min_{\eta \in \reals}
        \overline F_{\theta, \nu}(w, \eta) \,.
\end{align}

We have the following smoothness and convexity properties of $\overline F_{\theta, \nu}$.

\begin{property}
\label{property:a:setup:smoothing:modulus}
    Fix $0 < \theta < 1$ and $\nu > 0$.
    We have that $0 \le \overline F_{\theta, \nu}(w, \eta)
    - \overline F_\theta(w, \eta) \le \nu/(2\theta)$
    for all $(w, \eta) \in \reals^d \times \reals$.
    Suppose each $F_k$ is $B$-Lipschitz and $L$-smooth, this being true if 
    $f(\cdot, \xi)$ is $B$-Lipschitz and $L$-smooth for each $\xi \in \mathcal{Z}$.
    Then, we have that
    \begin{itemize}
        \item $w \mapsto \grad_w \overline F_{\theta, \eta}(w, \eta)$ is $L_w$-Lipschitz 
            for all $\eta \in \reals$ where $L_w := (L + B^2/\nu)/\theta$, and,
        \item $\eta \mapsto \frac{\partial}{\partial \eta} \overline F_{\theta, \nu}(w, \eta)$
            is $L_\eta$-Lipschitz for all $w \in \reals^d$ where $L_\eta := (\nu\theta)^{-1}$.
    \end{itemize}
    On the other hand, if each $F_k$ is convex (this being true if $f(\cdot, \xi)$ is convex for each $\xi \in \mathcal{Z}$),
    then $\overline F_{\theta, \nu}$ is jointly convex in $(w, \eta)$ over $\reals^d \times \reals$.
\end{property}
\begin{proof}
    Note under the hypotheses that $\norm{\grad F_k(w}\le B$.
    Fix a $k \in [N]$ and $\eta \in \reals$ and define $h(w) := g_\nu(F_k(w) - \eta)$. Starting with the chain rule, we get, 
    \begin{align*}
        \norm{\grad h(w) -  \grad h(w')}
        &= \norm{g_\nu'(F_k(w) - \eta) \grad F_k(w) - g_\nu'(F_k(w') - \eta) \grad F_k(w')} \\
        &= \norm{g_\nu'(F_k(w) - \eta)(\grad F_k(w) - \grad F_k(w')) 
            + \grad F_k(w')\big(g_\nu'(F_k(w) - \eta) - g_\nu'(F_k(w') - \eta)\big) } \\
        &\le |g_\nu'(F_k(w) - \eta)| \norm{\grad F_k(w') - \grad F_k(w')}
            + \norm{\grad F_k(w)} \big| g_\nu'(F_k(w) - \eta) - g_\nu'(F_k(w') - \eta) \big| \\
        &\le L\norm{w-w'} + B \cdot \frac{1}{\nu}|F_k(w) - F_k(w')| \\
        &\le \left( L + \frac{B^2}{\nu}\right) \norm{w-z} \,,
    \end{align*}
    where we used that $|g_\nu'|\le 1$.
    To show the smoothness of $\overline F_{\theta, \nu}$
    w.r.t. the first argument, it remains to use the
    triangle inequality and $\sum_k \alpha_k = 1$. 
    The proof of the second
    argument follows directly from the smoothness of $g_\nu$.
    
    When each $F_k$ is convex, note that 
    $(w, \eta) \mapsto F_k(w) - \eta$ is convex. 
    It follows that 
    $(w, \eta) \mapsto g_\nu(F_k(w) - \eta)$ is,
    as the maximum of a family of convex functions,
    also convex. Therefore, $\overline F_{\theta, \nu}$
    is convex since it is the sum of convex functions.
\end{proof}

Next, we note that minimization over $\eta \in \reals$
in \eqref{eq:a:smoothing:w-only} can be 
performed exactly in closed form.

\begin{property}
\label{property:a:smoothing:partial-min}
    Let $\{\alpha_k\}_{k \in [N]}$ be strictly positive and sum to $1$, $\nu>0$,and $x \in \reals^N$  be given. 
    Then, the minimizers of the function 
    \[
        h(\eta) := \eta + \frac{1}{\theta} \sum_{k=1}^N \alpha_k g_\nu(x_k - \eta) \,,
    \]
    constitute a closed interval $[\eta_{-}^{\star}, \eta_{+}^{\star}]$, which is computable by evaluating $h'$ at the points $x_k$ and $x_k - \nu$ for $k \in [N]$.  
    In particular, if
    $x_1 < \cdots < x_N$ and
    $0 < \nu < \min_{k \in [N-1]} \{x_{k+1} - x_k\}$,
    then $h$ is minimized at $\eta^\star$ defined using 
    $j^\star = \min\big\{ j \, : \, \sum_{k=1}^j \alpha_k \ge  1- \theta \big\}$ as
    \[
        \eta^\star = x_{j^\star} - 
            \frac{\nu}{\alpha_{j^\star}} 
            \left( \sum_{k=1}^{j^\star} \alpha_k - (1-\theta) \right)  \,.
    \] 
\end{property}
\begin{proof}
    Given the definition of $g_\nu$, the function $h$ is differentiable with:
    \begin{equation*}
        h'(\eta) = 1 - \frac{1}{\theta} \sum_{k=1}^N \alpha_k \left(\mathbbm{1}_{\eta \leq x_k - \nu} + \frac{x_k - \eta}{\nu} \mathbbm{1}_{\eta \in (x_k - \nu, x_k)} \right)
        \,.
    \end{equation*}
    The function $h'$ is thus non-decreasing, piecewise linear and continuous. Since the $\alpha_k$'s sum to $1$,  it satisfies $h'(\eta) \xrightarrow{\eta \rightarrow - \infty} 1 - {1}/{\theta} < 0$ and $h'(\eta) \xrightarrow{\eta \rightarrow + \infty} 1$. Hence by the intermediate value theorem, the solution of the equation $h'(\eta) = 0$ is a closed interval that we will denote $[\eta_{-}^{\star}, \eta_{+}^{\star}]$.
    
    Now, we define sets $S, A, B$ as 
    \begin{align*}
        S &= \{x_k \,:\, k \in [N]\} \cup \{x_k - \nu \,:\, k \in [N]\}\,, \\
        A &= \{\eta \in S \,:\, h'(\eta) \geq 0\} \,, 
        \quad \text{and} \,, \\
        B &= \{\eta \in S \, :\, h'(\eta) \leq 0\} \,.
    \end{align*}
    We note that $A$ and $B$ are not empty since $\max S \in A$ and $\min S \in B$. 
    
    Further, define $a, b$ as $a := \min A$, $b := \max B$. 
    We note by continuity and piecewise linearity of $h'$ that the left derivative of $h'$ at $a$ is non-negative. 
    It follows now for any $\eta < a$ that $h'(\eta) < h'(a)$. By symmetry, we also have $h'(\eta) > h'(b)$ for any $\eta > b$. We have then two possible cases:
    \begin{itemize}
        \item If $h'(a) > 0$, then 
            it necessarily holds that $a \neq b$ and $h'(b) < 0$ and, 
            \begin{equation*}
                \eta_-^\star = \eta_+^\star = b + \frac{-h'(b) (a-b)}{h'(a) - h'(b)}
            \end{equation*}
            \item If $h'(a) = 0$, then $h'(b) = 0$ and by the fact that $h'$ is increasing in the left neighborhood of $a$, we necessarily have $a = \eta_-^{\star}$. By symmetry, we get $b = \eta_+^\star $.
    \end{itemize}{}
    By definition of $a$, it is clear that $a \geq \eta_-^\star$.
    Thus,
    the set of minimizers of $h$ can be computed by evaluating $h'$ on the set $S$.

    For the second part, suppose that $x_{k+1} > x_k + \nu$ for each $k \in [N-1]$.
    Note that the term $g_\nu(x_k - \eta)$ is a quadratic for any $\eta \in \reals$ for at most one $k \in [N]$.
    A direct calculation shows that (letting $\overline \theta := 1 - \theta)$
    \begin{align*}
        h(\eta) = 
    \begin{cases}
        -\frac{1-\theta}{\theta} \eta + 
        \frac{1}{\theta} \sum_{k=1}^N \alpha_k x_k
        \,, & \text{ if } \eta < x_1 - \nu \\
        \frac{\eta}{\theta}
        \left(\sum_{k=1}^j \alpha_k - \overline\theta \right)
        + \frac{\alpha_j}{2\nu}(x_j - \eta)^2 + 
         \sum_{k=1}^j \frac{\alpha_k \nu}{2\theta} + \sum_{k=j+1}^N \frac{\alpha_k x_k}{\theta}  
        \,, & \text{ if } x_j - \nu \le \eta < x_j \, ; \,
        j \in [N]  \\
        \frac{\eta}{\theta}
        \left(\sum_{k=1}^j \alpha_k - \overline\theta \right) +
        \sum_{k=1}^j \frac{\alpha_k \nu}{2\theta} + \sum_{k=j+1}^N \frac{\alpha_k x_k}{\theta}  
        \,, & \text{ if } x_j \le \eta < x_{j+1}-\nu \, ; \,
        j \in [N-1]  \\
        \eta + \frac{\nu}{2\theta}
        \,, & \text{ if } \eta \ge x_N \,.
    \end{cases}
    \end{align*}
    Let $j^\star$ be as defined above. We separate two cases.
    \begin{itemize}
    \item  Suppose that $\sum_{k=1}^{j^\star} \alpha_k > 1-\theta$. In this case, 
    $h$ is strictly decreasing on $(-\infty, x_{j^\star} - \nu)$ and strictly increasing on $[x_{j^\star}, \infty)$. In the interval $[x_{j^\star} - \nu, x_{j^\star})$, $h$ is a quadratic
    which is minimized uniquely at $\eta^\star \in (x_{j^\star} - \nu, x_{j^\star})$ as defined in the statement above. 
    Thus, $\eta^\star$ is the unique minimizer of $h$ in this case.
    \item 
    Instead, suppose that $\sum_{k=1}^{j^\star} \alpha_k = 1-\theta$.
    Then, $j^\star \le N-1$. Notice that 
    $h$ is a strictly
    decreasing function on $(-\infty, x_{j^\star})$, and 
    $h$ is non-decreasing on $[x_{j^\star}, \infty)$
    (in particular, it is constant on $[x_{j^\star}, x_{j^\star + 1} - \eta)$). Therefore, 
    $\eta^\star = x_{j^\star}$ is a minimizer of $h$.
    \end{itemize}
    \vspace*{-5ex}
\end{proof}

Finally, we state the following
technical lemma, which 
establishes the property of
uniform level-boundedness (defined in the statement of the lemma; see also~\citep[Definition~1.16]{rockafellar2009variational}).
This will be needed for the proof of Corollary~\ref{cor:dist}.

\begin{lemma}\label{lem:levelset}
Fix $\theta \in (0, 1)$ and $\nu > 0$. 
Consider $\overline F_{\theta, \nu}$ 
defined in \eqref{eq:a:setup:smoothing},
where each $F_k$ is continuous.
Then, the function $\overline F_{\theta,\nu}$ is level-bounded in $\eta$ locally uniformly in $w$.
That is, 
for every $\widehat w \in \reals^d$ and $\lambda \in \reals$, there exists some $\rho > 0$ such that the set
\[
    S_{\widehat w, \lambda} := 
    \left\{
        (w, \eta) \in \reals^d \times \reals 
        \, : \, \norm{w - \widehat w} \le \rho \,, \, \overline F_{\theta, \nu}(w, \eta) \le \lambda 
    \right\}
\]
is bounded.
\end{lemma}

\begin{proof}
\begin{enumerate}[label={(\alph*)}]
\item
Fix a $\widehat w \in \reals^d, \lambda > 0$.
Also, fix a $\delta > 0$ and let $\rho > 0$ be such that 
\[
    \max_{k \in [N]} |F_k(w) - F_k(\widehat w)| \le \delta
\]
for all $w \in B_\rho := \{ w \, :\, \norm{w - \widehat w} \le \rho\}$, 
the ball of radius $\rho$ around $\widehat w$.
This follows from the continuity of $F_k$'s.

\item \label{part:proof:level-bound:2}
Let $\widehat \eta_w \in \argmin_\eta \overline F_{\theta, \nu}(w, \eta)$. 
We now show that there exist
$-\infty < \eta^\star_- < \eta^\star_+ < \infty$
such that $\widehat \eta_w \in [\eta^\star_- , \eta^\star_+]$
for all $w\in B_\rho$.
It follows from Property~\ref{property:a:smoothing:partial-min} that 
\[
    \min_k F_k(w) - \nu \le \widehat \eta_w \le \max_k F_k(w)\,, 
\]
and therefore, using that $|F_k(w) - F_k(\widehat w)| \le \delta$,
we get, 
\[
    \min_k F_k(\widehat w) - \delta - \nu \le \widehat \eta_w \le \max_k F_k(\widehat w) + \delta\,.
\]
Let $\eta^\star_- := \min_k F_k(\widehat w) - \delta - \nu$ and 
$\eta^\star_+ :=  \max_k F_k(\widehat w) + \delta$.
Clearly,
$-\infty < \eta^\star_- < \eta^\star_+ < \infty$.

\item 
Next, for any fixed $w \in B_\rho$, we show that 
$\{\eta \, : \, \overline F_{\theta, \nu}(w, \eta) \le \lambda\}$, is uniformly bounded, by looking at the behaviour of 
$\overline F_{\theta, \nu}(w, \eta)$ outside of the segment $[\eta^\star_- , \eta^\star_+]$.
We have from the proof of Property~\ref{property:a:smoothing:partial-min}
that 
for $\eta < \eta^\star_-$,
\[
    \overline F_{\theta, \nu}(w, \eta) 
    = -\frac{1-\theta}{\theta}\eta + 
    \frac{1}{\theta} \sum_{k=1}^N \alpha_k F_k(w)
    \ge -\frac{1-\theta}{\theta}\eta + 
    \frac{1}{\theta} \sum_{k=1}^N \alpha_k F_k(\widehat w) - \frac{\delta}{\theta} \,,
\]
and for $\eta  > \eta^\star_+$
\[
    \overline F_{\theta, \nu}(w, \eta) = \eta + \frac{\nu}{2\theta} \,.
\]
The preceding two expressions tell us that
$\overline F_{\theta, \nu}(w, \eta)$ grows 
linearly outside $[\eta^\star_-, \eta^\star_+]$
with a slope which is independent of $w$.
If follows that $\{\eta \, : \, \overline F_{\theta, \nu}(w, \eta) \le \lambda\}$ is uniformly bounded for all 
$w \in B_\rho$. We can conclude that $S_{\widehat w, \lambda}$ is bounded.
\end{enumerate}
\vspace*{-3ex}
\end{proof}

\section{Algorithm: Convergence Proofs and Full Details} \label{sec:a:algo}
Here, we give the convergence proofs of
the results stated in the main text.
The proof of Proposition~\ref{prop:main:am:converge:asymp} is given as follows:
Part~\ref{prop:main:am:part:ncvx} in Proposition~\ref{prop:a:am-approx-stoc:convergence},
Part~\ref{prop:main:am:part:subdiff} in Corollary~\ref{cor:dist}
and, Parts~\ref{prop:main:am:part:cvx1}-\ref{prop:main:am:part:cvx2} in Corollary~\ref{cor:a:am:stoc:convex}.
The proofs of Proposition~\ref{prop:a:am-approx-stoc:convergence}
and Corollary~\ref{cor:a:am:stoc:convex} are elementary
while the proof of Corollary~\ref{cor:dist} requires 
some technical lemmas. 

\myparagraph{Setup}
We first recall the setup.
Following Properties~\ref{property:a:duality}
and \ref{property:a:overline-F-cvx}, 
consider 
the following minimization problem 
in place of \eqref{eq:setup:objective}:
\begin{align}
\label{eq:a:algo:main_goal_nonsmooth}
    \min_{(w, \eta) \in \reals^d \times \reals} 
    \overline F_{\theta}(w, \eta) \,.
\end{align}

Since this problem is nonsmooth, we 
fix some $\nu > 0$, and
consider the following
smooth surrogate
\begin{align}
\label{eq:a:algo:main_goal}
    \min_{(w, \eta) \in \reals^d \times \reals} 
    \overline F_{\theta, \nu}(w, \eta) \,.
\end{align}

\myparagraph{Algorithm}
Recall that the template alternating minimization procedure
(from Algorithm~\ref{algo:a:fed:proposed:am-raw}) 
to solve 
Problem~\eqref{eq:a:algo:main_goal} is to start with 
some $w_0 \in \reals^d$ and iterate as
\begin{align}
\label{eq:a:algo:am:exact}
\begin{aligned}
    \eta_t &\in \argmin_{\eta \in \reals} \overline F_{\theta, \nu}(w_t, \eta) \\
    w_{t+1} &\approx \argmin_{w \in \reals^d} \overline F_{\theta, \nu}(w, \eta_t) \,.
\end{aligned}
\end{align}
Due to %
Property~\ref{property:a:smoothing:partial-min},
the $\eta$-step can be computed exactly in closed form 
by simply sorting 
$\{F_k(w_t)\}_{k\in[N]}$ obtained from each of the devices.
For the inexact $w$-step, 
we assume that the random variable $w_{t+1}$ satisfies 
\begin{align}
\label{eq:a:algo:am:inexactness_bound:expectation}
    \expect\left[
        \overline F_{\theta, \nu}(w_{t+1}, \eta_t) 
        | \mathcal{F}_t \right]
    - \min_w \overline F_{\theta, \nu}(w, \eta_t) \le \epsilon_t \,,
\end{align}
where $\mathcal{F}_t := \sigma(w_t)$ is the sigma field generated by $w_t$ and 
$(\epsilon_t)_{t=0}^\infty$ is a given positive sequence. 

\myparagraph{Convergence Results}
Then, we can show almost sure convergence
of \eqref{eq:a:algo:am:exact} to a 
stationary point of $\overline F_{\theta, \nu}$ provided that the 
inexactness $\epsilon_t = o(t^{-1})$, e.g., $\epsilon_t = t^{-(1+\delta)}$
for some $\delta > 0$.
Note that this proof does not require convexity.

\begin{proposition} \label{prop:a:am-approx-stoc:convergence}
    Fix $0 < \theta < 1$ and $\nu > 0$ and suppose that 
    $F_k$ is $B$-Lipschitz and $L$-smooth for each $k \in [N]$.
    Consider the sequence $\big((w_t, \eta_t)\big)_{t=0}^\infty$ 
    produced by the iteration \eqref{eq:a:algo:am:exact} 
    using the inexactness criterion in \eqref{eq:a:algo:am:inexactness_bound:expectation} with 
    a positive sequence $(\epsilon_t)_{t=0}^\infty$ which satisfies 
    $\sum_{t=0}^{\infty} \epsilon_t < \infty$.
    Then, we have that 
    $\norm{\grad_{w, \eta} \overline F_{\theta, \nu}(w_t, \eta_t)} \to 0$ almost surely.
\end{proposition}
\begin{proof}
    First note that 
    \begin{align} \label{eq:a:algo:convergence-stoc:proof:1}
        F_{\theta, \nu}(w_{t+1}) = 
        \overline F_{\theta, \nu}(w_{t+1}, \eta_{t+1})
        \le \overline F_{\theta, \nu}(w_{t+1}, \eta_t) \,.
    \end{align}
    Fix an iteration $t$ and denote $h(w) := \overline F_{\theta, \nu}(w, \eta_t)$. From Property~\ref{property:a:setup:smoothing:modulus},
    we have that $h$ is $L_w$-smooth.
    Let $\widetilde w_t = w_t - \grad h(w_t)/L_w$ 
    denote the point obtained by one step of 
    gradient descent on $h$ from $w_t$ with step-size $1/L_w$.
    Then, we have that~\citep[see e.g.,][Thm. 2.1.5]{nesterov2013introductory}
    \[
        h(w_t) - h(\widetilde w_t) \ge \frac{1}{2L_w} \normsq{\grad h(w_t)} \,.
    \]
    Therefore, we deduce that 
    \begin{align*}
        \expect[h(w_{t+1}) | \mathcal{F}_t] 
            &\le \min h + \epsilon_t 
            ~\le~ h(\widetilde w_t) + \epsilon_t \\
            &\le h(w_t) - \frac{1}{2L_w} \normsq{\grad h(w_t)} + \epsilon_t \,.
    \end{align*}
    Combining this with \eqref{eq:a:algo:convergence-stoc:proof:1}
    and using $h(w_t) = F_{\theta, \nu}(w_t)$, we get,
    \[
        \expect[F_{\theta, \nu}(w_{t+1}) | \mathcal{F}_t] 
            - F_{\theta, \nu}(w_t) \le 
            - \frac{1}{2L_w} \normsq{\grad_w \overline F_{\theta, \nu}(w_t, \eta_t)} + \epsilon_t \,.
    \]
    Next, we take another expectation over $\mathcal{F}_t$ to get 
    \[
        \expect[F_{\theta, \nu}(w_{t+1})] 
            - \expect [F_{\theta, \nu}(w_t)] \le 
            - \frac{1}{2L_w} \expect \normsq{\grad_w \overline F_{\theta, \nu}(w_t, \eta_t)} + \epsilon_t \,.
    \]
    Summing this over $\tau = 0$ to $t-1$, and using 
    $F_{\theta, \nu}(w_t) \ge \min_w F_{\theta, \nu}(w) 
    =: F_{\theta, \nu}^\star$,
    we get, 
    \[
        \sum_{\tau=0}^{t-1} \expect \normsq{\grad_w \overline F_{\theta, \nu}(w_\tau, \eta_\tau)} 
        \le 2L_w(F_{\theta, \nu}(w_0) - F_{\theta, \nu}^\star) + 2L_w \sum_{\tau=0}^{t-1} \epsilon_\tau. 
    \]
    Since $(\epsilon_t)$ is summable, there exists a constant $C < \infty$ such that, letting $t \to \infty$, we get
     \[
        \expect \left[\sum_{\tau=0}^{\infty} \normsq{\grad_w \overline F_{\theta, \nu}(w_\tau, \eta_\tau)}\right]
        =
        \sum_{\tau=0}^{\infty} \expect \normsq{\grad_w \overline F_{\theta, \nu}(w_\tau, \eta_\tau)} 
        \le C \,.
    \]
    This yields that the probability of having a finite sum is total,
    i.e., 
    \[
    1 = \prob\left(\sum_{\tau=0}^{\infty} \normsq{\grad_w \overline F_{\theta, \nu}(w_\tau, \eta_\tau)} < \infty \right) 
    \leq \prob\Big(\normsq{\grad_w \overline F_{\theta, \nu}(w_\tau, \eta_\tau)}\to 0\Big) \,, 
    \]
    which exactly means that $\normsq{\grad_w \overline F_{\theta, \nu}(w_t, \eta_t)} \to 0$ almost surely.

    To complete the proof, note from the first-order optimality conditions of the $\eta$-step,
    we have for for each $t$ that
    \[
        \frac{\partial}{\partial \eta} \overline F_{\theta, \nu}(w_t, \eta_t) = 0 \,.
    \]
\end{proof}

Next, we sharpen the convergence result in the presence of convexity.
\begin{corollary}
\label{cor:a:am:stoc:convex}
    Consider the setting of Proposition~\ref{prop:a:am-approx-stoc:convergence}.
    Suppose, in addition, that each $F_k$ is convex (which is true if 
    $f(\cdot, \xi)$ is convex for each $\xi \in \mathcal{Z}$).
    Then, we have almost surely that 
    $\overline F_{\theta, \nu}(w_t, \eta_t) \to \min \overline F_{\theta, \nu}$, or equivalently, that 
    $F_{\theta, \nu}(w_t) \to \min F_{\theta, \nu}$.
    Furthermore, we have almost surely that 
    \[
         \limsup_{t \to \infty} 
         F_\theta(w_t) 
         \le 
         \min F_\theta + \frac{\nu}{2\theta} \,.
    \]
\end{corollary}
\begin{proof}
    Let $\overline w_t \in \reals^{d+1}$ denote the pair $(w_t, \eta_t)$.
    With abuse of notation, we write $\overline F_{\theta, \nu}(\overline w_t)$ to denote $\overline F_{\theta, \nu}(w_t, \eta_t)$.
    Let $\overline w^\star$ denote a global minimizer of $\overline F_{\theta, \nu}$.
    Since $\overline F_{\theta, \nu}$ is convex (Property~\ref{property:a:setup:smoothing:modulus}) and differentiable, 
    we get that
    \begin{align*}
        0 \le \overline F_{\theta, \nu}(\overline w_t) - 
        \overline F_{\theta, \nu}(\overline w^\star)
        &\le \grad_{\overline w}\overline F_{\theta, \nu}(\overline w_t)\T (\overline w - \overline w^\star) \\
        &\le 
        \norm{\grad_{\overline w}\overline F_{\theta, \nu}(\overline w_t)} \norm{\overline w - \overline w^\star} 
        \stackrel{\mathrm{a.s.}}{\to} 0 \,,
    \end{align*}
from Proposition~\ref{prop:a:am-approx-stoc:convergence}.
The claim about convergence on $F_{\theta, \nu}$ follows because 
$F_{\theta, \nu}(w_t) = \overline F_{\theta, \nu}(w_t, \eta_t)$
from \eqref{eq:a:algo:am:exact}, and $\min F_{\theta, \nu} = \min \overline F_{\theta, \nu}$ due to convexity.

The claim about convergence on $ F_\theta$ follows
because (a) $0 \le \overline F_{\theta, \nu}(\overline w) - \overline F_\theta(\overline w) \le \nu/(2\theta)$ 
for all $\overline w \in \reals^{d+1}$ (Property~\ref{property:a:setup:smoothing:modulus}), 
(b) $\min \overline F_{\theta, \nu} \le 
    \min \overline F_{\theta}
    + {\nu}/({2\theta})$ (consequence of (a)),
    and,
(c) $\min \overline F_\theta = \min F_\theta$ (due to convexity) as
\begin{align*}
    F_\theta(w_t)
    &\le \overline F_\theta(w_t, \eta_t)
    \stackrel{(a)}{\le} \overline F_{\theta, \nu}(w_t, \eta_t) 
    \to \min \overline F_{\theta, \nu} 
    \stackrel{(b)}{\le} \min \overline F_{\theta} + \frac{\nu}{2\theta} 
    \stackrel{(c)}{=} \min F_\theta + \frac{\nu}{2\theta} \,.
\end{align*}
\end{proof}

In order to show the result on the subdifferential, 
we first need the following technical lemma. 
The notation $\partial$ refers to the Clarke 
subdifferential --- see Appendix~\ref{sec:a:techn:subdiff} 
for details.

\begin{lemma}\label{lem:diff}
Consider the setting of 
Proposition\;\ref{prop:a:am-approx-stoc:convergence}.
The function $F_{\theta,\nu}$ is locally Lipschitz and differentiable almost everywhere. 
Furthermore, $F_{\theta, \nu}$ is 
differentiable at $w$ precisely 
when 
\[
    Y_{\theta, \nu}(w):= \left\{\nabla_w \overline F_{\theta, \nu}(w,\eta) \text{ for all } \eta \in \argmin_{\eta' \in \reals} \overline F(w, \eta') \right\}
\]
is reduced to a singleton. 
In this case, we have $ Y_{\theta, \nu}(w) = \{\grad F_{\theta, \nu}(w)\}$. In general,
we have ${\partial} F_{\theta,\nu}(w) = \conv Y_{\theta, \nu}(w)$.
\end{lemma}

\begin{proof}
Recall from the setting of
Proposition\;\ref{prop:a:am-approx-stoc:convergence} that
each $F_k$ is continuously differentiable. 
The result is essentially a consequence of\;\citep[Theorem 10.58]{rockafellar2009variational}, 
but we need to invoke several other results of\;\cite{rockafellar2009variational}, as follows. Proposition 9.10  gives that $F_{\theta, \nu}$ is locally Lipschitz (or strictly continuous in the terminology of the book). Theorems 8.49 and 9.13(b) then give that the Clarke subdifferential is the convex hull of the limiting subdifferential (which is closed and bounded). Finally, with the help of Lemma\;\ref{lem:levelset}, we can apply Theorem 10.58, to get the expressions of the limiting subdifferential with $Y_{\theta, \nu}(w)$, and the characterization of differentiability. Thus, we get the expressions of the statement and the proof is complete. Note 
that in the convex case, we retrieve the result of  
\cite[Corollary 4.5.3]{hiriart1996convex}.
\end{proof}

This property gives the following subdifferential result as a corollary of Proposition\;\ref{prop:a:am-approx-stoc:convergence}.

\begin{corollary}
\label{cor:dist}
    Consider the setting of Proposition\;\ref{prop:a:am-approx-stoc:convergence}.
    We have convergence to stationarity: the distance of the subdifferential to $0$, denoted
    $\dist(0,\partial F_{\theta, \nu}(w_t))$, 
    vanishes almost surely.
\end{corollary}

\begin{proof}
We can bound the distance of the subdifferential to $0$ from
the expressions of of Lemma\;\ref{lem:diff}, as follows:
\begin{align*}
\dist(0,\partial F_{\theta, \nu}(w_t)) 
    &= \dist(0, \conv Y_{\theta, \nu}(w_t))
    \leq \dist(0, Y_{\theta, \nu}(w_t)) \\
    &= \min_{\eta \text{ optimal}} \norm{\nabla_w \overline F_{\theta, \nu}(w_t,\eta)}
    \leq \norm{\nabla_w \overline F_{\theta, \nu}(w_t,\eta_t)}
    \,.
\end{align*}
The asymptotic stationarity is thus a direct consequence of Proposition\;\ref{prop:a:am-approx-stoc:convergence}.
\end{proof}

\subsection{Quantile Computation with Secure Aggregation}
\label{sec:a:algo:quantile}

Recall from the Section~\ref{sec:algos}
that Algorithm~\ref{algo:main:fed:proposed},
as stated, requires each selected client device
to send its loss to the server for the client filtering
step. We now present a way to perform this step 
without any reduction in privacy from directly 
sending client losses to the server. 
This can be achieved by implementing the quantile 
computation using secure aggregation. 
This section is based on \cite{pillutla2019robust},
who show how to compute the geometric median using secure aggregation.

\myparagraph{Setup}
Suppose we wish to find the $\tau$-quantile of 
$x_1, \cdots, x_m \in \reals$ with respective weights 
$\alpha_1, \cdots, \alpha_m > 0$.
Recall that $\mu_\tau$ is the $\tau$ quantile if 
\[
    \sum_{k\,:\, x_k \le \mu_\tau} \alpha_k \ge \tau \,,
    \quad \text{and,} \quad
    \sum_{k \,: \, x_k \ge \mu_\tau} \alpha_k \ge 1-\tau \,.
\]

It is known~\citep[e.g.,][Chap. 1, Exercise 3]{ferguson1967mathematical} that 
$\mu_\tau$ is a $\tau$-quantile iff it it 
minimizes $H_\tau : \reals\to \reals$ defined as 
\[
    H_\tau(\mu) := \sum_{k=1}^m \alpha_k h_\tau(x_k - \mu) \,, \quad \text{where}, \quad
    h_\tau(\rho) := 
    \begin{cases}
        \tau \rho\,, & \text{ if } \rho \ge 0 \,,\\
        -(1-\tau)\rho\,, & \text{ if } \rho < 0 \,.
    \end{cases}
\]

\myparagraph{Algorithm}
Recall that secure aggregation can find a weighted 
mean of vectors (and hence, scalars) 
distributed across $m$ devices 
without revealing each device's vector to other 
devices or the server.
We now show how to compute 
a quantile as an iterative weighted mean, making it amenable
to implementation via secure aggregation. 
We note that this might not be the most
efficient way of implementing quantile computation in a
privacy-preserving manner. 
This is because secure aggregation was designed 
for high-dimensional vectors and it might be possible to do much better in the case of scalars. 
Nevertheless, we present the algorithm
as a proof-of-concept.
The underlying algorithm, based on the principle
of majorization-minimization was used, e.g.,in~\citep{hunter2000quantile}.

For any $\widetilde \mu \notin \{x_1, \cdots, x_m\}$ , 
define 
\[
    \widetilde H_\tau(\mu; \widetilde \mu)
    := \frac{1}{4} \sum_{k=1}^m \alpha_k
    \left[ 
        \frac{\normsq{x_k - \mu}}{\norm{x_k - \widetilde \mu}} + (4\tau - 2)(x_k - \mu) + \norm{x_k - \widetilde \mu}
    \right] \,,
\]
as a majorizing surrogate for $H_\tau$ at $\widetilde \mu$, i.e., 
$\widetilde H_\tau(\cdot \, ; \widetilde \mu) \ge H_\tau$
and 
$\widetilde H_\tau(\widetilde \mu ; \widetilde \mu) = H_\tau(\widetilde \mu)$.
Note that $\widetilde H_\tau(\mu ; \widetilde \mu)$ 
is an isotropic quadratic in $\mu$.

A majorization-minimization algorithm to minimize $H$ 
and hence find the $\tau$-quantile can thus be given as
\begin{align}
\nonumber
    \mu_{t+1} &=
    \begin{cases}
    \argmin_\mu \widetilde H_\tau(\mu; \mu_t)\,,
    & \text{ if } \mu_t \notin \{x_1, \cdots, x_m\} \\
    x_k\,, &\text{ if } \mu_t = x_k \text{ for some } k \in [m] 
    \end{cases}
    \\ &=
    \begin{cases}
    \big(\sum_{k=1}^m \beta_{k, t} x_k + (2q-1)\big)/{\sum_{k=1}^m \beta_{k, t}}\,,
    & \text{ if } \mu_t \notin \{x_1, \cdots, x_m\} \\
    x_k\,, &\text{ if } \mu_t = x_k \text{ for some } k \in [m]\,, 
    \end{cases}
\label{eq:a:algo:quantile_mm}
\end{align}
where 
\[
    \beta_{k, t} = \frac{\alpha_k}{\norm{x_k - \mu_t}} \,.
\]

\myparagraph{Modifying \newfl}
We now modify Algorithm~\ref{algo:main:fed:proposed}
to perform the quantile computation in a privacy-preserving
manner. However, the server can no longer perform 
the client filtering step. For this, we pass $\eta$ to 
the clients and let them filter themselves in the run 
of \textit{LocalUpdate}.
The overall algorithm in given in Algorithm~\ref{algo:a:fed:proposed:modified}.

\begin{algorithm}[tb]
	\caption{The \newfl{} algorithm: Modified Version}
	\label{algo:a:fed:proposed:modified}
\begin{algorithmic}[1]
		\Require Function $F$ distributed over $N$ devices, 
	   		number of local updates $n_\mathrm{local}$, 
	   		learning rate sequence $(\gamma_t)$,
	   		devices per round $m$,
	   		initial iterate $w_0$, conformity level $0 < \theta < 1$
	   	\Statex \textbf{Server executes:}
	    \For{$t=1,2, \cdots$}
	    	\State Sample devices $S_t \sim \mathrm{Unif}([N])^m$
	    	\State Broadcast $w_t$ to each device $k \in S_t$
	    	\State Device $k$ computes $F_k(w_t)$ 
	    	\State $\eta_t \gets \text{Quantile}\left(1-\theta,  \big(F_k(w_t), \alpha_k)\big)_{k\in S_t}\right)$
	    	using Secure Aggregation via \eqref{eq:a:algo:quantile_mm} 
	    	\ParFor{each device $k \in S_t$}
	    		\State $(w_{k, t}, \alpha_{k,t}) \gets $ \Call{LocalUpdate}{$k, w_t, \eta_t$}
	    	\EndParFor %
	   		\State $w_t \gets \textit{SecureAggregate}\left( \{(w_{k,t}, \alpha_{k,t} )\}_{k \in S_t'}\right)$
	    \EndFor
	    \Statex
	    \Function{LocalUpdate}{$k, w, \eta$}
	    \Comment Run on device $k$
	    \If{$F_k(w) \ge \eta$}
	        \Comment{Device $k$ has passed the filter}
	        \For{$i = 1, \cdots, n_\mathrm{local}$}
	    		\State Update $w \gets w - \gamma_t \grad f(w;\xi_i)$ using $\xi_i \sim q_k$
	    	\EndFor
	    	\State \Return $(w, \alpha_k)$
	    \Else \, 
	        \Comment{Device $k$ has failed the filter}
	    	\State \Return $(w, 0)$
	    \EndIf
	    \EndFunction
\end{algorithmic}
\end{algorithm}

\section{Experimental Results: Complete Details} \label{sec:a:expt}

We conduct our experiments on three datasets
from computer vision and natural language processing.
These datasets contain a natural, non-iid split of 
data which is reflective of data heterogeneity encountered in federated learning.
In this section, we describe in details the experimental setup and the results. Section~\ref{sec:datasets} described the datasets and tasks. Section~\ref{sec:hyperparameters} gives a detailed description of the hyperparameters used and the evaluation methodology. Lastly, Section~\ref{sec:exp_results} gives the experimental results.

Since each device has a finite number of datapoints
in the examples below, we let its probability distribution $q_k$ to be the empirical distribution 
over the available examples, and the weight
$\alpha_k$ to be proportional to the number of 
datapoints available on the device. 

\subsection{Datasets and Tasks}\label{sec:datasets}
We use the three following datasets, described in detail below. 
The data was preprocessed using
the LEAF framework~\cite{caldas2018leaf}.

\subsubsection{EMNIST for handwritten-letter recognition}

\myparagraph{Dataset}
EMNIST~\cite{cohen2017emnist} is a character recognition dataset. This dataset contains images of handwritten digits or letters, labeled with their identification (a-z,A-Z, 0-9).
The images are grey-scaled pictures of $28 \times 28 = 784$ pixels. 

\myparagraph{Train and Test Devices}
Each image is also annotated with the ``writer'' of the image, i.e.,
the human subject who hand-wrote the digit/letter during the data collection process. From this set of devices, we discard all devices containing less than 100 images and randomly subsampled half of the remaining devices resulting in $1730$ total devices. We performed the subsampling for computational tractability. 
We finally split these $1730$ devices into a training set of devices and a testing set of devices of equal sizes. 

\myparagraph{Model}
We consider the following models for this task.

\begin{itemize}
    \item \textbf{Linear Model}:
    We use a linear softmax regression model.
    In this case each $F_k$ is convex.
    We train parameters $w \in \mathbb{R}^{62 \times 784}$. Given an input image $x\in \mathbb{R}^{784}$, the score of each class $c\in [62]$ is the dot product $\langle w_c,  x \rangle$. The probability $p_c$ assigned to each class is then computed as a softmax: $p_c = \exp{\langle w_c,  x \rangle}/ \sum_{c'} \exp{\langle w_{c'},  x \rangle}$. The prediction for a given image is then the class with the highest probability.
    
    \item \textbf{ConvNet}: We also consider a convolutional neural network. Its architecture satisfies the following scheme:
    \begin{align*}
        \parbox{30pt}{\centering  Input \\ \small{784}} \longrightarrow & \parbox{70pt}{\centering Conv 2D \\ \small{filter = 32} \\ \small{kernel = $5 \times 5$}} \longrightarrow \;\;\;\; \parbox{30pt}{ReLU} \;\; \longrightarrow \parbox{76pt}{\centering Max Pool \\ \small{kernel = $2 \times 2$} \\ \small{stride = $2$}} \longrightarrow  \parbox{76pt}{\centering Conv 2D \\ \small{filter = 64} \\ \small{kernel = $5 \times 5$}} \\
       &\longrightarrow   \;\; \parbox{30pt}{ReLU} \!\!\! \longrightarrow \!\!\! \parbox{76pt}{\centering Max Pool \\ \small{kernel = $2 \times 2$} \\ \small{stride = $2$}} \longrightarrow \!\!\!\!\!\!\!\!\! 
       \parbox{76pt}{\centering F.C. \\ \small{units = $62$}} \longrightarrow  \text{score} \\
    \end{align*}
    In other words, it contains two convolutional layers with max-pooling and one fully connected layer (F.C) of which outputs a vector in
    $\reals^{62}$. The outputs of the ConvNet are scores with respect to each class. They are also used with a softmax operation to compute probabilities.
\end{itemize}

The loss used to train both models is the multinomial logistic loss $L(p,y) = - \log p_y$ where $p$ denotes the vector of probabilities computed by the model and $p_y$ denotes its $y$\textsuperscript{th} component. 
In the convex case we add a quadratic regularization term of the form $({\lambda}/{2}) \|w\|_2^2$. 

\subsubsection{Sent140 for Sentiment Analysis}

\myparagraph{Dataset}
Sent140~\citep{go2009twitter} is a text dataset of 1,600,498 tweets produced by 660,120 Twitter accounts. Each tweet is represented by a character string with emojis redacted. Each tweet is labeled with a binary sentiment reaction (i.e., positive or negative), which is inferred based on the emojis in the original tweet. 

\myparagraph{Train and Test Devices}
Each device represents a twitter account and contains only tweets published by this account. From this set of devices we discarded all devices containing less that 50 tweets, and split the 877 remaining devices rest of devices into a train set and a test set of sizes $438$ and $439$ respectively.
This split was held fixed for all experiments.
Each word in the tweet is encoded by 
its $50$-dimensional GloVe embedding~\citep{pennington2014glove}.

\myparagraph{Model}
We consider the following models.

\begin{itemize}
\item \textbf{Linear Model}: We consider a 
    $l_2$-regularized linear logistic 
    regression model where the parameter vector $w$  is of dimension $50$. In this case, each $F_k$ is convex. We summarize each tweet by the average of the GloVe embeddings of the words of the tweet. 
    
\item \textbf{RNN}: The nonconvex model is a Long Short Term Memory (LSTM) model~\citep{hochreiter1997long} 
    built on the GloVe embeddings of the words of the tweet. 
    The hidden dimension of the LSTM is same as the embedding dimension, i.e., $50$. We refer to it as ``RNN''.
\end{itemize}

The loss function is the binary logistic loss.

\subsubsection{Shakespeare for Language Modeling}

\myparagraph{Dataset}
The dataset consists of text from the Complete Works of William Shakespeare as raw text.
We formulate the task as a multi-class classification problem with 
53 classes (a-z, A-Z, other) as follows.
At each point, we consider the previous $H=20$ characters, 
and build $x \in \{0, 1\}^{H\times 53}$ as a one-hot encoding of these $H$ characters. 
The goal is then to predict the next character, which can belong to 53 classes.

\myparagraph{Train and Test Devices}
Each device corresponds to a role in a given play (e.g., Brutus from The Tragedy of Julius Caesar). 
All devices with less than 100 total examples are discarded, and the remaining devices are split into $544$ training and $545$ testing devices.

\myparagraph{Models}
We use a Gated Recurrent Unit (GRU) model~\cite{cho2014learning}
with $128$ hidden units for this purpose. 
We refer to it as ``RNN'' in the plots.
This is followed by a fully connected layer with 53 outputs,
the output of which is used as the score for each character. 
As in the case of image recognition, probabilities are obtained using the softmax operation.
We use the multinomial logistic loss.

\subsection{Algorithms, Hyperparameters and Evaluation strategy}\label{sec:hyperparameters}

\subsubsection{Algorithms}
We compare \fedavg~\citep{mcmahan2017communication}
with the proposed algorithm, given in Algorithm \ref{algo:main:fed:proposed} on the datasets presented in Section~\ref{sec:datasets}.
We run Algorithm \ref{algo:main:fed:proposed} with different conformity levels $\theta$ for $\theta \in \{0.8, 0.5, 0.1\}$. 

\subsubsection{Hyperparameters}
    
\myparagraph{Rounds}    
We measure the progress of each algorithm by the 
number of calls to secure aggregation, i.e., the 
number of communication rounds.
    Both \fedavg{} and Algorithm~\ref{algo:main:fed:proposed} require one call to secure 
    aggregation per iteration, hence the number of communication rounds is also equivalently the number of iterations of each algorithm.
    
    For the experiments, we choose the number of 
    communication rounds depending on the convergence of the optimization and a budget on wall-clock time.
    For the EMNIST dataset, we run the algorithm for $3000$ communication rounds with the linear model and $1000$ for the ConvNet. For the Sent140 dataset, we run the $1000$ communication rounds for the linear model and $600$ for the RNN. For the Shakespeare dataset, we run $600$ communication rounds. 
 
\myparagraph{Devices per Round}   
    We chose the number of devices per round similar to the 
    baselines of \cite{mcmahan2017communication}. All devices are assumed to be available and selections are made uniformly at random. 
    In particular, we select $100$ devices per round
    for all experiments with the exception of Sent140 RNN
    for which we used $50$ devices per round.
    
\myparagraph{\textit{LocalUpdate} and Minibatch Size}
    Each selected device (which is not un-selected, in case of Algorithm~\ref{algo:main:fed:proposed}) 
    locally runs $1$ epoch of mini-batch stochastic gradient descent
    in the \textit{LocalUpdate} method. 
    We used the default mini-batch of $10$ for all experiments~\cite{mcmahan2017communication}, except for $16$ for EMNIST ConvNet and Shakespeare. This is because the latter experiments
    were run using on a GPU, as we describe in the section on the hardware.

\myparagraph{Learning rate scheme} 
    We now describe the learning rate $\gamma_t$ used during \textit{LocalUpdate}.
    For the linear model we used a constant fixed learning rate $\gamma_t \equiv \gamma_0$, while 
    for the neural network models, 
    we using a step decay scheme of the learning rate $\gamma_t = \gamma_0 c^{-\lfloor t/t_0 \rfloor}$ for some where $\gamma_0$ and $0 < c \le 1$ are tuned. 
    We tuned 
    the learning rates only for the baseline \fedavg{}
    and used the same learning rate for \newfl{} at all values of $\theta$. 
    
    For the neural network models,
    we fixed $t_0$ so that the learning rate was decayed once or twice during the fixed time horizon $T$. 
    In particular, we used $t_0 = 400$ for EMNIST ConvNet (where $T=1000$), 
    $t_0 = 200$ for Sent140 RNN (where $T=600$) and 
    $t_0 = 200$ for Shakespeare (where $T=600$).
    We tuned $c$ from the set $\{2^{-3}, 2^{-2}, 2^{-1}, 1\}$,
    while the choice of the range of $\gamma_0$ 
    depended on the dataset-model pair. 
    The tuning criterion we used was the mean of
    the loss distribution over the training devices
    (with device $k$ weighted by $\alpha_k$)
    at the end of the time horizon.
    That is, we chose the $\gamma_0, c$ 
    which gave the best terminal training loss. 

\myparagraph{Tuning of the regularization parameter}
    The regularization parameter $\lambda$ for linear models was tuned with cross validation from the set $\{10^{-k}\,:\, k \in \{3,\cdots, 8\} \}$. This was performed as described below.
    
    For each dataset, we held out half the training devices as validation devices. Then, for different values of the regularization parameter, 
    we trained a model with the (smaller subset of) training devices and evaluate its performance on the validation devices. 
    We selected the value of the regularization parameter
    as the one which gave the smallest $10$\textsuperscript{th} percentile of the misclassification error on the validation devices.

\subsubsection{Evaluation Strategy and Other Details}
\label{sec:evaluation_strategy}

\myparagraph{Evaluation metrics}
    We record the loss of each training device and the misclassification error of each testing device, as measured on its local data. 
    
    The evaluation metrics noted in Section~\ref{sec:exp_results} are the following : the weighted mean of the loss distribution over the training devices, the (unweighted) mean misclassification error over the testing devices, the weighted $\tau$-percentile of the loss over the training device and the (unweighted) $\tau$-percentile of the misclassification error over the testing devices for values of $\tau$ among $\{20, 50, 60, 80, 90, 95 \}$. The weight used for training
    device $k$ is $\alpha_k$, which was set proportional to the number of datapoints on the device.
     
\myparagraph{Evaluation times}    
    We evaluate the model during training process for once every $l$ rounds. The value of $l$ used was $l=50$ for EMNIST linear model, 
    $l=10$ for EMNIST ConvNet and Shakespeare, $l=20$ for Sent140 linear model and $l=25$ for Sent140 RNN.

\myparagraph{Hardware}
    We run each experiment as a simulation as a single process. 
    The linear models were trained on 
    m5.8xlarge AWS instances, each with an Intel Xeon Platinum 8000 series processor with $128$ GB of memory running at most $3.1$ GHz. 
    The neural network experiments were trained on workstation with 
    an Intel i9 processor with $128$ GB of memory at $1.2$ GHz, and two Nvidia Titan Xp GPUs. The Sent140 RNN experiments were run on a CPU
    while the other neural network experiments were run using GPUs.

\myparagraph{Software Packages}
    Our implementation is based on NumPy using the Python language. 
    In the neural network experiments, we use PyTorch to implement the 
    \textit{LocalUpdate} procedure, i.e., the model itself and the 
    automatic differentiation routines provided by PyTorch to make
    SGD updates.

\myparagraph{Randomness}
    Since several sampling routines appear in the procedures such as the selection of devices or the local stochastic gradient, we carry our experiments with five different seeds and plot the average metric value over these seeds. Each simulation is run on a single process. 
    Where appropriate, we report one standard deviation from the mean.
    
\subsection{Experimental Results}\label{sec:exp_results}

We now present the experimental results of the paper. 
\begin{itemize}[topsep=0.2pt]
    \setlength\itemsep{0.2 pt}
    \item We study the performance of each algorithm over the course of the optimization.
    \item We plot the histograms the distribution of train losses and test misclassification error over the devices at the end of the training process.
    \item We present in the form of scatter plots the training loss and test misclassification error across devices achieved at the end of training, versus the number of local data points on the device.
    \item We present the number of devices selected at each communication round for \newfl{} (after device filtering).
    \item We study the effect of implementing the $w$-step as multiple
    iteration of \fedavg{} (recall that Algorithm~\ref{algo:main:fed:proposed} uses one round of \fedavg{}
    for the $w$-step of the alternating minimization meta-algorithm
    of Algorithm~\ref{algo:a:fed:proposed:am-raw}).
\end{itemize}

\myparagraph{Performance Across Iterations}
    We group plots by models and datasets.
    The $x$ axis of the plots below represents the number of communication rounds along the simulation. The $y$-axis represents either the training loss or the testing accuracy (either the mean or some percentile).
    Table~\ref{tab:a:expt:iter-summary}
    lists the figure numbers with the corresponding plots.
    \begin{table}[h!]
    \caption{Figure numbers for performance across iterations for 
        different datasets and different models.}
    \label{tab:a:expt:iter-summary}
    \begin{center}
    \begin{tabular}{ccc}
        \toprule
        \textbf{Dataset} & \textbf{Model} & \textbf{Figure} \\
        \midrule
        EMNIST & Linear Model & Figure~\ref{fig:a:expt:emnist-linear} \\
        EMNIST & ConvNet & Figure~\ref{fig:a:expt:emnist-conv} \\
        Sent140 & Linear Model & Figure~\ref{fig:a:expt:sent140-linear} \\
        Sent140 & RNN & Figure~\ref{fig:a:expt:sent140-lstm} \\
        Shakespeare & RNN & 
            Figure~\ref{fig:a:expt:shake-lstm} \\
        \bottomrule
    \end{tabular}
    \end{center}
    \end{table}

\begin{figure}[t!]
\centering
\begin{subfigure}[b]{\textwidth}
    \centering
   \includegraphics[width=\linewidth]{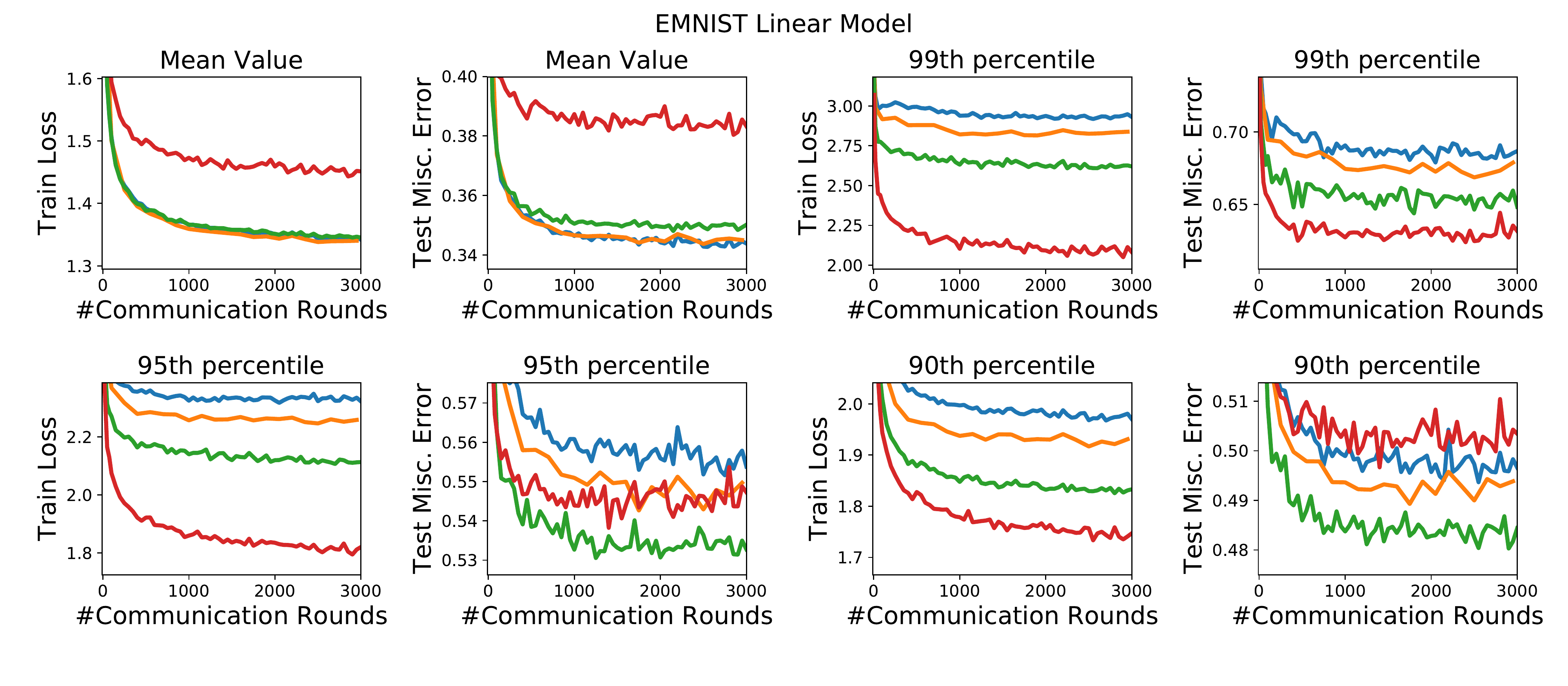}
   \label{fig:Ng1} 
\end{subfigure}

\vspace{-20pt}
\begin{subfigure}[b]{\textwidth}
    \centering
   \includegraphics[width=\linewidth]{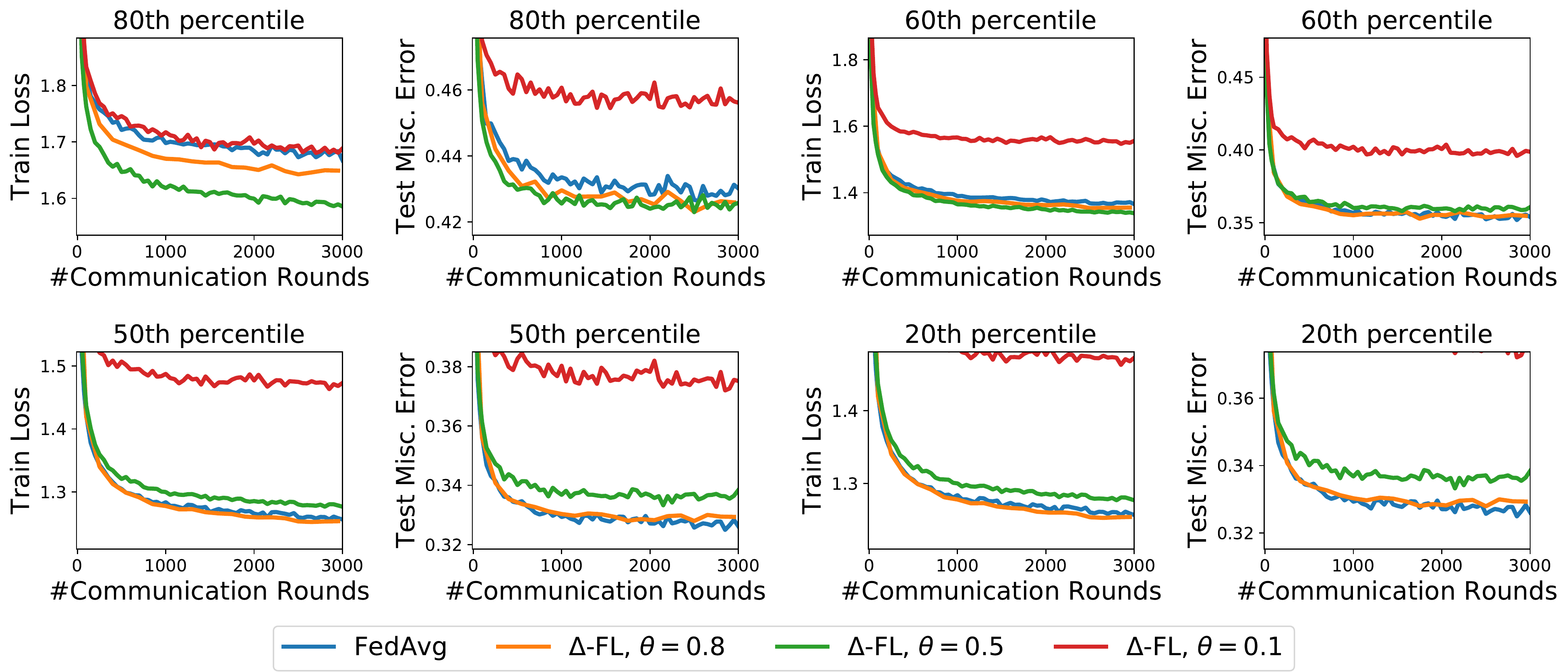}
   \label{fig:Ng2}
\end{subfigure}

\caption{Performance across iterations of EMNIST linear model.}
\label{fig:a:expt:emnist-linear}
\end{figure}

\begin{figure}[t!]
\centering

\begin{subfigure}[b]{\textwidth}
    \centering
   \includegraphics[width=\linewidth]{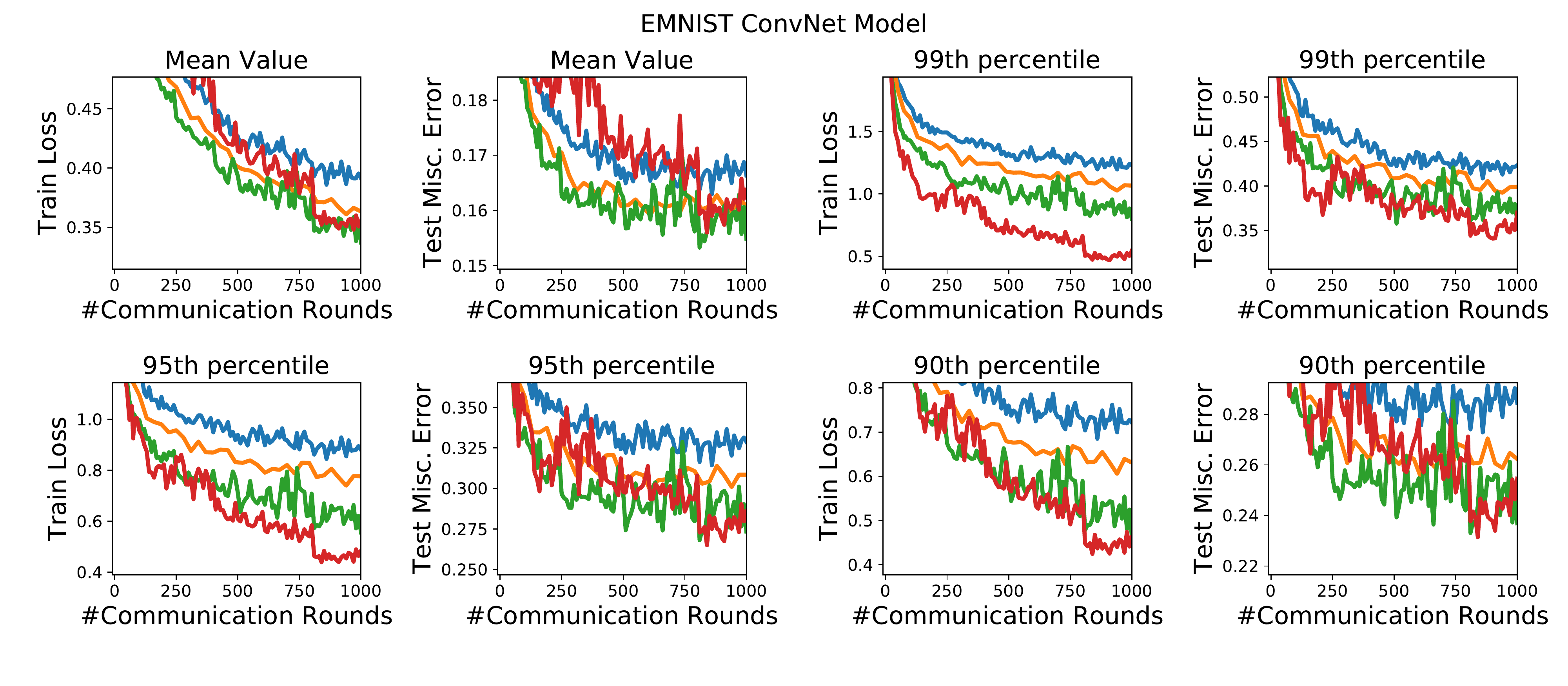}
   \label{fig:Ng1} 
\end{subfigure}

\vspace{-20pt}
\begin{subfigure}[b]{\textwidth}
    \centering
   \includegraphics[width=\linewidth]{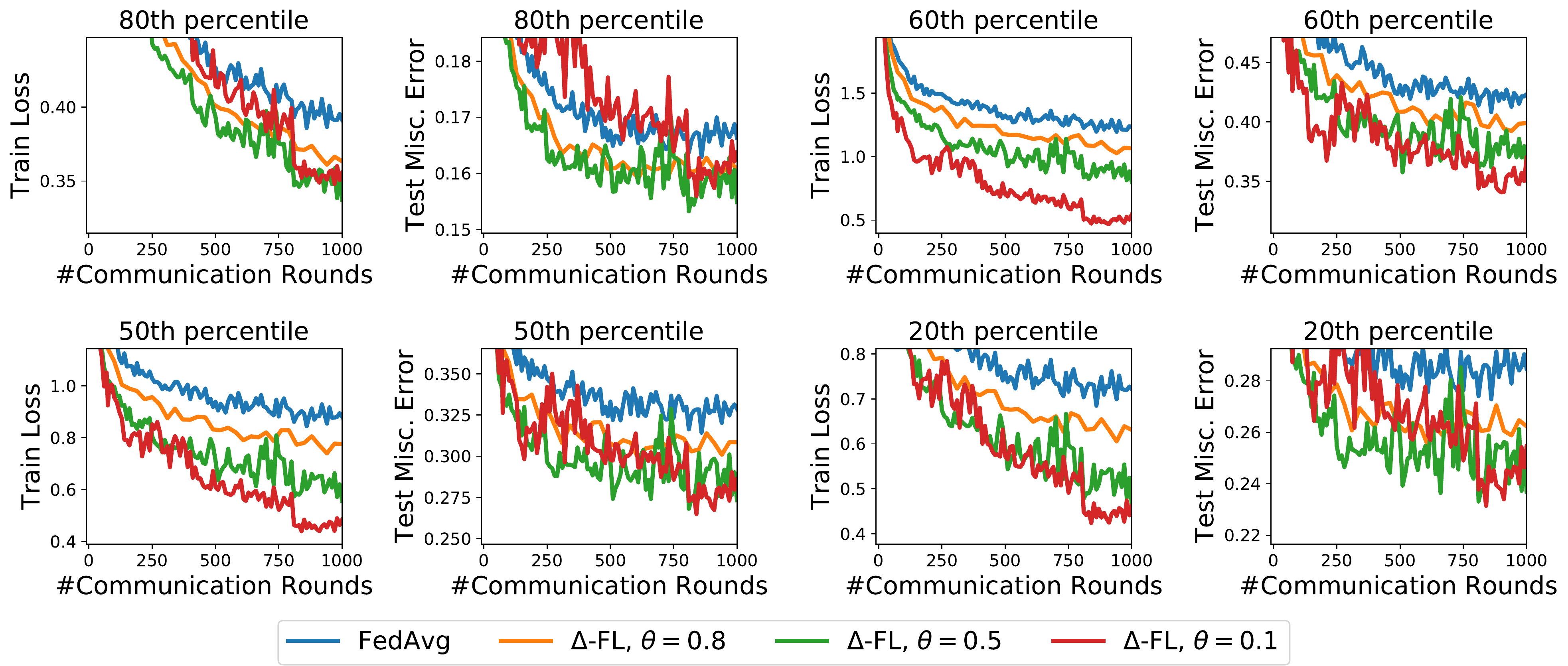}
   \label{fig:Ng2}
\end{subfigure}
\caption{Performance across iterations of EMNIST ConvNet model.}
\label{fig:a:expt:emnist-conv}
\end{figure}

\begin{figure}[t!]
\hspace{-2pt}\begin{subfigure}[b]{1.0\textwidth}
    \centering
   \includegraphics[width=\linewidth]{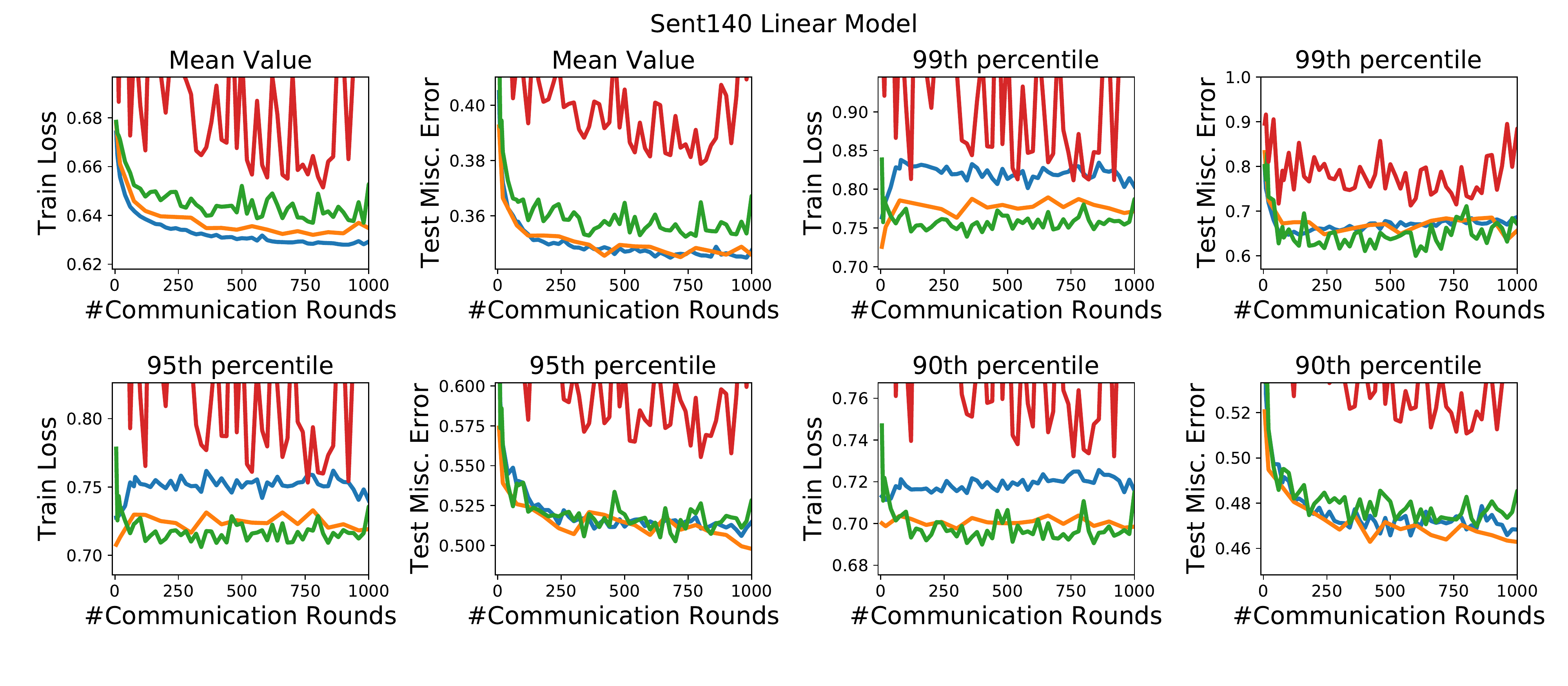}
   \label{fig:Ng1} 
\end{subfigure}

\vspace{-20pt} 
\begin{subfigure}[b]{1.0\textwidth}
    \centering
   \includegraphics[width=\linewidth]{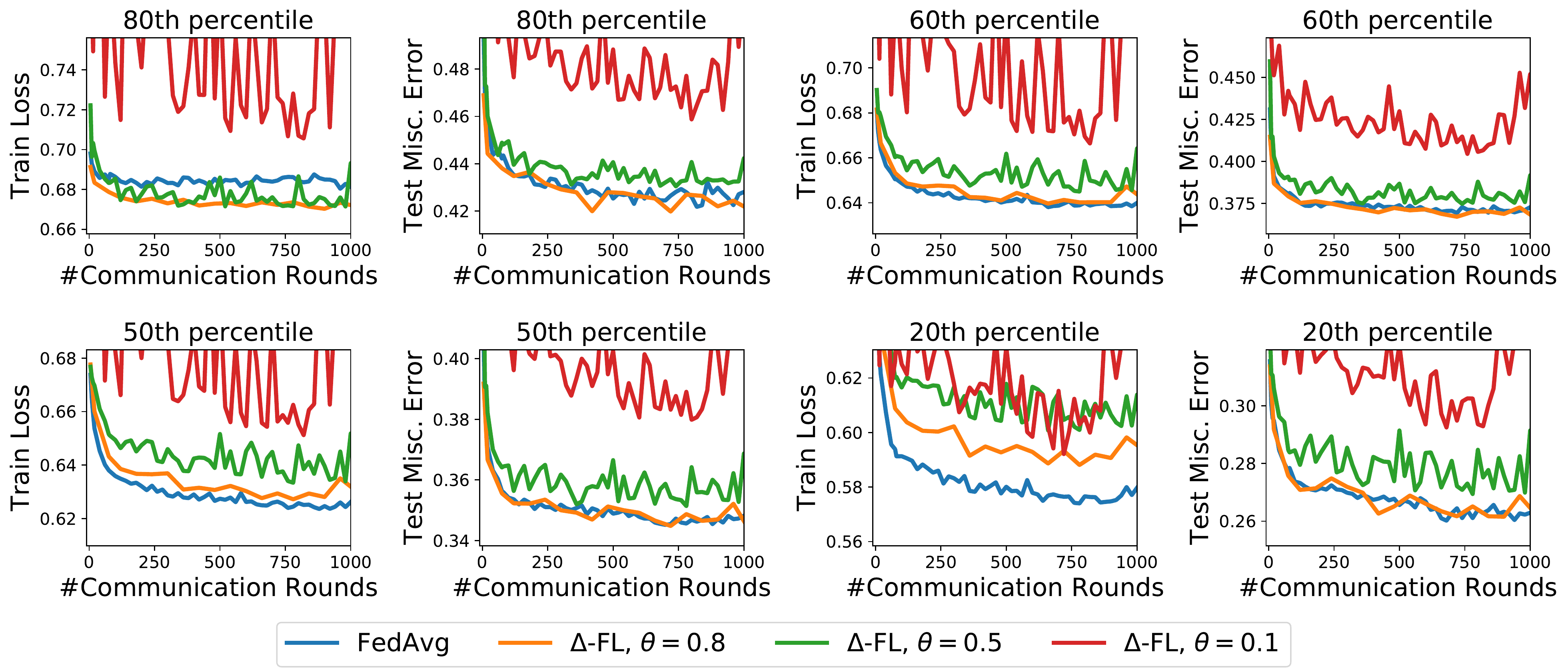}
   \label{fig:Ng2}
\end{subfigure}
\caption{Performance across iterations of Sent140 linear model.}
\label{fig:a:expt:sent140-linear}

\end{figure}

\begin{figure}[t!]
\centering
\begin{subfigure}[b]{1.0\textwidth}
    \centering
   \includegraphics[width=\linewidth]{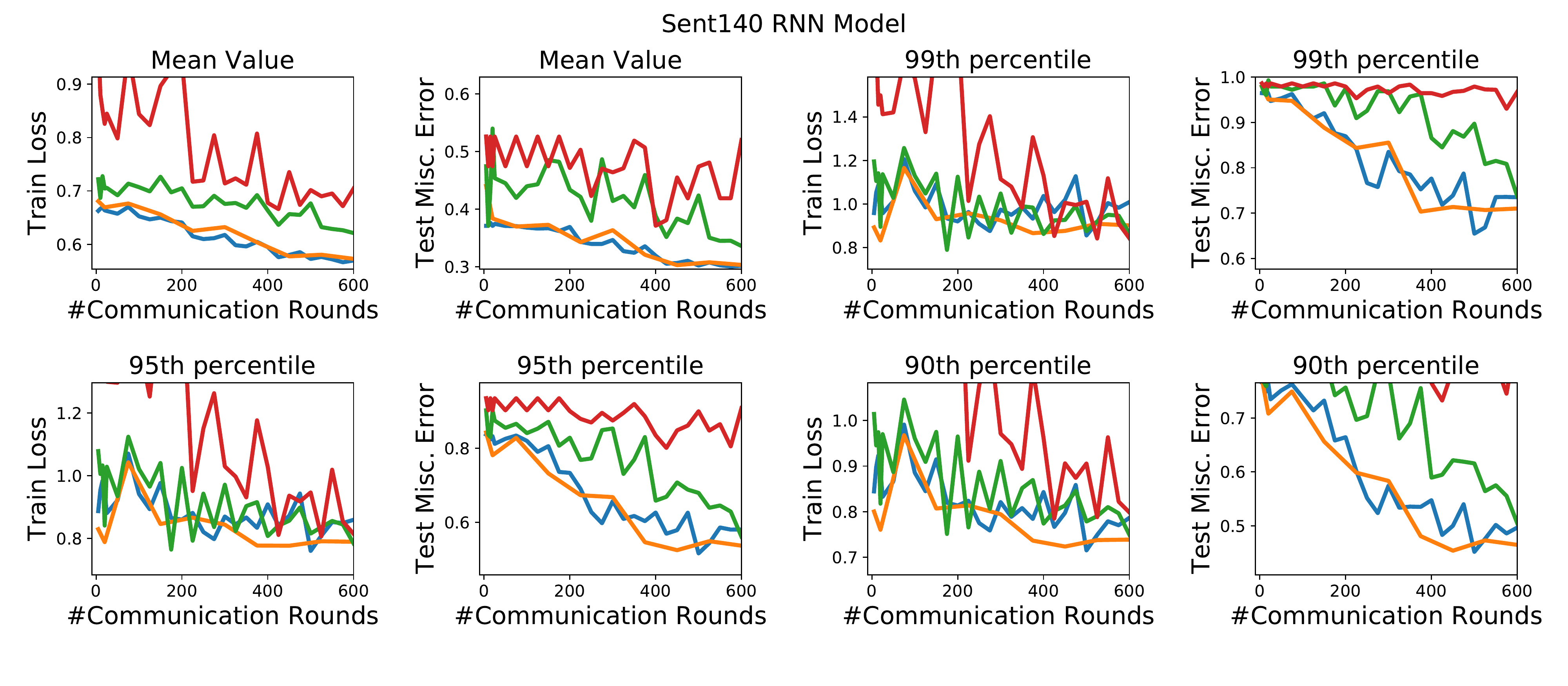}
   \label{fig:Ng1} 
\end{subfigure}

\vspace{-20pt} 
\begin{subfigure}[b]{1.0\textwidth}
    \centering
   \includegraphics[width=\linewidth]{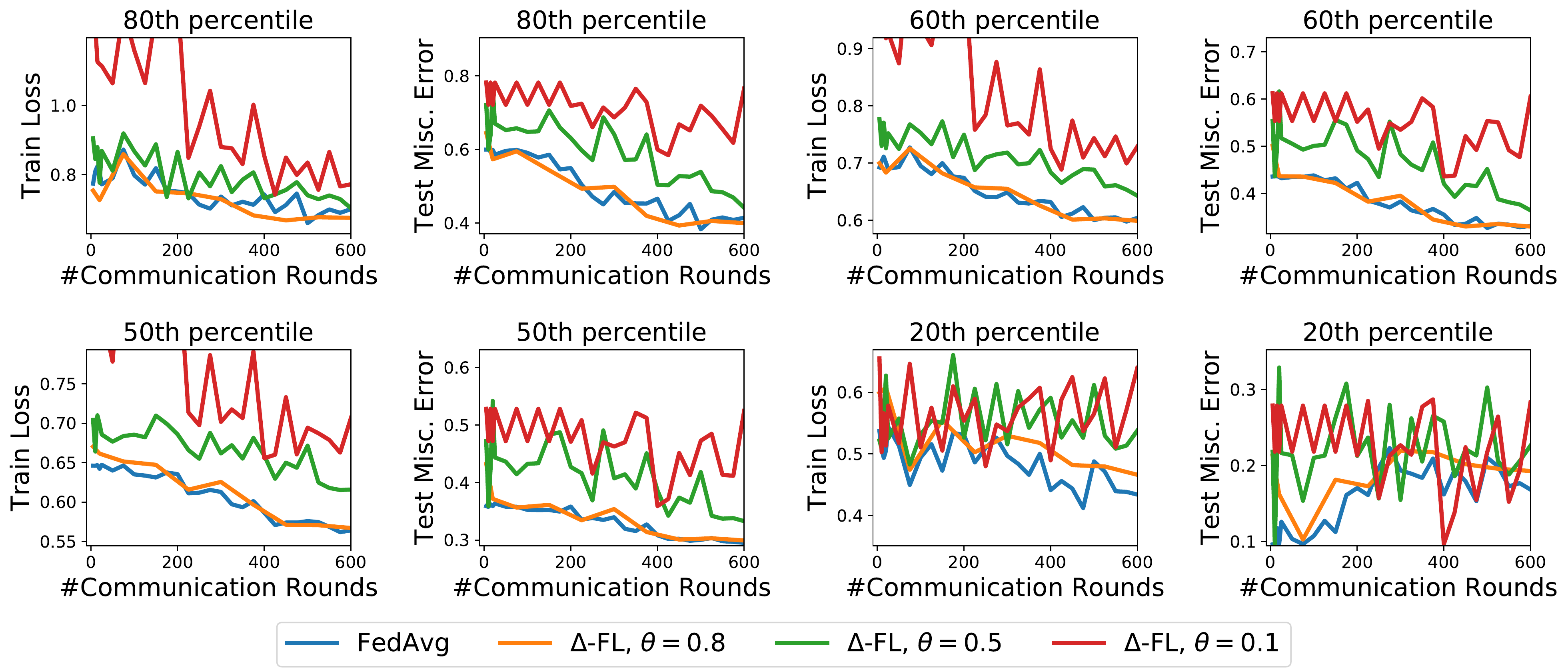}
   \label{fig:Ng2}
\end{subfigure}
\caption{Performance across iterations of Sent140 RNN model.}
\label{fig:a:expt:sent140-lstm}
\end{figure}

\begin{figure}[t!]
\hspace{-2pt}\begin{subfigure}[b]{1.0\textwidth}
    \centering
   \includegraphics[width=\linewidth]{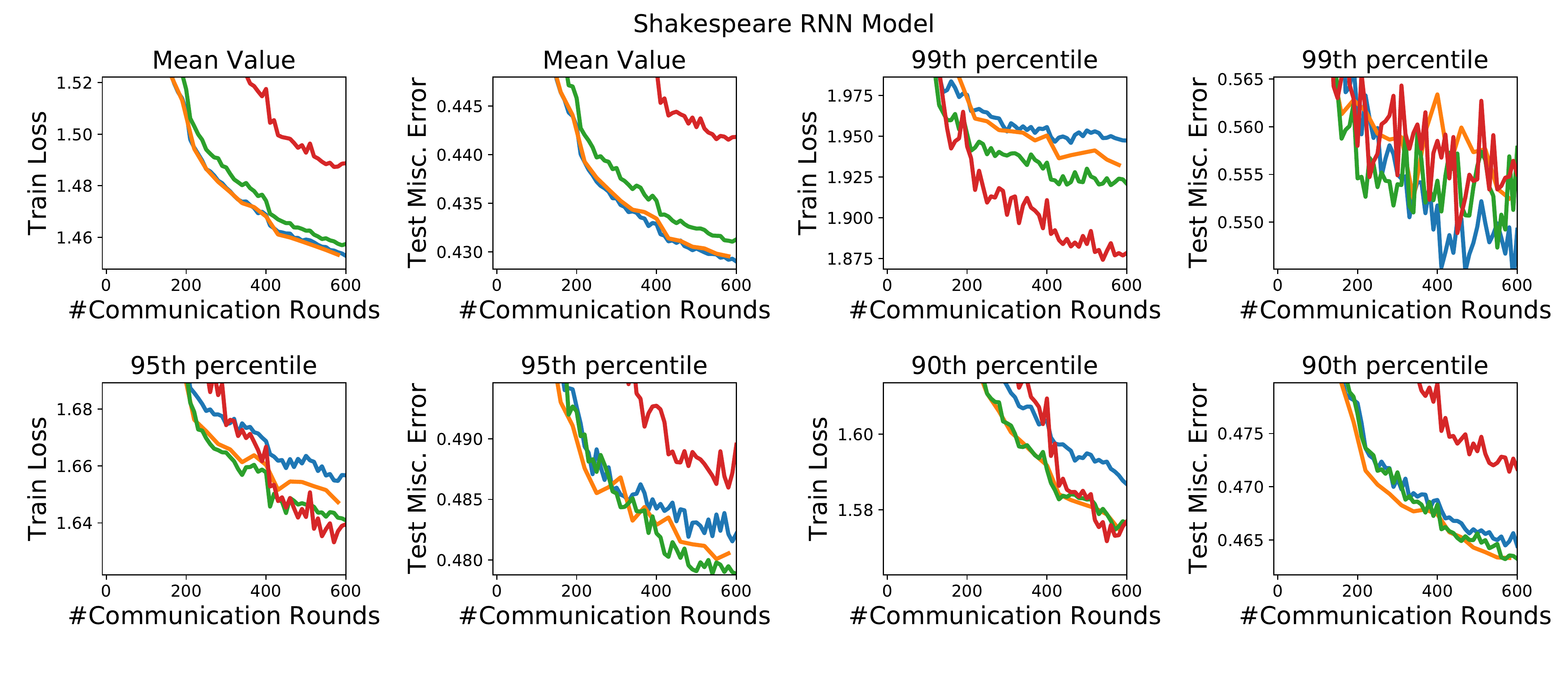}
   \label{fig:Ng1} 
\end{subfigure}

\vspace{-10pt}\begin{subfigure}[b]{1.0\textwidth}
    \centering
   \includegraphics[width=\linewidth]{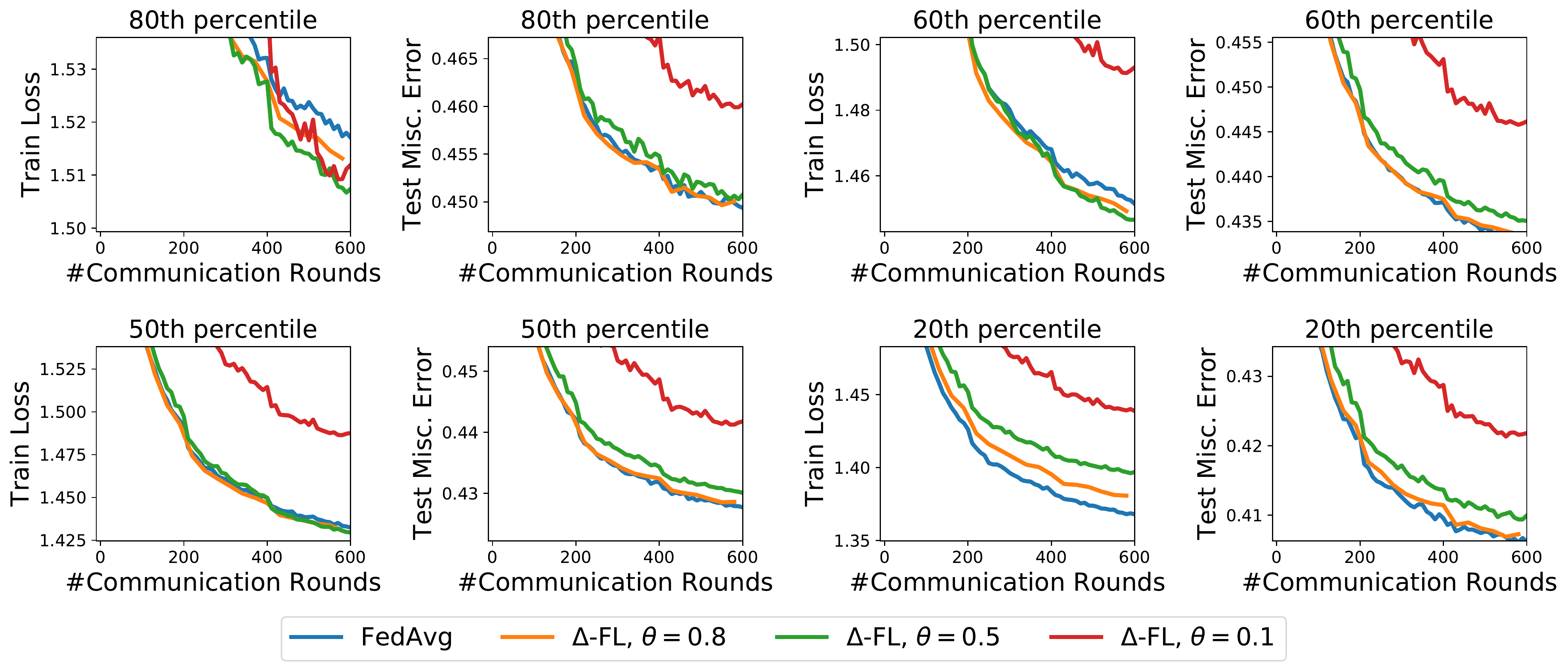}
   \label{fig:Ng2}
\end{subfigure}
\caption{Performance across iterations on the Shakespeare dataset.}
\label{fig:a:expt:shake-lstm}
\end{figure}

\myparagraph{Histograms of Loss and Test Misc. Error over Devices}
Here, we plot the histograms of the loss distribution over training devices and the misclassification error distribution over testing devices.
We report the losses and errors obtained at the end of the training 
process. Each metric is averaged per device over 5 runs of the random seed. 

Figure~\ref{fig:a:expt:hist-emnist} shows the histograms for EMNIST,
while Figure~\ref{fig:a:expt:hist-sent-shake} shows the histograms 
for Sent140 and Shakespeare.

We note that \newfl{} tends to exhibit thinner upper tails at some 
values of $\theta$ and often at multiple values of $\theta$. This shows 
the benefit of using \newfl{} over vanilla \fedavg.

\begin{figure}[t!]
   \includegraphics[width=\linewidth]{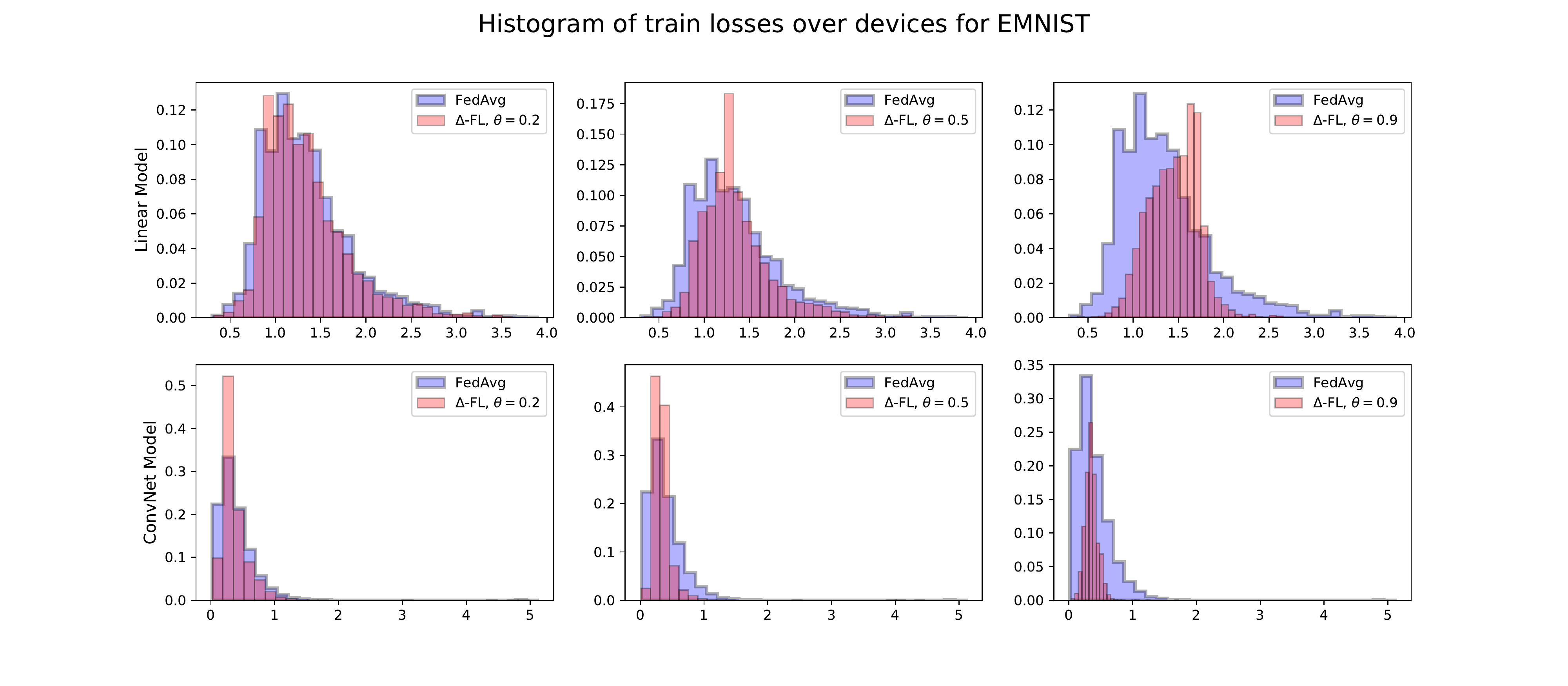}
   \includegraphics[width=16cm, height=8cm]{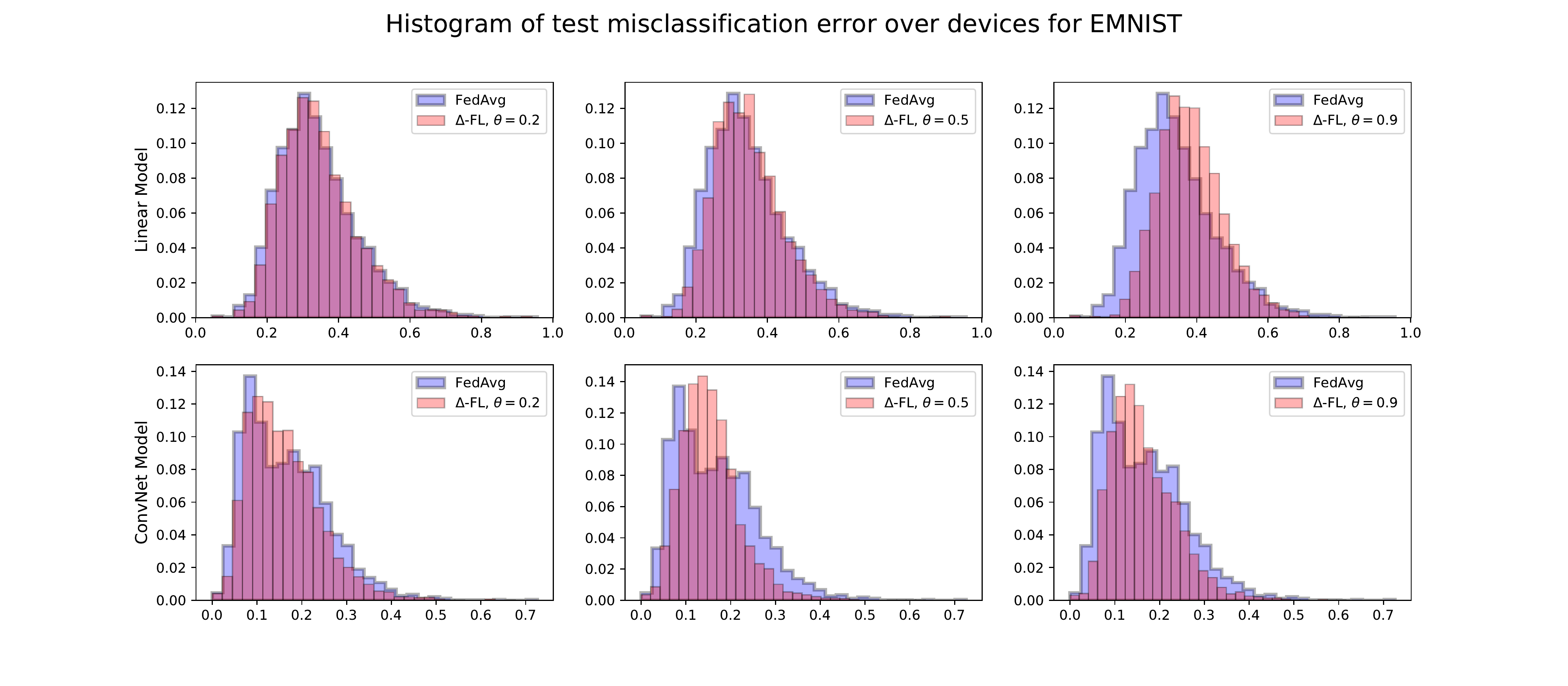}
   \caption{Histogram of loss distribution over training devices and misclassification error distribution over testing devices for EMNIST. The identification of the model (linear or ConvNet) is given 
   on the $y$-axis of the histograms.}
   \label{fig:a:expt:hist-emnist}
\end{figure}

\begin{figure}[t!]
    \centering
   \includegraphics[width=0.97\linewidth]{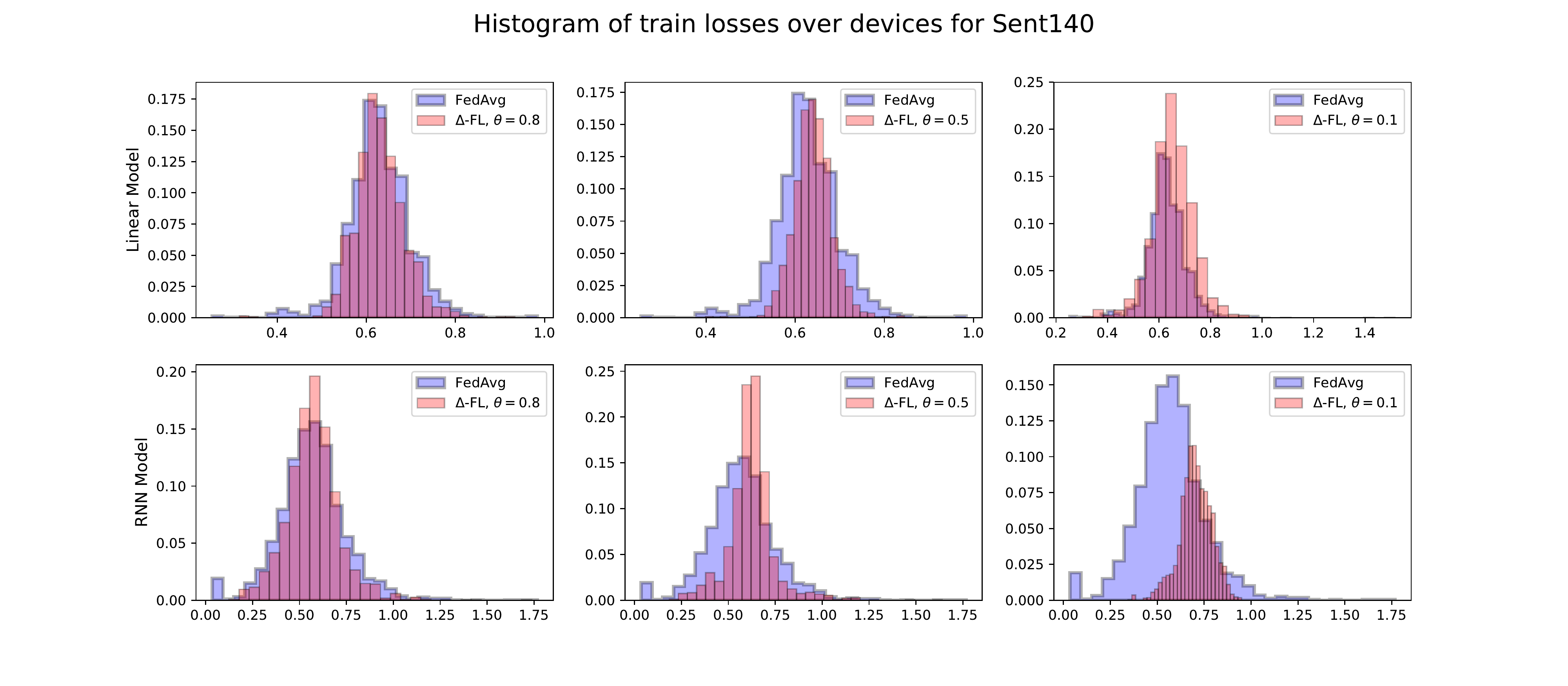}

   \includegraphics[width=0.97\linewidth]{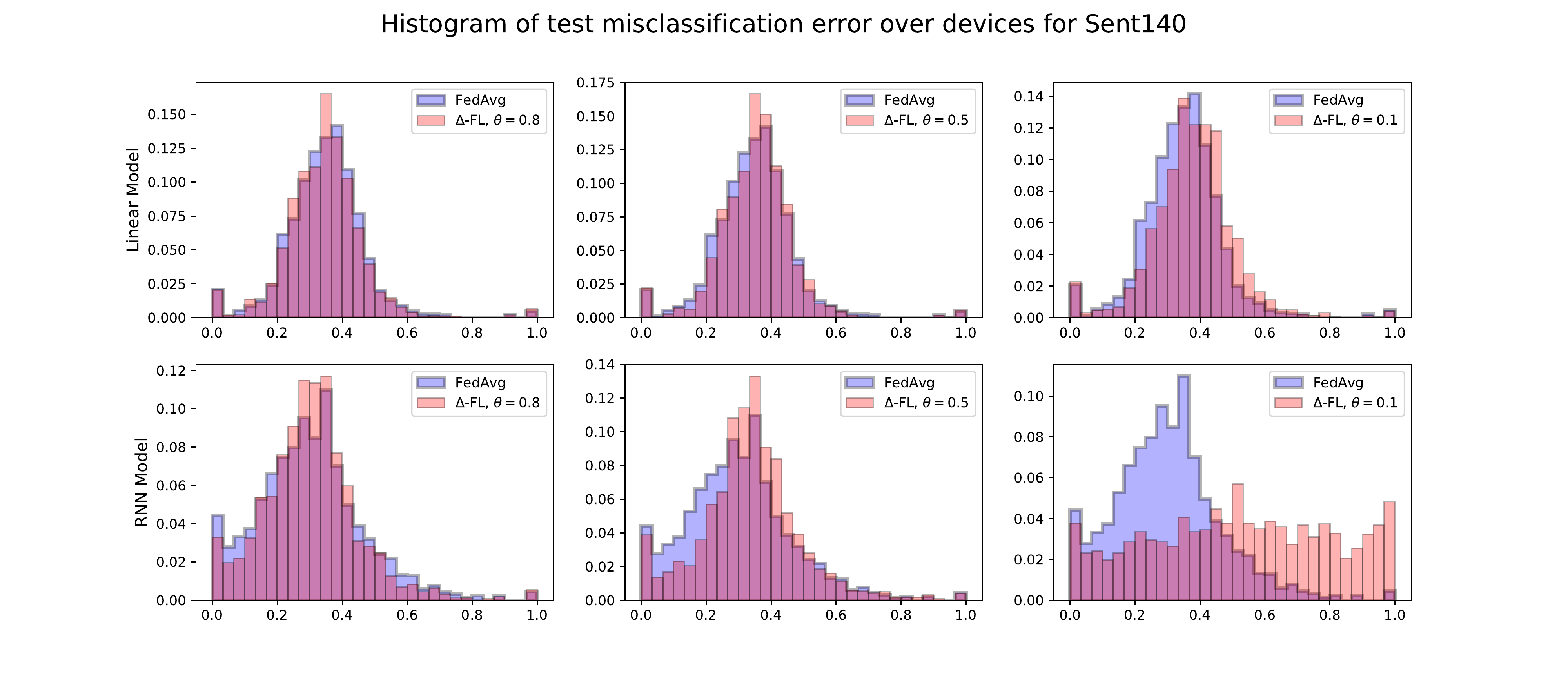}

   \includegraphics[width=0.97\linewidth]{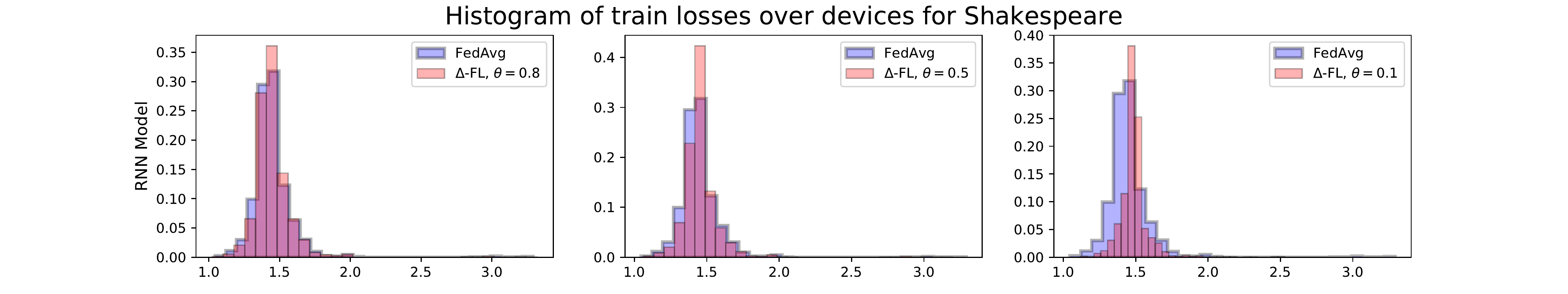}
   
   \includegraphics[width=0.97\linewidth]{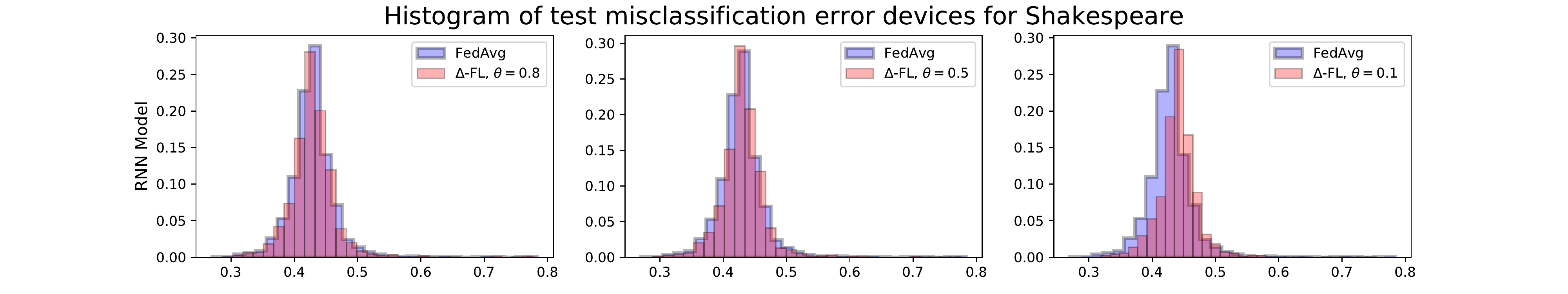}
   \caption{Histogram of loss distribution over training devices and misclassification error distribution over testing devices for Sent140 and Shakespeare. The identification of the model (linear or RNN) is given on the $y$-axis of the histograms.}
   \label{fig:a:expt:hist-sent-shake}
\end{figure}

\myparagraph{Performance compared to local data size}
Next, we plot the loss on training devices versus the amount of local data on the device
and the misclassification error on the test devices versus the 
amount of local data on the device.
See Figure~\ref{fig:a:expt:scatter-emnist} for EMNIST and 
Figure~\ref{fig:a:expt:scatter-sent-shake} for Sent140 and Shakespeare.

Observe, firstly, that \newfl{} reduces the variance of 
of the loss on the train devices.
Secondly, note that amongst test devices with a small number 
of data points (e.g., $<200$ for EMNIST or $<100$ for Sent140), \newfl{} reduces the variance of the misclassification error. Both effects are more pronounced on 
the neural network models.

\begin{figure}[t!]
\centering
   \includegraphics[width=\linewidth]{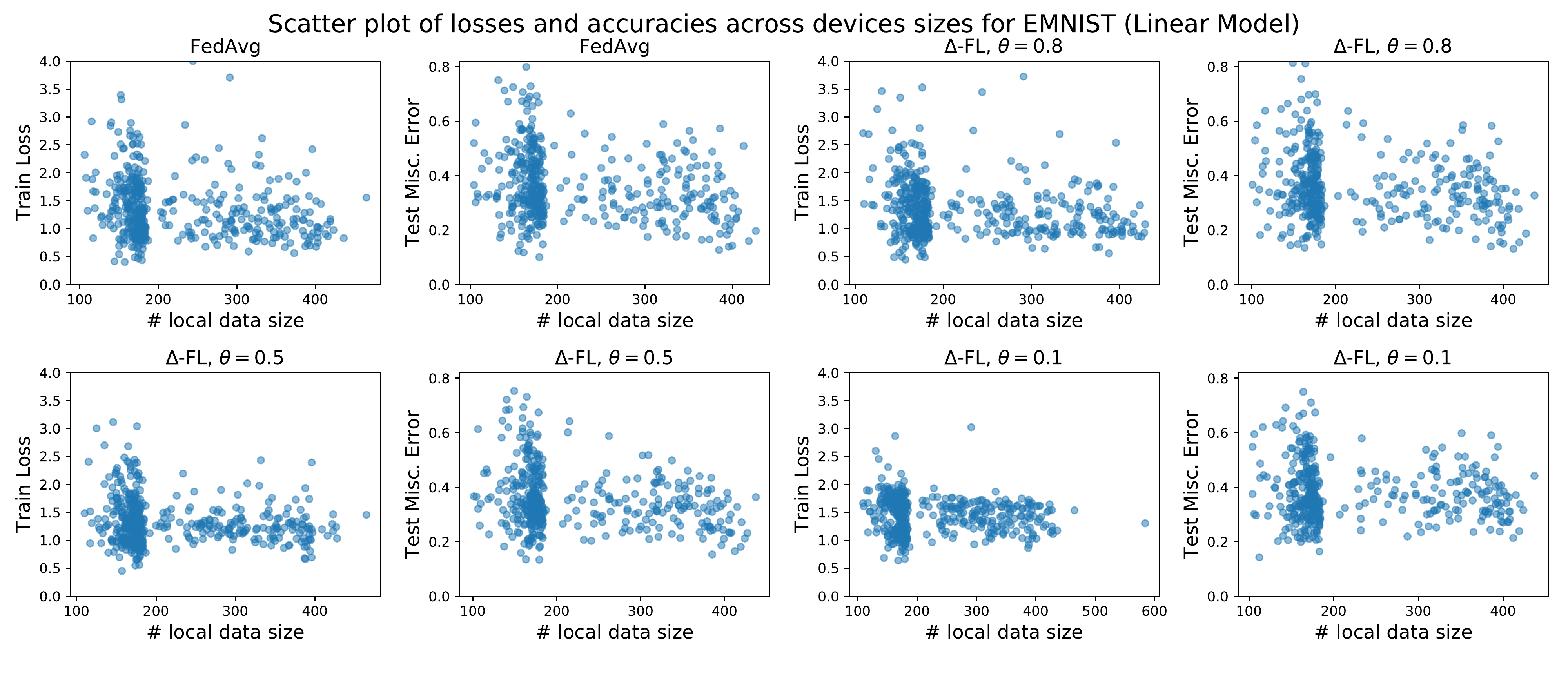}
   
   \includegraphics[width=\linewidth]{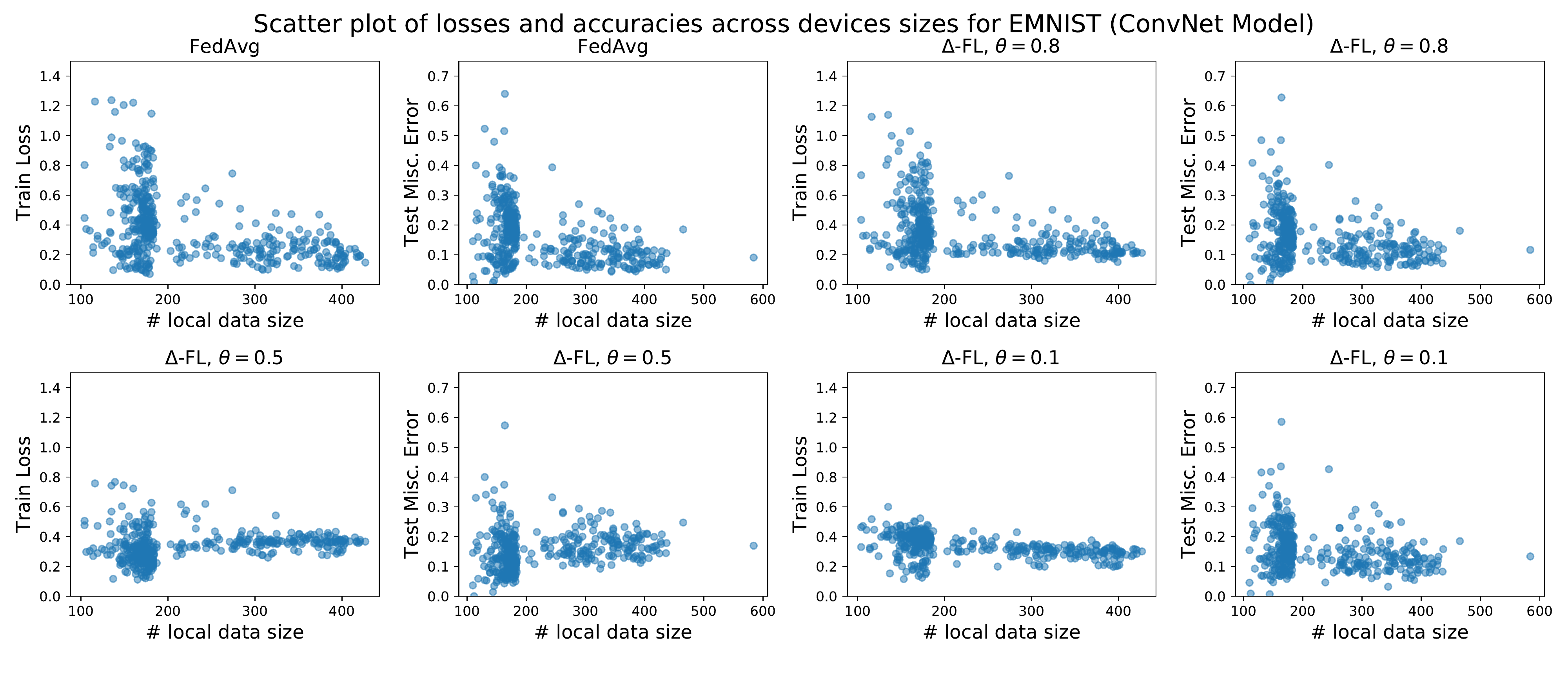}
   \caption{Scatter plot of (a) loss on training device vs. amount of local data, and (b) misclassification error on testing device vs. amount of local data for EMNIST.}
   \label{fig:a:expt:scatter-emnist}
\end{figure}

\begin{figure}[t!]

   \includegraphics[width=0.96\linewidth]{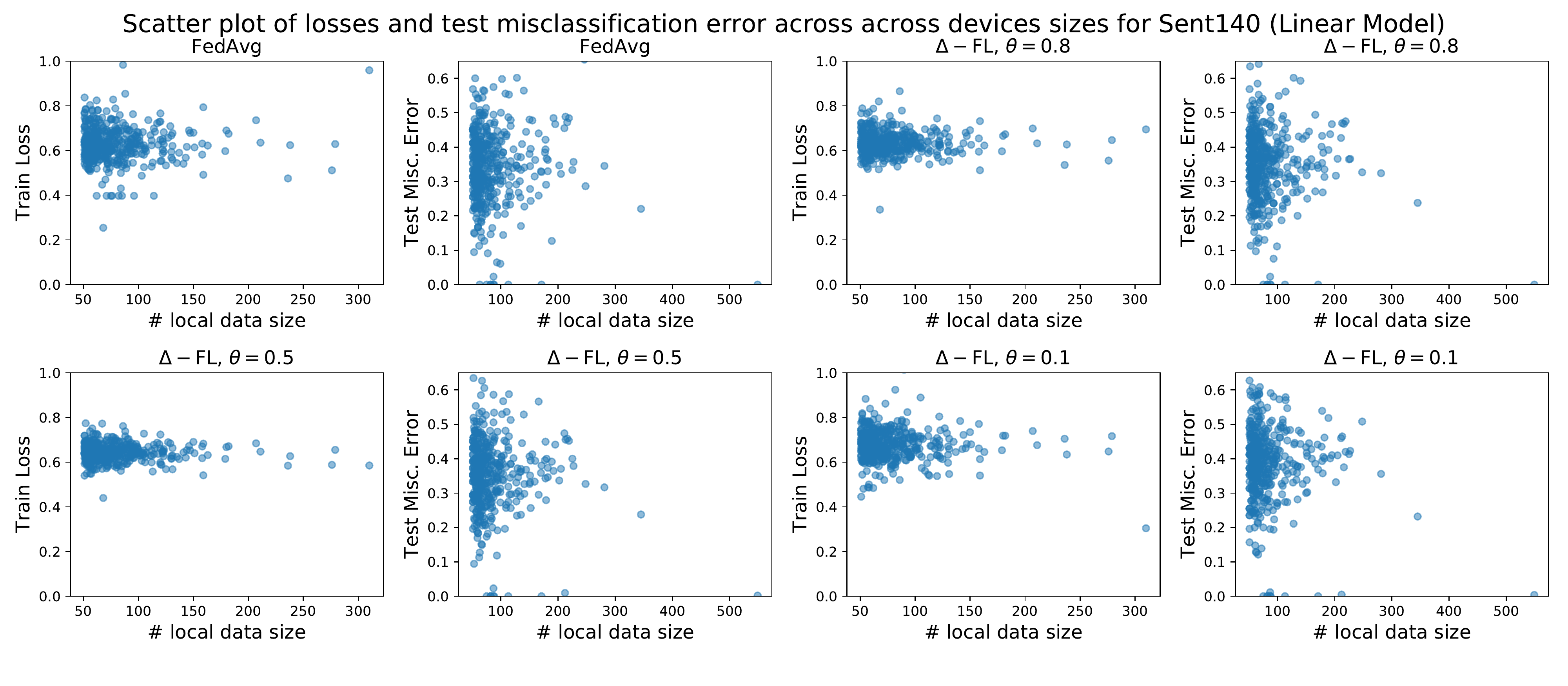}
   
   \includegraphics[width=0.96\linewidth]{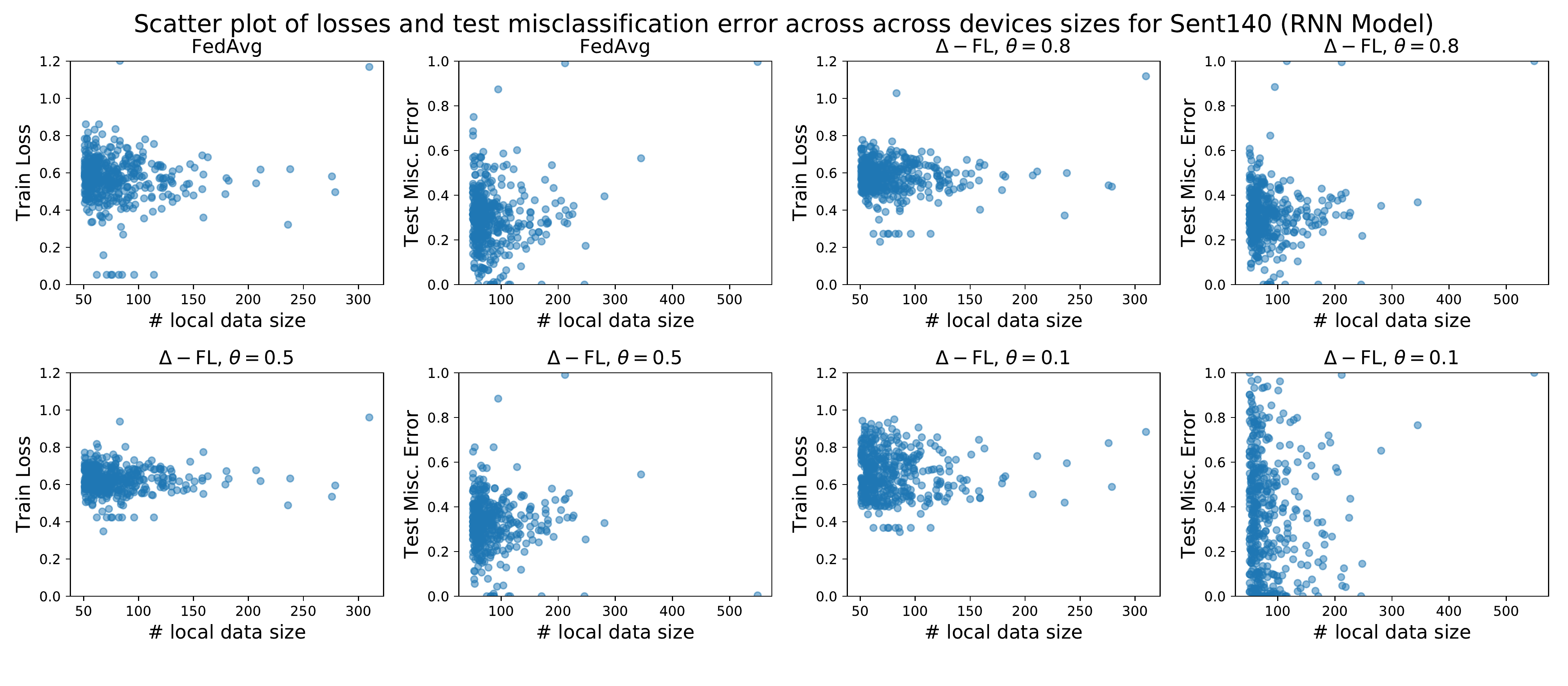}
   
   \includegraphics[width=0.96\linewidth]{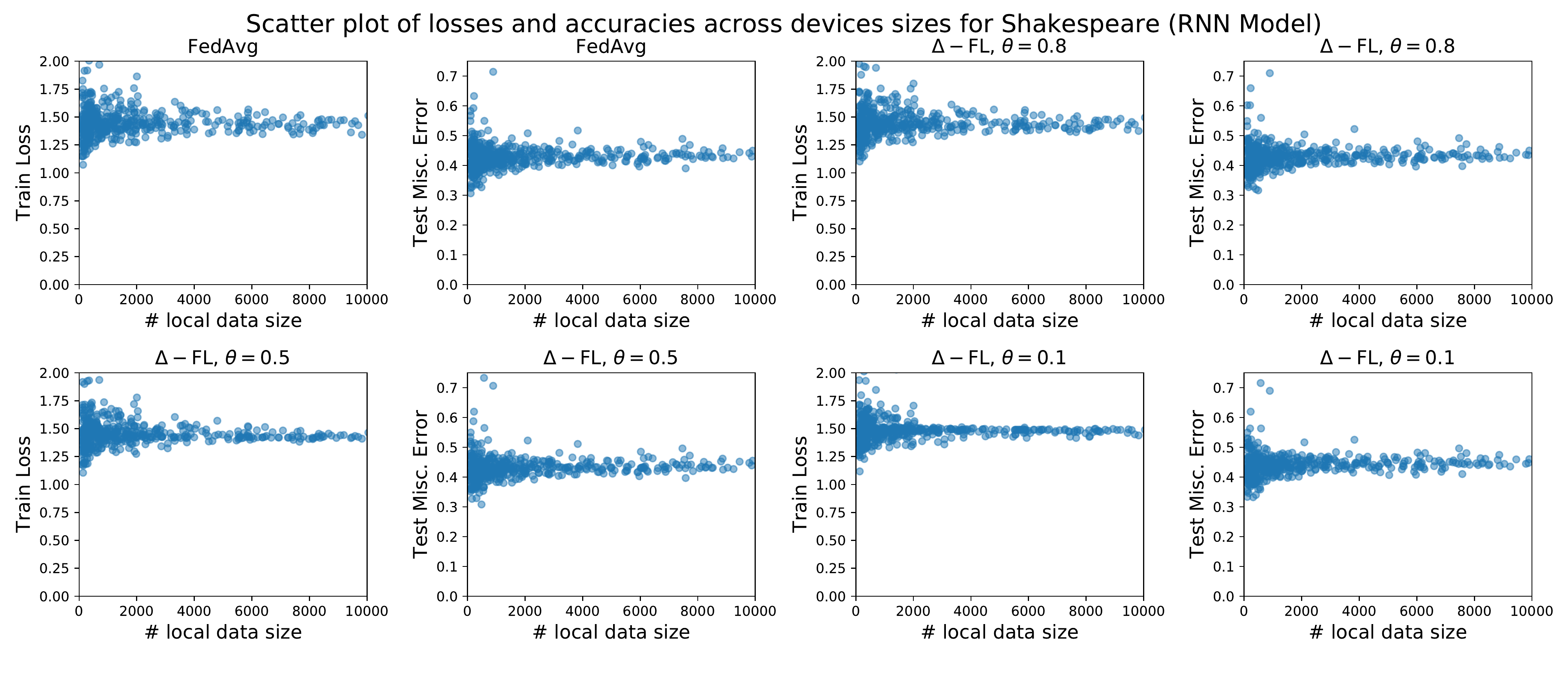}
   \caption{Scatter plot of (a) loss on training device vs. amount of local data, and (b) misclassification error on testing device vs. amount of local data for Sent140 and Shakespeare.}
   \label{fig:a:expt:scatter-sent-shake}
\end{figure}

\myparagraph{Number of Devices Selected per Communication Rounds}
Next, we plot the number of devices selected per round (after device 
filtering, if applicable). The shaded area denotes the maximum
and minimum over 5 random runs. We see from Figure~\ref{fig:a:expt:device-per-round}
that device-filtering is stable in the number of devices filtered out.
\begin{figure}[t!]
    \centering
    \includegraphics[width=\textwidth]{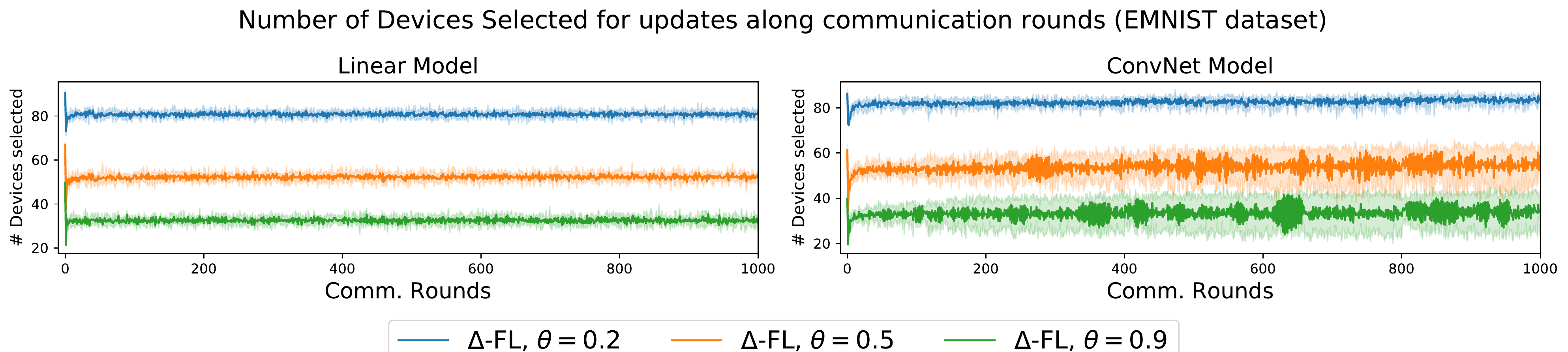}
    \caption{Number of devices selected per round (after device filtering, if applicable) for the EMNIST dataset. The shaded region denotes the 
    maximum and minimum over 5 random runs.}
    \label{fig:a:expt:device-per-round}
\end{figure}

\myparagraph{Effect of the Period of Update of $\eta$}
Here, we compare with the experiments to update $\eta$ every $T_\eta$ rounds. Equivalently, this amounts to running the $w$-step of the 
alternation minimization meta-algorithm for $T_\eta$ steps.
See Figure~\ref{fig:a:expt:t-eta-emnist}
for EMNIST linear model and Figure~\ref{fig:a:expt:t-eta-sent140}
for Sent140 linear model.
We see that on the Sent140 linear model, the choice of $T_\eta$ does not 
make any difference, while for EMNIST, larger $T_\eta$ appears to help. 
Note that Algorithm~\ref{algo:main:fed:proposed}
corresponds to $T_\eta = 1$.

\begin{figure}[t!]
    \centering
   \includegraphics[width=\linewidth]{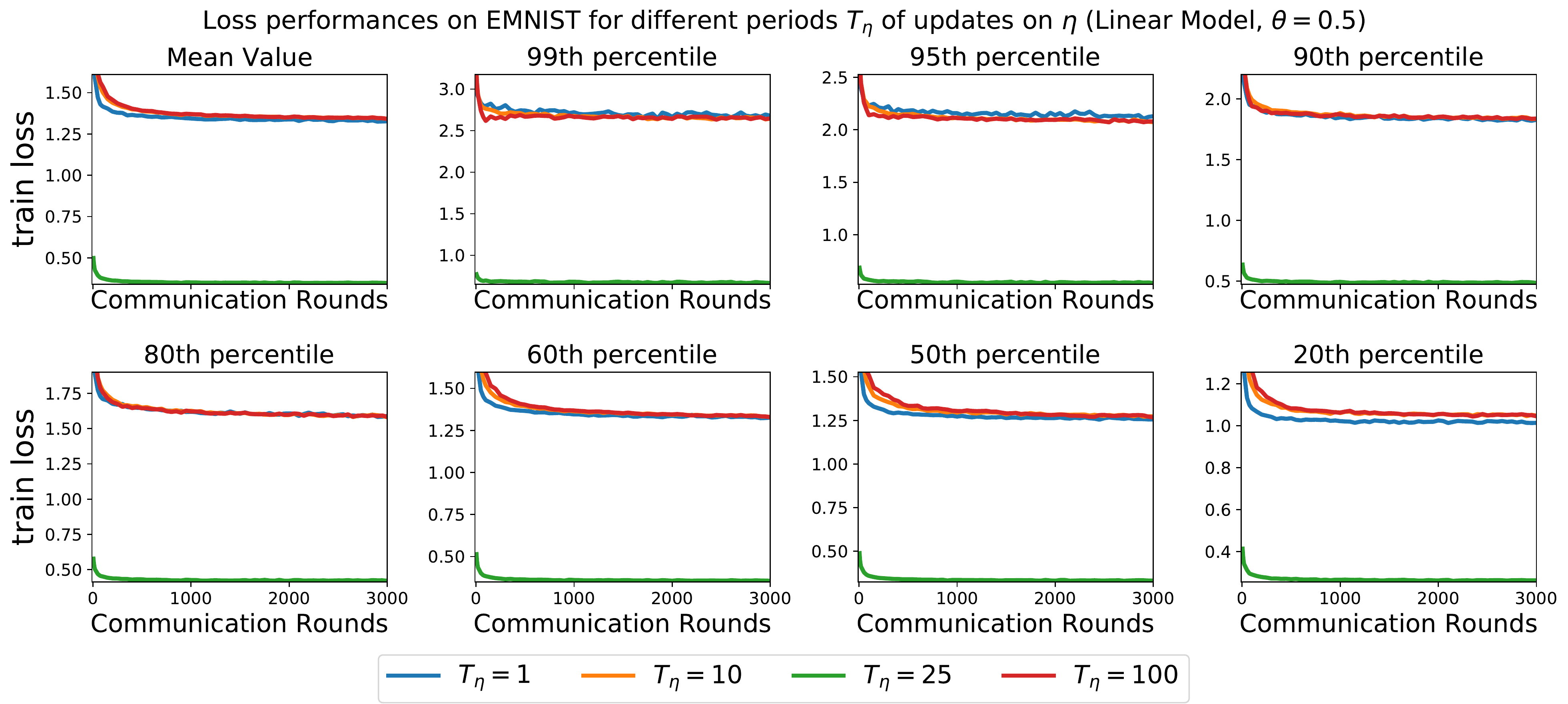}

   \includegraphics[width=\linewidth]{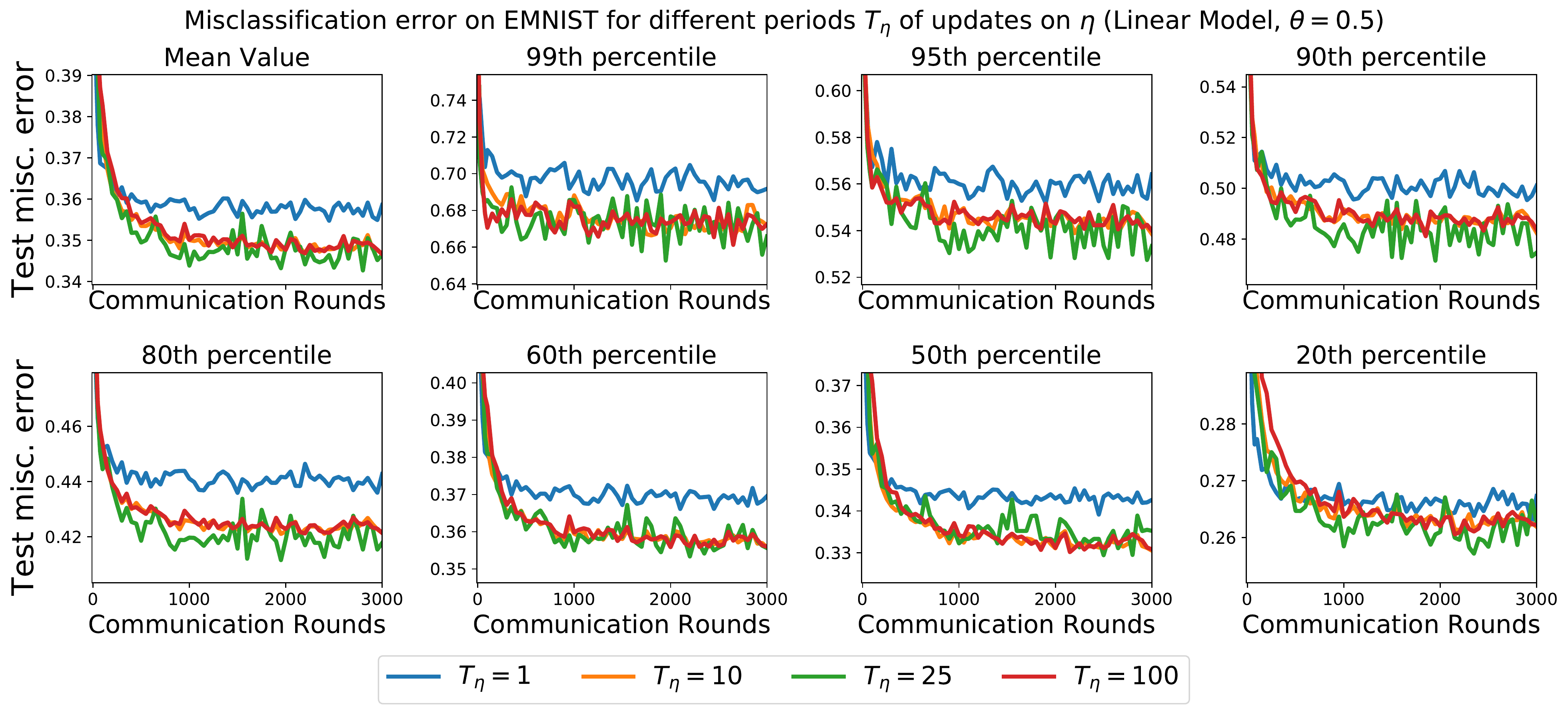}
   \caption{Effect of different choices of $T_\eta$ on EMNIST linear model.}
   \label{fig:a:expt:t-eta-emnist}
\end{figure}

\begin{figure}[t!]
    \centering
   \includegraphics[width=\linewidth]{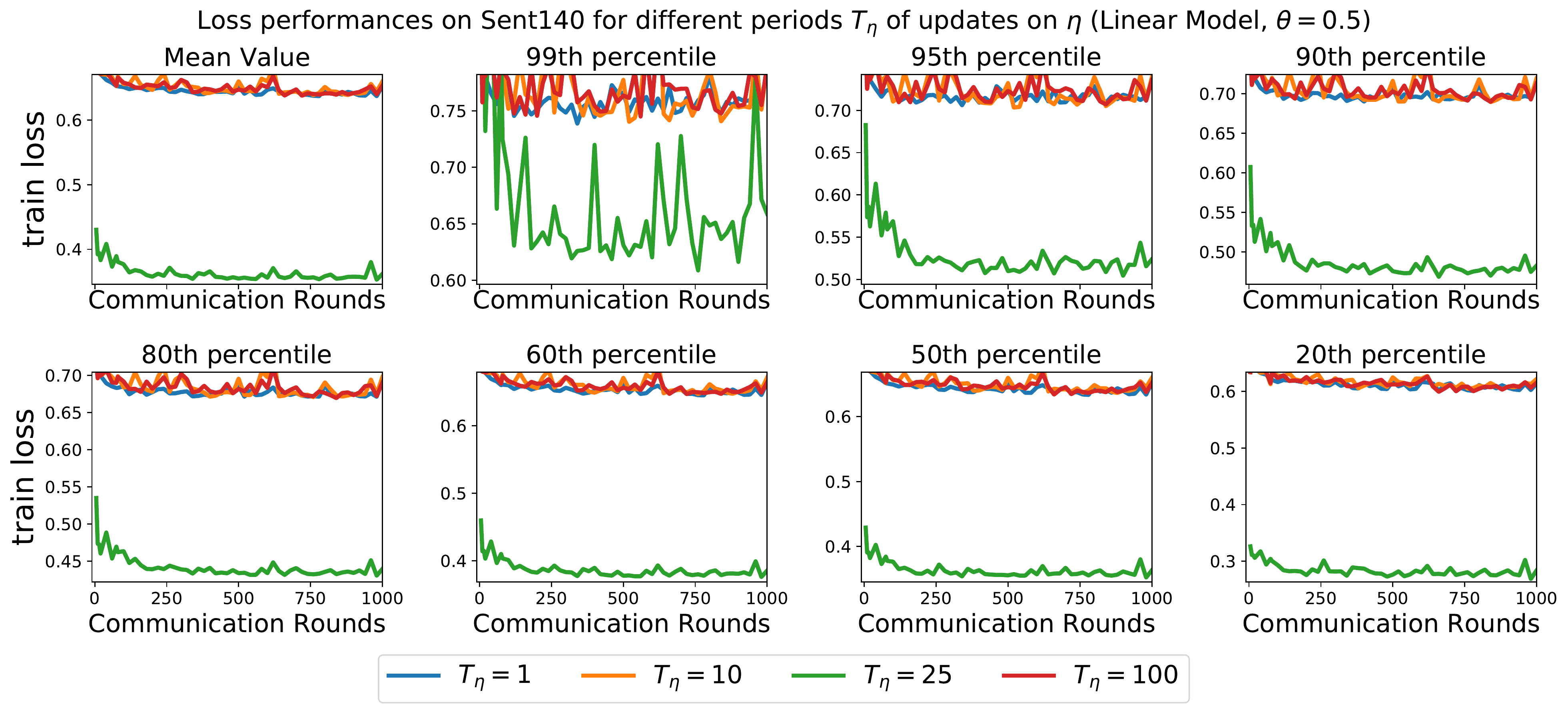}
   
   \includegraphics[width=\linewidth]{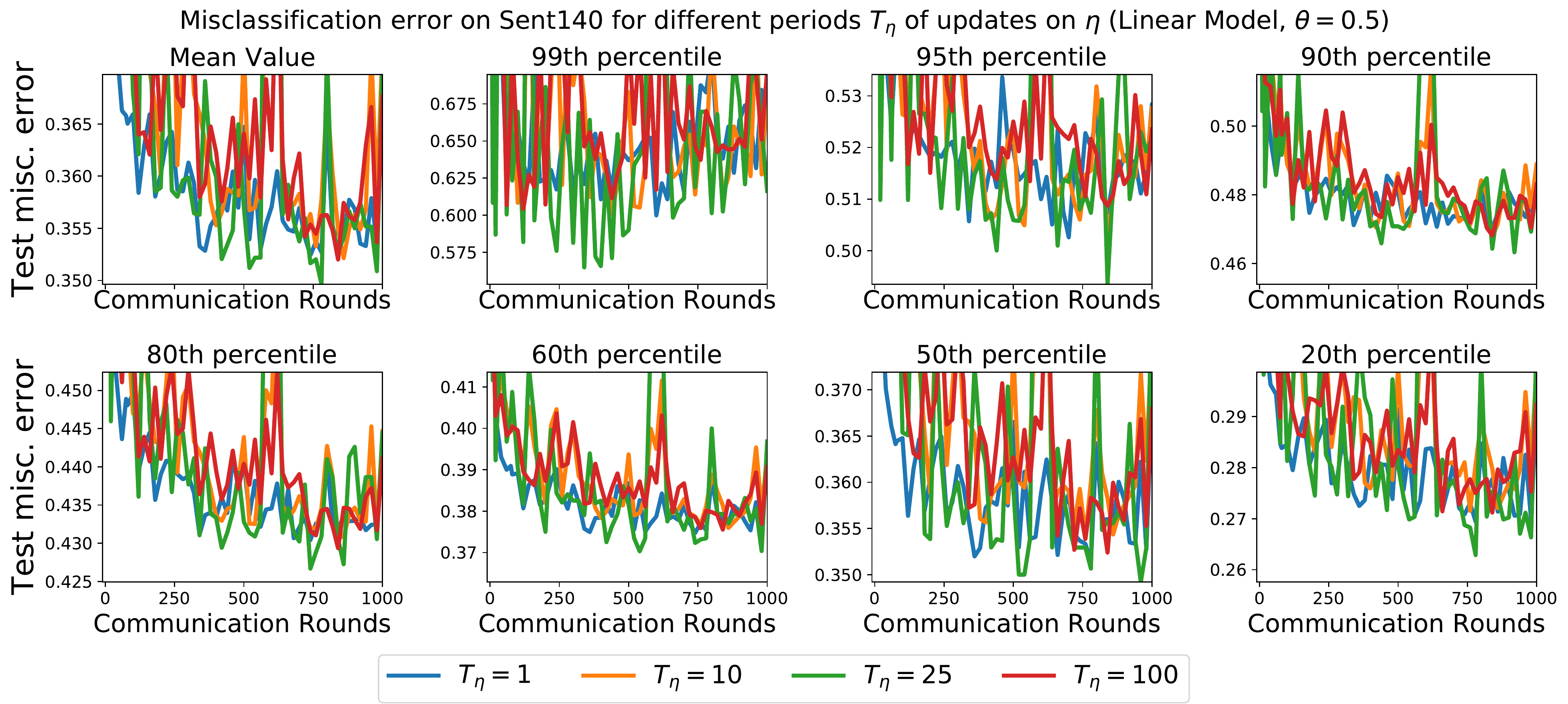}
   \caption{Effect of different choices of $T_\eta$ on Sent140 linear model.}
   \label{fig:a:expt:t-eta-sent140}
\end{figure} 
\clearpage
\section{Properties of Subdifferentials and Smoothing}\label{sec:a:technical_lemmas}
\subsection{Subdifferentials of Non-convex Functions}
\label{sec:a:techn:subdiff}

We briefly recall the standard notions of subgradients for nonsmooth functions. We follow the terminology of standard textbooks~\cite{clarke1990optimization,rockafellar2009variational}.
For a function 
$\varphi\colon\reals^d\rightarrow\reals\cup\{+\infty\}$, we define the regular (or Fr{\'e}chet) subdifferential of $\varphi$ at $\bar{x}$ as 
\[
\partial^R \varphi(\bar{x}) =\{z\in \mathbb{R}^n: ~\varphi(x)\geq
\varphi(\bar{x})+{z}\T(x-\bar{x}) + o(\norm{x-\bar{x}})\},
\]
which corresponds to the set of gradients of smooth functions that are below $\varphi$ and coincide with it at $x$. We then introduce
the limiting subdifferential as the set of all limits produced by regular subgradients
\[%
\partial^L \varphi(\bar{x})= \limsup_{x\rightarrow \bar{x}, \varphi(x)\rightarrow \varphi(\bar{x})}
\partial^R \varphi(x).
\]%
We also consider the (Clarke) subdifferential which can be defined, when $\varphi$ is locally Lipschitz, by the convex hull of the limiting subdifferential:
\[
\partial \varphi(\bar{x}) = \text{conv}\ \partial^L \varphi(\bar{x}).
\]
These notions generalize (sub)gradients of both
smooth functions and convex functions: for these functions indeed, the three  subdifferentials coincide, and they reduce to $\{\nabla \varphi(\bar x)\}$ when $\varphi$ is smooth and to the standard subdifferential from convex analysis when $\varphi$ is convex.

\subsection{Infimal Convolution Smoothing}
\label{sec:a:techn:smoothing}

A convex, non-smooth function $h$ can be smoothed by  infimal convolution with a smooth function~\citep{nesterov2005smooth,beck2012smoothing}. 
We use its dual representation, recalled below.

\begin{definition} \label{defn:smoothing:inf-conv}
	For a given convex function $h:\reals^m \to \reals$, a smoothing function $\omega: \dom h^* \to \reals$ which is 
	1-strongly convex with respect to $\norm{\cdot}$, 
	and a parameter $\nu > 0$, define
	\begin{align*}
	h_{\nu \omega}(z) := \max_{u \in \dom h^*} \left\{ \inp{u}{z} -  h^*(u) - \nu \omega(u) \right\} \,.
	\end{align*}
	as the smoothing of $h$ by $\nu \omega$.
\end{definition}

We now state a classical result showing how the parameter $\nu$ 
controls both the approximation error and the level of the smoothing.
For a proof, see~\citep[Proposition 39]{pillutla2018smoother}, which is an extension of~\citep[Theorem 4.1, Lemma 4.2]{beck2012smoothing}.

\begin{theorem} \label{thm:setting:beck-teboulle}
	Consider the setting of Def.~\ref{defn:smoothing:inf-conv}. 
	The smoothing $h_{\nu \omega}$ is continuously differentiable and its gradient, given by 
	\[
	\grad h_{\nu \omega}(z) = \argmax_{u \in \dom h^*} \left\{ \inp{u}{z} - h^*(u) - \nu \omega(u) \right\}
	\]
	is $1/\nu$-Lipschitz with respect to $\normd{{\cdot}}$,
	the dual norm to $\norm{\cdot}$. 
	Moreover, letting $h_{\nu \omega} \equiv h$ for $\nu = 0$, the smoothing satisfies, for all $\nu_1 \ge \nu_2 \ge 0$,
	\begin{align*}
		(\nu_1 - \nu_2) \inf_{u \in \dom h^*} \omega(u) 
		\le 
		h_{\nu_2 \omega}(z) - h_{\nu_1 \omega}(z) 
		\le 
		(\nu_1 - \nu_2) \sup_{u \in \dom h^*} \omega(u) \,.
	\end{align*}
\end{theorem}

\end{document}